\documentclass[11pt]{article}
\usepackage[utf8]{inputenc}
\usepackage[margin=1in]{geometry}
\usepackage{lmodern}

\usepackage{microtype}
\usepackage{graphicx}
\usepackage{subfigure}
\usepackage{booktabs}
\usepackage[round]{natbib}

\usepackage{amsmath,amsthm,amsfonts,amssymb,mathdots,array,mathrsfs,bm,bbm,stmaryrd,xcolor}

\usepackage{breakcites}

\usepackage[T1]{fontenc}
\usepackage{enumerate}
\usepackage{inputenc}

\usepackage{graphicx}
\usepackage{wrapfig}

\usepackage{algorithm}
\usepackage{algorithmic}
\usepackage{color}

\usepackage[utf8]{inputenc} 
\usepackage[T1]{fontenc}    

\usepackage[colorlinks,linkcolor=red,filecolor=blue,citecolor=blue,urlcolor=blue]{hyperref}

\usepackage{url}            
\usepackage{amsfonts}       
\usepackage{nicefrac}       
\usepackage{microtype}      
\usepackage{xcolor}
\usepackage{enumitem}
\usepackage{bbm}

\usepackage{tikz}
\usetikzlibrary{mindmap}

\usepackage{tikzpagenodes}
\usepackage{mathabx}

\usepackage[toc,page]{appendix}

\usepackage{math_notation}

\newtheorem*{theorem*}{Theorem}

\newtheorem{theorem}{Theorem}
\newtheorem{lemma}{Lemma}

\newtheorem{remark}{Remark}

\newenvironment{proofoutline}
 {\proof[Proof outline]}
 {\endproof}

\def \dh {\mathsf{d_H}}
\def \dH {\mathsf{d_H}}

\def \bPitrue {\bPi^{\natural}}
\def \bBtrue {\bB^{\natural}}
\def \bbetatrue {\bbeta^{\natural}}
\def \betatrue {\beta^{\natural}}

\def \bPiopt {\bPi^{\opt}}
\def \bBopt {\bB^{\opt}}

\def \term {\mathsf{Term}}
\def\srank#1{\textup{srank}(#1)}

\def\wtminus#1#2{\wt{#1}_{\setminus (#2)}}
\def\sampminus#1#2{#1_{\setminus (#2)}}

\definecolor{celadon}{rgb}{0.67, 0.88, 0.69}
\definecolor{chromeyellow}{rgb}{1.0, 0.65, 0.0}
\definecolor{columbiablue}{rgb}{0.61, 0.87, 1.0}

\allowdisplaybreaks

\begin{document}

\title{\bf Optimal Estimator for Linear Regression \\ with Shuffled Labels\vspace{0.3in}}

\author{
  \textbf{Hang Zhang, \ \  Ping Li} \vspace{0.1in}\\
  Cognitive Computing Lab\\
  Baidu Research\\
  10900 NE 8th St. Bellevue, WA 98004, USA\\
  \texttt{ \{zhanghanghitomi,\ pinli98\}@gmail.com}
}
\date{\vspace{0.1in}}
\maketitle

\begin{abstract}
\noindent\footnote{Preliminary results appeared in Proceedings of the 37th International Conference on Machine Learning (ICML'20).}This paper considers the task of linear regression with shuffled labels, i.e.,
$\bY = \bPitrue \bX \bBtrue + \bW$, where $\bY \in \RR^{n\times m}, \bPitrue \in \RR^{n\times n}, \bX\in \RR^{n\times p}, \bBtrue \in \RR^{p\times m}$,
and $\bW\in \RR^{n\times m}$, respectively, represent the
sensing results, (unknown or missing) corresponding information, sensing matrix, signal of interest, and additive sensing noise.
Given the observation $\bY$ and sensing matrix $\bX$, we propose
a one-step estimator to reconstruct $(\bPitrue, \bBtrue)$.
From the computational perspective, our estimator's complexity
is $O(n^3 + np^2m)$, which is no greater than the maximum complexity of a
linear assignment algorithm (e.g., $O(n^3)$) and a least square algorithm (e.g., $O(np^2 m)$).
From the statistical perspective, we divide the minimum
$\snr$ requirement into four regimes, e.g.,
unknown, hard, medium, and easy regimes; and
present sufficient conditions for the correct permutation
recovery under each regime:
$(i)$ $\snr \geq \Omega(1)$ in the easy regime;
$(ii)$ $\snr \geq \Omega(\log n)$ in the medium regime;
and $(iii)$ $\snr \geq \Omega((\log n)^{c_0}\cdot n^{\nfrac{c_1}{\srank{\bBtrue}}})$ in the hard regime ($c_0, c_1$ are some positive constants and $\srank{\bBtrue}$ denotes the stable rank of $\bBtrue$).
In the end, we also provide numerical experiments to confirm
the above claims.

\end{abstract}

\vspace{0.2in}


\newpage

\section{Introduction}\label{sec:intro}
The recent years  have witnessed a renaissance of permuted linear regression, or ``unlabeled sensing", due to its broad spectrum of applications ranging from database merging, to privacy, to communications, to computer vision, to robotics, to sensor networks, etc \citep{pananjady2018linear, unnikrishnan2015unlabeled, slawski2020two, slawski2017linear, pananjady2017denoising, zhang2022permutation}.
Here, we briefly review several use cases of permuted linear regression, including record linkage, data de-anonymization,
and header-free communications. For more information about other applications, we refer interested readers to the references thereof.

\begin{itemize}
\item
\textbf{Record linkage}. Given multiple databases containing
information about the same entities, the objective is to merge them
into one comprehensive database. However,
these data may not be well-aligned due to the data formatting or data quality issues.
Modeling the mismatches as a permutation is investigated as a mitigation strategy.

\item
\textbf{Data de-anonymization.}	
This task can be regarded as the opposite side of privacy protection.
The intruders aim to infer the hidden identities/labels in certain private networks with public information. One commonly used method is to compare the correlation between this information and to find matching pairs with
maximum correlation sum.
Here, the permuted linear regression
arises as a natural generalization.

\item
\textbf{Header-free communication.}
Another potential application of permuted linear regression is in the \emph{internet of things} (IOT) network, where communications happen frequently while the transmitted messages are usually in short length. To improve the bandwidth efficiency, the sensor identity
is omitted during transmission and thus the signal decoding involves
first restoring the correspondence.
\end{itemize}

In this paper, we consider the canonical setting for a
permuted linear sensing, which is written as
\begin{align}
\label{eq:sys_model}
\bY = \bPitrue\bX\bBtrue + \bW,
\end{align}
where $\bPitrue \in \{0, 1\}^{n\times n}$ denotes the unknown permutation matrix,
$\bX \in \RR^{n\times p}$ represents the design (sensing) matrix,
$\bBtrue \in \RR^{p\times m}$
is the signal of interest, and $\bW \in \RR^{n\times m}$ denotes the additive noise.
Compared with the canonical model of linear regression with well-aligned data, our task is to infer both the signal $\bBtrue$ and the missing correspondence information $\bPitrue$ from the pair $(\bX, \bY)$.

\subsection{Related work}

The study on permuted linear regression has a long history and
can be at least traced back to 1970s under the name of ``broken sample problem''~\citep{degroot1976matching, degroot1980estimation, goel1975re, bai2005broken}.
In recent years, we have witnessed a revival of its study and
can broadly divide these research works into
two categories: $(i)$ single observation model
and $(ii)$ multiple observations model.

We first discuss the work on the single observation model. In~\citet{unnikrishnan2015unlabeled}, they focus on the
noiseless single observation model, namely,
$\bW = \bZero$ and $m = 1$.
Assuming the entries in the sensing matrix $\bX$
are drawn from a continuous distribution and
$\bBtrue\in \RR^p$ is an arbitrary vector residing
within a linear space with dimension $p$,
\citet{unnikrishnan2015unlabeled} establish the necessary condition $n\geq 2p$
for the correct recovery.
Similar results have been discovered by~\citet{domankic2018permutations} but
with different approaches.
In~\citet{pananjady2018linear}, they investigate the noisy case under
the single observation model, i.e., $m=1$.
After obtaining the statistical limit of the minimum $\snr$
required for permutation recovery, e.g., $\Omega(n^c)$ ($c>0$ is some positive constant), they analyze the
\emph{maximum likelihood} (ML) estimator and show its performance
matches the order of the statistical limits thereof.
However, the ML estimator is NP-hard in general except for the special case when $p=1$, in other words, $\bBtrue$ is a scalar.
Later, an approximation algorithm for permutation recovery is presented in~\cite{hsu2017linear}; and in~\citet{slawski2017linear}, the permutation recovery is studied
from the viewpoint of denoising. We will give a more technical discussion of these works ~\citep{pananjady2018linear, hsu2017linear, slawski2017linear}, as their settings have a huge overlap with ours.
The content is deferred to Section~\ref{sec:single_obs_related_work} until we have collected the required facts. An independent line of research can be found in~\citet{tsakiris2019homomorphic, peng2021homomorphic}, where the
permuted linear regression is studied from the viewpoint of algebraic geometry.

Next, we survey the work on the multiple observations model~\citep{pananjady2017denoising, slawski2020two, zhang2022permutation, zhang2023greed, zhang2023one}.
In~\citet{pananjady2017denoising}, the authors study
the problem of
reconstructing the product $\bPitrue\bX \bBtrue$.
In~\citet{zhang2022permutation}, the focus is shifted
to reconstruct the individual values of
$\bPitrue$ and $\bBtrue$.
Statistical limits are presented as well as an investigation
of the ML estimator. In~\citet{slawski2020two},
they take a similar viewpoint of~\citet{slawski2017linear}
and use a denoising-based method for correspondence
recovery.
In \cite{zhang2023one}, a follow-up of this work, they consider the sparse matrix setting (i.e., each column of $\bBtrue$ is $k$-sparse) and design the estimator based on \citet{zhang2020optimal}.
A more detailed comparison between our works and the existing literature
can be found in Section~\ref{sec:multi_obs_related_work}.

In addition to the works mentioned above, other research, e.g.,
~\cite{haghighatshoar2018signal, tang2021low, jeong2020recovering, slawski2022permuted,fang2023regression}, are also related to our
 works. A detailed discussion is omitted due to their relatively loose connection.

\subsection{Contributions}
We define the \emph{signal-to-noise-ratio} ($\snr$)
as $\fnorm{\bBtrue}^2/(m\cdot \sigma^2)$
before describing our contributions.
\begin{itemize}
\item
We propose a one-step estimator for permuted linear regression, which consists of two sub-parts:
one for the permutation recovery and the other
for the signal recovery. We show the first sub-part is
with computational complexity $O(n^3)$, which is the
same as the oracle estimator (i.e., a linear assignment algorithm) where $\bBtrue$ is given a prior. For the second sub-part, our estimator's computational cost is
$O(np^2 m)$, which is the same as the least square algorithm generally associated with classical linear regression.

\item
We prove that our estimator almost reaches the
statistical limits for permutation recovery.
First, we study the single observation model (i.e., $m=1$) and show
our estimator can yield the ground truth permutation
matrix when
$\log \snr \gsim \log n$ and signal length $p$ is one, in other words, $\bBtrue$ is a scalar.
This bound matches the statistical limit up to some
positive constant.
Second, we investigate the multiple observations model, namely, $m > 1$. We divide the required
$\snr$ into four regimes, i.e., unknown, hard, medium, and easy regimes. In the medium and hard regimes, we show that our $\snr$ requirement for correct permutation recovery matches the
statistical limits up to a multiplicative
polynomial of $\log n$; while in the
easy regime, our $\snr$ requirement matches
the limits up to a multiplicative factor of
certain positive constant. Compared with
previous works, our work has a lead in both
the computational and statistical perspectives.
A detailed summary is put in Table~\ref{tab:compare}.
\end{itemize}

\newpage

\subsection{Notations}

Denote $c$, $c^{'}$, $c_i$ as some positive constants,
whose values are not necessarily the same even
for those with the same notations.
We denote $a\lsim b$ if there exists
some positive constants $c_0 > 0$ such that
$a\leq c_0 b$. Similarly, we define
$a\gsim b$ provided $a\geq c_0 b$ for
some positive constant $c_0$.
We write $a\asymp b$ when
$a\lsim b$ and $a\gsim b$ hold simultaneously.

We call  $z$ a
sub-gaussian \emph{random variable} (RV) with $\|z\|_{\psi_2} \leq K$
if it satisfies
$\Expc e^{\nfrac{z^2}{K^2}} \leq 2$
 (Section~$2.5.2$ in~\citet{vershynin2018high}).
A centered isotropic random vector $\bx$ is
defined such that $\Expc \bx = \bZero$ and
$\Expc \bx\bx^{\rmt}= \bI$ (Definition~$3.2.1$ in~\citet{vershynin2018high}).

For an arbitrary matrix $\bM$,
we denote
$\bM_{i, :}$ as its $i$th row,
$\bM_{: , i}$ as its $i$th column, and
$\bM_{ij}$ as its $(i,j)$th element.
The Frobenius norm of $\bM$ is defined as
$\fnorm{\bM}$ while the operator
norm is denoted as $\opnorm{\bM}$, whose definitions can be found
in  Section~$2.3$ of~\citet{golub2012matrix}.
Its stable rank
is defined as $\srank{\cdot} \defequal \fnorm{\cdot}^2/\opnorm{\cdot}^2$
(Section~$2.1.15$ in~\citet{tropp2015introduction}).
The inner product between matrices are
denoted as $\La \cdot, \cdot\Ra$; while the inner product
between vectors are denoted as $\la \cdot, \cdot \ra$.

Associate with each permutation matrix $\bPi$, we define
the operator $\pi(\cdot)$ that transforms index $i$
to $\pi(i)$ under $\bPi$.
The Hamming distance
$\dH(\bPi_1, \bPi_2)$ between permutation matrix
$\bPi_1$ and $\bPi_2$ is defined as
$\dH\bracket{\bPi_1, \bPi_2} = \sum_{i=1}^n \Ind\bracket{\pi_1(i) \neq \pi_2(i)}$.
The \emph{signal-to-noise-ratio} ($\snr$) is defined as
$\snr = \fnorm{\bBtrue}^2/(m\cdot \sigma^2)$.
More notations are in the supplementary material.

\begin{table*}[!t]
\centering
\caption{Comparison with the prior art.
All results are presented in their best orders, which may
only hold true in certain regimes.
The computational cost refers to the number of iterations.
Moreover, the logarithmic term is omitted in the notation $\wt{\Omega}(\cdot)$
and $\wt{O}(\cdot)$.
Notation $\rank(\cdot)$ denotes the rank of the corresponding matrix, $n_{\textup{min}}$ denotes the minimum sample number,
and $h_{\textup{max}}$ denotes the maximum allowed number of permuted rows.
Notation $\boldcheckmark$ means the requirement is met;
$\bigtimes$ means the requirement is not met;
and $\textbf{N/A}$ means not applied.
}
\vsp
\label{tab:compare}
\resizebox{6.7in}{!}{%
{
\begin{tabular}{@{}l|ccccccccccc@{}}\toprule
 & \multicolumn{2}{c}{\textup{Comput. Optim.}}
&& \multicolumn{2}{c}{\textup{Statis. Optim.}}
&& \multicolumn{2}{c}{${n_{\textup{min}}}/{p}~(\geq)$}
&& \multicolumn{2}{c}{${h_{\textup{max}}}/{n}~(\leq)$}  \\
\cmidrule{2-3} \cmidrule{5-6} \cmidrule{8-9} \cmidrule{11-12}
&   $m=1$ & $m\gg 1$ && $m=1$ & $m\gg 1$ && $m=1$ & $m\gg 1$  &&  $m=1$ & $m\gg 1$ \\
\midrule
\citep{pananjady2018linear} & $\boldcheckmark$ & $\textbf{N/A}$ && $\boldcheckmark$ & $\textbf{N/A}$ && $\wt{\Omega}(1)$  & $\textbf{N/A}$ & & $\wt{O}(1)$ & $\textbf{N/A}$ \\ \\
\citep{hsu2017linear}  & $\bigtimes$ & $\textbf{N/A}$ &&  $\bigtimes$ & $\textbf{N/A}$  && $\wt{\Omega}(1)$  & $\textbf{N/A}$ & & $\wt{O}(1)$  & $\textbf{N/A}$ \\ \\
\citep{slawski2017linear}    &$\bigtimes$ & $\textbf{N/A}$ &&  $\boldcheckmark$ & $\textbf{N/A}$  && $\wt{\Omega}(1)$  & $\textbf{N/A}$ & & $\wt{O}(\log^{-1}n)$ & $\textbf{N/A}$   \\ \\
\citep{zhang2022permutation} & $\textbf{N/A}$ & $\bigtimes$ & &  $\textbf{N/A}$ & $\boldcheckmark$ && $\textbf{N/A}$ & $\wt{\Omega}(1)$ && $\textbf{N/A}$ & $\wt{O}\bracket{\log^{-1} \rank(\bBtrue)}$  \\ \\
\citep{slawski2020two} & $\textbf{N/A}$ & $\bigtimes$ & & $\textbf{N/A}$ & $\boldcheckmark$ &&  $\textbf{N/A}$
& $\wt{\Omega}(p)$ && $\textbf{N/A}$ & $\wt{O}(\log^{-1} n)$\\ \\
\textbf{This work} & $\boldcheckmark$ & $\boldcheckmark$ & & $\boldcheckmark$ & $\boldcheckmark$   && $\wt{\Omega}(1)$ & $\wt{\Omega}(1)$ && $\wt{O}(1)$ & $\wt{O}(1)$\\
\bottomrule
\end{tabular}
}
}
\end{table*}

\subsection{Road map}
The organization of this paper is as follows.
In Section~\ref{sec:alg_descrip}, we formally state our
problem setting, present our estimator and its design insight, and review the minimax lower bounds. Then, we separately investigate our estimator's statistical properties
under the single observation model ($m = 1$)
and multiple observations model ($m > 1$).
Corresponding discussions are put in
 Section~\ref{sec:single_obser} and Section~\ref{sec:multi_observe}, respectively.
Simulation results are presented in Section~\ref{sec:simul} and the
conclusions are drawn in  Section~\ref{sec:conclusion}. The technical
details are deferred to the Appendix.

\section{Problem Setting}\label{sec:alg_descrip}
We start the discussion with a formal restatement of the
sensing model
\vspace{-0.05in}
\begin{align}
\label{eq:sys_model_unlabel_sense}
\bY = \bPitrue \bX \bBtrue + \bW,
\vspace{-0.05in}
\end{align}
where $\bY \in \RR^{n\times m}$ denotes the observation,
$\bPitrue \in \{0, 1\}^{n\times n}$ is the unknown permutation matrix
such that $\sum_i \bPitrue_{i, j} = \sum_{j}\bPitrue_{i,j} = 1$,
$\bX\in \RR^{n\times p}$ denotes the
sensing matrix such that
 each entry $\bX_{ij}$ are i.i.d. centered isotropic sub-gaussian RV with
 $\norm{\bX_{ij}}{\psi_2} \lsim 1$, i.e.,
 $\Expc \bX_{ij} = 0$ and $\Expc \bX_{ij}^2 = 1$,
$\bBtrue\in \RR^{p\times m}$ denotes the signal of interests,
and $\bW \in \RR^{n\times m}$ represents the additive Gaussian noise
with each entry $\bW_{ij}$ being Gaussian RV with zero mean
and $\sigma^2$ variance, namely,
$\bW_{ij}\stackrel{\textup{i.i.d}}{\sim} \normdist(0, \sigma^2)$.
\footnote{We call $n$ sample number, $p$ length of signal, and $m$ measurement number.}

Our goal is to reconstruct the pair
$(\bPitrue, \bBtrue)$ from observations $\bY$ and the
sensing matrix $\bX$. In the following context, we put our major focus on the
permutation recovery:
on one hand, this problem reduces to the classical setting of linear regression once ground truth permutation is obtained; on the other hand,
no meaningful bound on $\fnorm{\wh{\bB} - \bBtrue}$ can be obtained with incorrect correspondence information (e.g., $\bPitrue$).

\subsection{Estimator and its design insight}

\begin{algorithm}[h]
\caption{One-Step Estimator.}
\label{alg:one_step_estim}
\begin{algorithmic}
\STATE {\bfseries Input:} observation $\bY$ and sensing matrix $\bX$.
\STATE {\bfseries Output:} pair $(\bPiopt,~\bBopt)$, which
is written as
\begin{align}
\label{eq:optim_estim_pi}
\bPiopt &= \argmax_{\bPi \in \calP_n} \
\La \bPi, \bY\bY^{\rmt}\bX\bX^{\rmt}\Ra, \\
\bBopt &= \bX^{\dagger} \wh{\bPi}^{\rmt}\bY,
\label{eq:optim_estim_B}
\end{align}
where $\bX^{\dagger} = (\bX^{\rmt}\bX)^{-1}\bX^{\rmt}$ denotes the pseudo-inverse of
$\bX$ (Section~$5.5.2$ in~\citet{golub2012matrix}) and
$\calP_n$ is the set of all possible permutation matrices.
\end{algorithmic}
\end{algorithm}

We propose a one-step estimator, whose details are summarized
in Algorithm~\ref{alg:one_step_estim}.
Before a thorough investigation of
our estimator's properties, we first present its underlying
design insight, which is quite straightforward.
Considering the oracle situation where $\bBtrue$
is given a prior,
we can reconstruct the  permutation matrix $\bPitrue$ via
\begin{align}
\label{eq:one_step_oracle}
\bPiopt = \argmax_{\bPi\in \calP_n} \La \bPi, \bY (\nfrac{\bB^{\natural \rmt}}{\alpha})\bX^{\rmt}\Ra,
\end{align}
where $\alpha > 0$ is an arbitrary scaling constant.
Back to our case, we can see that the major difficulty
comes from the lack of knowledge about $\bBtrue$, to put
more precisely, the direction of $\bBtrue$, since the
solution to \eqref{eq:one_step_oracle} remains the same
up to some positive scaling factor.

Inspired by the recent progress in non-convex optimization
\cite{candes2010matrix, chi2019nonconvex, balakrishnan2017statistical}, we would like to approximate $\bBtrue$'s direction
with $\bX$ and $\bY$. Note that $\Expc \bX^{\rmt}\bY
=(n-h)\bBtrue$, which is parallel to $\bBtrue$ if $h < n$ ($h$ denotes the number of permuted rows by $\bPitrue$, e.g.,
$h \defequal \dH(\bI, \bPitrue)$).
Assuming that $\bX^{\rmt}\bY$ is close
to $\Expc \bX^{\rmt}\bY$,
we design our estimator
by substituting $\bBtrue$ in \eqref{eq:one_step_oracle} with
$\bX^{\rmt}\bY$, which then leads to the permutation estimation
in \eqref{eq:optim_estim_pi}. Once the permutation is obtained,
we can restore \eqref{eq:sys_model_unlabel_sense} to
the classical setting of linear regression and then estimate
$\bBtrue$ with a \emph{least-square} estimator~\citep{golub2012matrix}.

Although at first glance this idea looks simple, if not naive,
in the following context, we will show that our estimator can reach statistical optimality in a broad regime.
In addition, our estimator is tuning-free, to put more specifically, no estimation of the noise variance $\sigma^2$ nor the number of permuted rows $h$ is required.

\subsection{Computational cost}
This subsection concerns our estimator's computational cost.
We consider two types of oracle estimators as benchmarks.
\begin{itemize}
\item
\textbf{Oracle estimator I.}
We consider the oracle scenario where $\bBtrue$ is given a prior.
Then, we can recover $\bPiopt$ as in \eqref{eq:one_step_oracle},
which is with computational cost $O(n^3)$.

\item
\textbf{Oracle estimator II.}	
We consider the oracle scenario when $\bPitrue$ is known in advance.
The sensing relation in \eqref{eq:sys_model_unlabel_sense}
reduces to the classical multivariate linear regression
and least-square algorithm for $\bBopt$
takes up to $O(np^2m)$ time.	
\end{itemize}

Then we study our estimator's computational
cost. In the first step \eqref{eq:optim_estim_pi}, we only sacrifice one matrix multiplication, i.e.,
replacing $\bBtrue$ by the product $\bX^{\rmt}\bY$.
Since the computational bottleneck
lies in solving the linear assignment problem
\citep{kuhn1955hungarian, bertsekas1992forward},
one additional matrix multiplication does not
change the computational complexity, which is
also of order $O(n^3)$.
Similarly, in the second step
\eqref{eq:optim_estim_B} our estimator sacrifices another
matrix multiplication, whose cost
is negligible when compared
with the total cost of \textbf{Oracle Estimator II}.

With the relation
$n^3 + np^2 m\leq 2(n^3 \vcup np^2m)$, we conclude our estimator's computational
cost $O(n^3 + np^2m)$ is in the same order as the maximum computational costs of the above-mentioned two oracle estimators.

\subsection{Mini-max lower bounds}
Before studying the statistical properties of
our algorithm, we review the mini-max lower bounds on the
permutation recovery.
First, we consider the single observation model, i.e.,
$m =1$. We have
\begin{theorem}[Theorem~$2$ in~\citet{pananjady2018linear}]
\label{thm:single_statis_lb}
For any estimator $\wh{\bPi}$, we have the
error probability $\Prob(\wh{\bPi}\neq \bPi^{\natural})$
exceed $1- c_0 \cdot e^{-c_1n\delta}$ provided that
$2 + \log(1 + \snr) \leq (2-\delta)\log n~,0 < \delta < 2$.
\end{theorem}
\noindent
This theorem suggests that we need $\snr$ to be at least
of order $\Omega(n^c)$ to avoid construction failure of
$\bPitrue$.
Then we move on to the multiple observations model, i.e.,
$m > 1$. The corresponding lower bound is summarized as
\begin{theorem}[Theorem~$1$ in~\citet{zhang2022permutation}]
\label{thm:multi_statis_lb}	
For any estimator $\wh{\bPi}$, we have the
error probability $\Prob(\wh{\bPi}\neq \bPi^{\natural})$ exceed
$1/2$, provided that
$\logdet\bracket{\bI + \nfrac{\bB^{\natural\rmt}\bBtrue}{\sigma^2}}  <
\frac{\log n! - 2}{n}$.
\end{theorem}
\begin{remark}
Theorem~\ref{thm:single_statis_lb} can be
regarded as a special case of Theorem~\ref{thm:multi_statis_lb}, since the former
theorem can be obtained from the latter theorem by letting
$m = 1$. 	
\end{remark}
\noindent
Denote $\lambda_i(\cdot)$ as the
$i$th singular value,
we can rewrite
$\logdet\bracket{\bI + \nfrac{\bB^{\natural\rmt}\bBtrue}{\sigma^2}}$
as  $\sum_i \log[1 + \nfrac{\lambda_i^2(\bBtrue)}{\sigma^2}]$
and approximately $\logdet\bracket{\bI + \nfrac{\bB^{\natural\rmt}\bBtrue}{\sigma^2}}$
as $\srank{\bBtrue}\cdot \log(1 + \snr)$.
Then, we conclude that no estimator can reliably recover
$\bPitrue$ if $\log \snr \lsim \frac{\log n}{\srank{\bBtrue}}$.

The following context studies our algorithm's statistical properties under the single observation model ($m = 1$) and
multiple observations model ($m > 1$), respectively.
We will show that the required
$\snr$s for correct permutation recovery are close to the statistical
limits thereof.

\section{Single Observation Model}
\label{sec:single_obser}
This section considers the single observation model, namely,
$m = 1$. To distinguish this case with multiple observations
model, i.e., $m > 1$, we rewrite the
sensing relation in \eqref{eq:sys_model_unlabel_sense} as
\begin{align}
\label{eq:single_obser_sense}
\by = \bPitrue\bX\bbetatrue + \bw,
\end{align}
where $\by,~\bw \in \RR^n$ and $\bbetatrue \in \RR^p$
all reduce to vectors.
Based on whether $p \geq 2$, we
find that our estimator exhibits vastly
different behaviors.

\subsection{A warm-up example: $p = 1$}
\label{subsec:warm_up}
First, we consider a warm-up example
where $m = 1$ and $p = 1$.
In such situation, we have $\bX$ to be a vector of length $n$
and $\bbetatrue$ be a scalar. Then we have
\vsp

\begin{theorem}
\label{thm:warm_up}
Assume $\bx$ be an isotropic log-concave random vector
with zero mean and $\norm{\bx}{\psi_2} \lsim 1$.
Consider the large-system limit where $n$ is sufficiently
large and assume
$(i)$ $h = \dH(\bI, \bPitrue) \leq \nfrac{n}{4}$,
and $(ii)$ $n\geq 2p$.
Our estimator in \eqref{eq:optim_estim_pi}
can return the correct
permutation matrix, i.e., $\bPiopt = \bPitrue$,
with probability at least $1-c_1 n^{-1}$
provided that
$\log \snr \geq c\cdot \log n$.
\end{theorem}

\begin{proofoutline}
To emphasize the fact that sensing matrix
$\bX$ reduces to a vector when $p = 1$, we denote it as $\bx$.
To start with, we consider the noiseless case, where
$\snr$ is infinite.
With some simple algebraic manipulations, we can
expand the inner product
$\langle \bPi, \by\by^{\rmt}\bx\bx^{\rmt}\rangle$
in \eqref{eq:optim_estim_pi} as
\begin{align}
\label{eq:warm_up_noiseless_expand}
\langle \bPi, \by\by^{\rmt}\bx\bx^{\rmt} \rangle =
(\betatrue)^2 \langle \bx, \bPitrue\bx\rangle \cdot \langle \bPi\bx, \bPitrue\bx\rangle.
\end{align}
Since the first term $(\betatrue)^2 \langle \bx, \bPitrue\bx\rangle$
is positive with high probability, we conclude that
the maximum is reached when $\bPiopt = \bPitrue$.
Back to the noisy case,
we interpret the observation $\by$ as
a perturbed version of product $\bPitrue\bx\betatrue$. Provided
the perturbation is significantly small,
our estimator can obtain the correct permutation
matrix $\bPitrue$ with high likelihood.
For the technical details, we refer
the interested readers to the appendix.
\end{proofoutline}

First, we comment on the tightness of Theorem~\ref{thm:warm_up}.
According to Theorem~$2$ in~\citet{pananjady2018linear},
which is restated as Theorem~\ref{thm:single_statis_lb},
correct permutation recovery
requires $\log \snr\gsim \log n$ at least.
Easily, we can see that our estimator matches this statistical limit with a difference up to some multiplicative constant and hence concludes its tightness.

In addition, we would like to mention the
general case where $p\geq 2$ can be transformed
to the warm-up case where $p = 1$ if the direction
$\be \defequal \bbetatrue/\fnorm{\bbetatrue}$
is known. The detailed construction method comes as follows.
With the Gram-Schmidt process (Section~$5.2.7$ in~\citet{golub2012matrix}), we can construct an orthonormal matrix $\bQ \in \RR^{p\times p}$  whose first column is $\be$. Then we rewrite
\eqref{eq:single_obser_sense} as
\[
\by = \fnorm{\bbetatrue}\bPitrue\bracket{\bX\bQ}_{:, 1} + \bw.
\]
Easily, we can show
$\bracket{\bX\bQ}_{:, 1}$ has similar statistical
properties as $\bX_{:, 1}$
and can restore it to the model thereof, to which
Theorem~\ref{thm:warm_up} also applies.

At last, we should stress that our proof of
Theorem~\ref{thm:warm_up} relies
heavily on the expansion in \eqref{eq:warm_up_noiseless_expand},
which only holds true when $p =1$ or $\be \defequal \bbetatrue/\fnorm{\bbetatrue}$ is known.
In fact, we will see shortly that our estimator
fails to obtain correct permutation matrix--even with infinite $\snr$--when $p$ exceeds over $2$.

\subsection{The general case: $p \geq 2$}
This subsection studies a more general case where
$m = 1$ and $p \geq 2$. We consider the noiseless
case ($\sigma = 0$) and set $\bbetatrue$ as $[1000;~1000]^{\rmt}$.
The sample number $n$ is picked to be $1000$.
We evaluate the permutation recovery by Hamming distance
$\dh(\cdot, \bPitrue)$.

First, we obtain $(\wh{\bPi}^{(0)}, \bbeta^{(0)})$ with Algorithm~\ref{alg:one_step_estim}.
The error $\dh(\wh{\bPi}^{(0)}, \bPitrue)$ is approximately
$900$, which means almost all correspondence
$\pi^{\natural}(\cdot)$ are incorrectly detected. To refine the permutation recovery, we carry out the following alternative minimization
\[
\wh{\bPi}^{(t+1)} &= \argmax_{\bPi \in \calP_n} \La  \bPi, \bY \wh{\bB}^{(t)\rmt}\bX^{\rmt}\Ra;  \\
\wh{\bbeta}^{(t+1)} &= \bX^{\dagger}\wh{\bPi}^{(t+1)\rmt}\bY,
\]
where $\wh{\bPi}^{(t)}$ and $\wh{\bbeta}^{(t)}$ denote the
reconstructed value of $\bPi$ and $\bbeta$ in the $t$th iteration, respectively. Numerical experiment suggests that the correct permutation matrix remains out of reach even after $100$ iterations. An illustration is put in Figure~\ref{fig:fail_example}.

\begin{figure}[!h]
\centering
\includegraphics[width = 3.6in]{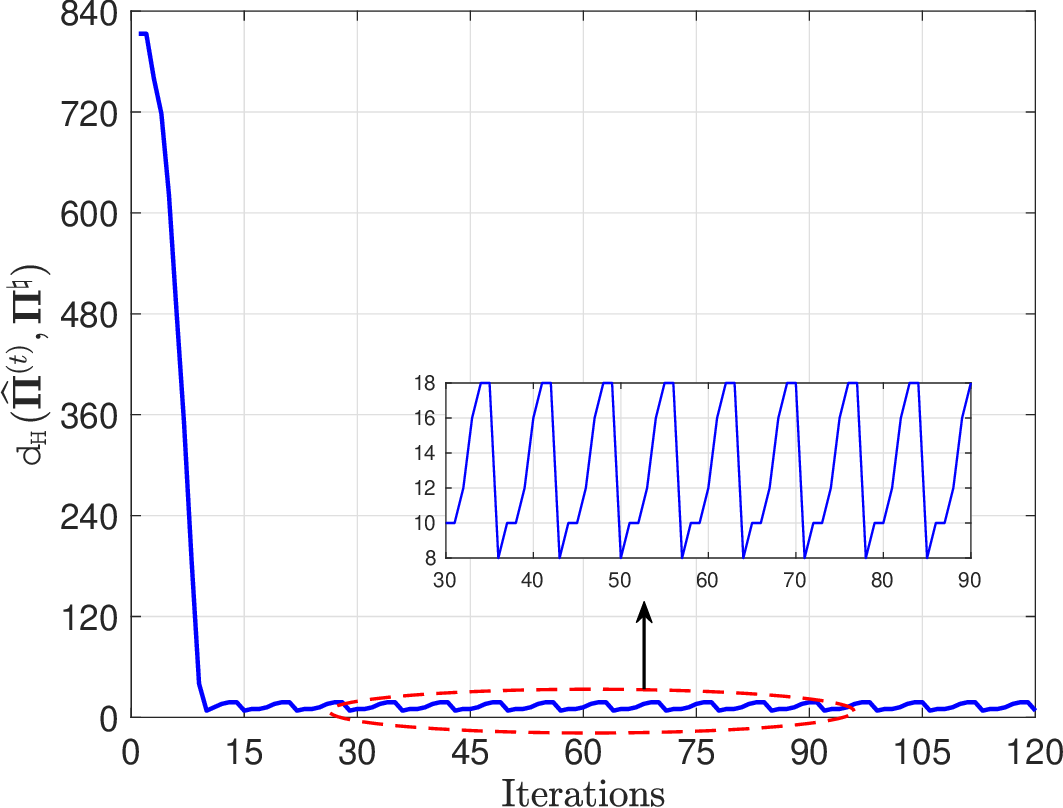}
\caption{Hamming distance $\dh(\wh{\bPi}^{(t)}, \bPitrue)$ when
$n = 1000$, $p = 2$, $m = 1$, $\bbeta^{\natural} = [1000;~1000]^{\rmt}$,
and $\sigma = 0$.
We only plot the behavior of $\dh(\wh{\bPi}^{(t)}, \bPitrue)$
in the first $100$ iterations since
it remains almost the same in the following $400$ iterations.
}
\label{fig:fail_example}
\end{figure}

The underlying reason is the low stable rank $\srank{\bBtrue}$,
which is one in the single observation model.
In the next section, we will show that the ground truth
$\bBtrue$ can be obtained in an almost effortless way
once $\srank{\bBtrue}$ exceeds certain threshold.

\subsection{Discussion of related work}
\label{sec:single_obs_related_work}
This subsection discusses the relation between our estimator
and prior work focusing on the single observation model (e.g., $m=1$)~\citep{slawski2017linear, pananjady2018linear, hsu2017linear, abid2017linear}. In~\citet{pananjady2018linear}, the ML estimator is investigated, which is only computable for the special case $(m, p) = (1,1)$ and NP-hard for the rest cases.
Their estimator gets the same $\snr$ requirement as ours, namely,
$\log\snr\gsim \log n$.

To handle the computational issue of the ML estimator,
\citet{hsu2017linear} propose an approximation algorithm with
polynomial complexity. Their $\snr$ requirement is $\snr \geq c\min\bracket{1, p/\log \log n}$,
which has a gap with the mini-max lower bound.
In addition, they focus on the
recovery of $\bBtrue$ rather than the permutation matrix $\bPitrue$.

Later,~\citet{slawski2017linear} study the problem
for the viewpoint of denoising. By putting a
sparse constraint on $h$ (e.g., the number of permuted rows),
they view the term
$(\bI - \bPitrue)\bX\bbetatrue$ as an additive sparse outlier.
After obtaining an estimate of $\bbetatrue$, they restore
the permutation information with a linear assignment algorithm.
Compared with Algorithm~\ref{alg:one_step_estim}, the estimator
in~\citet{slawski2017linear} has a more stringent requirement
on $h$ (i.e., $h\ll \nfrac{n}{\log n}$) but at the same time enjoys a broader class for $\bbetatrue$, to put more specific, they
allow $p>1$ under the single observation model.

A parallel line of work can be found in~\citet{abid2017linear},
where a consistent estimator is proposed based on the method-of-moments.
However, their analysis is only for the special case
$(m, p) = (1, 1)$ and focuses on $\fnorm{\wh{\bB}-\bBtrue}$.
Direct comparison with ours can be difficult.

\section{Multiple Observations Model}\label{sec:multi_observe}

The previous section considers the
single observation model ($m=1$) and this section extends
the discussion to the multiple observations model ($m > 1$).
We $(i)$ discover a much richer
behavior inherent in our estimator and $(ii)$ show that correct permutation can be obtained
with $\snr$ being a positive constant.

\begin{figure}[h]	
\centering
\includegraphics[width = 5.2in]{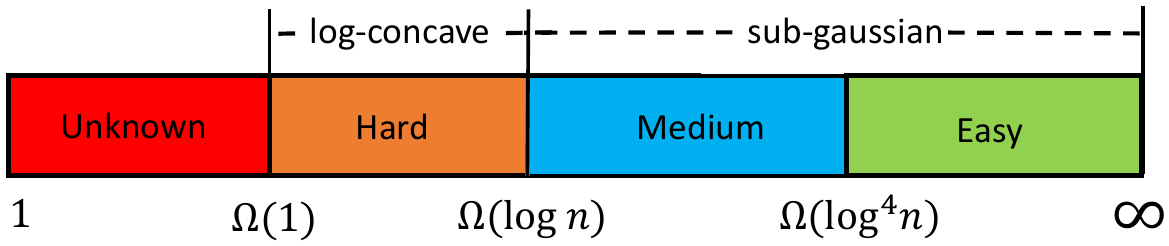}
\caption{Illustration of different regimes of the permutation recovery.
In \textbf{regime hard}, we require $\snr \gsim n^{\frac{c}{\srank{\bBtrue}}}$;
in \textbf{regime medium difficulty}, we require
$\snr \gsim \log^{c} n$; and
in \textbf{regime easy}, we require $\snr \geq c$.}
\label{fig:regime_illustrate}
\end{figure}

Based on the requirement on $\snr$ for correct permutation recovery,
we divide the estimator's performance into four
regimes in terms of $\srank{\bBtrue}$, which is illustrated in Figure~\ref{fig:regime_illustrate}.

\begin{itemize}
\item
\textbf{Unknown regime: $\srank{\bBtrue} \in [1, c_0)$}.
In this regime, our estimator's behavior remains a mystery. To the best of our knowledge,
no estimator's performance has been thoroughly studied in this regime.

\item
\textbf{Hard regime: $\srank{\bBtrue} \in [c_0, c_1 \log n)$}.
We assume $\bX_{ij}$ to be i.i.d. sub-gaussian random variables with log-concavity.
For the correct permutation recovery, we require $\snr$ satisfying $\snr \gsim n^{\frac{c}{\srank{\bBtrue}}}$.

\item
\textbf{Medium regime: $\srank{\bBtrue} \in [c_1\log n, c_2 \log^4 n)$}.
We only need $\bX_{ij}$ to be i.i.d. sub-gaussian random variables without enforcing log-concavity.
In addition, we require $\snr$ satisfying $\snr \gsim \log^c n$ to obtain the ground truth permutation matrix.

\item
\textbf{Easy regime: $\srank{\bBtrue} \in [c_2 \log^4 n, \infty)$}.
Still, we only assume $\bX_{ij}$ to be i.i.d. sub-gaussian random variables.
Here, we can relax the requirement on $\snr$ to be above some positive constant
for correct permutation recovery.
\end{itemize}

The formal statement of results is put in Theorem~\ref{thm:multi_snr_require_general_sub_gauss}
and Theorem~\ref{thm:multi_snr_require_log_concave}.

\begin{theorem}
\label{thm:multi_snr_require_general_sub_gauss}
Consider the sensing matrix $\bX$ with its entries
$\bX_{i, j}$ being sub-gaussian RV with zero mean and unit variance
$(1\leq i \leq n, 1\leq j \leq p)$.
Assuming that $(i)$ $n\gg p\cdot \log^3 n\cdot \log^2(n^2 p^3)$
and $(ii)$ $h\leq c_0\cdot  n$,
we can reliably obtain the permutation matrix with Algorithm~\ref{alg:one_step_estim}, namely,
$\Prob(\bPiopt = \bPitrue) \geq 1 - c_1 n^{-c_2} -c_3p^{-c_4}$,
in the following situations:
\begin{itemize}
\item
\textbf{Easy regime $(\srank{\bBtrue} \gg \log^4 n)$}: we have $\snr \geq c$;

\item
\textbf{Medium regime $(\log n \ll \srank{\bBtrue} \ll \log^4 n)$}:
we have $\log \snr \gsim \log \log n$.

\end{itemize}

\end{theorem}

Note that the above result only applies to the easy and medium regime, i.e.,
$\srank{\bBtrue} \gg \log n$. By enforcing additional
constraints, to put more specifically, $\bX_{ij}$ is log-concave,
we can generalize the above result to the hard regime, which is formally stated as

\begin{theorem}
\label{thm:multi_snr_require_log_concave}
Consider the sensing matrix $\bX$ with its entries
$\bX_{i, j}$ being log-concave sub-gaussian RV with zero mean and unit variance
$(1\leq i \leq n, 1\leq j \leq p)$.
Assuming that $(i)$ $n\gg p^{1 + \varepsilon}\cdot\log^{3(1 +\varepsilon)} n\cdot \log^{2(1+\varepsilon)}(n^2 p^3)$
and  $(ii)$ $h\leq c_0 \cdot n$,
we can reliably obtain the permutation matrix with Algorithm~\ref{alg:one_step_estim}, namely,
$\Prob(\bPiopt = \bPitrue) \geq 1 - c_1 n^{-c_2} -c_3p^{-c_4}$,
in the following situations:
\begin{itemize}
\item
\textbf{Easy regime $(\srank{\bBtrue} \gg \log^4 n)$}: we have $\snr \geq c$;

\item
\textbf{Medium regime $(\log n \ll \srank{\bBtrue} \ll \log^4 n)$}:
we have $\log \snr \gsim \log \log n$;

\item
\textbf{Hard regime $(c(\varepsilon)\lsim \srank{\bBtrue} \ll \log n)$}:
we have
\begin{align}
\label{eq:multi_snr_require_logconcave}	
\log \snr \gsim \frac{\log n}{\srank{\bBtrue}} + \log \log n,
\end{align}
where $\varepsilon > 0$ is an arbitrary positive constant and
$c(\varepsilon)$ is a positive constant depending on $\varepsilon$.
\end{itemize}
\end{theorem}
Family of log-concave sub-gaussian distributions include
standard Gaussian distribution and uniform distribution among $[-c,~c]$, where
$c$ denotes a certain positive constant.
However, not all sub-gaussian RVs are with log-concavity.
One example is that $\bX_{ij}$ is a Rademacher RV, i.e.,
$\Prob(\bX_{ij} = \pm 1) = \nfrac{1}{2}$.
We can see that
Theorem~\ref{thm:multi_snr_require_log_concave} requires log-concavity and
can allow a much broader range of
$\srank{\bBtrue}$ than Theorem~\ref{thm:multi_snr_require_general_sub_gauss}, to put more specific, a positive constant $\srank{\bBtrue}$ is permitted.
In the following,
we will see some numerical results
implying that the log-concavity may be inseparable
from the positive constant $\srank{\bBtrue}$,
in other words, $\srank{\bBtrue}$ cannot be
$O(1)$ if the log-concavity assumption is violated.

\subsection{Results discussion}
\paragraph{Comparison with single observation model.}
The most noticeable implications of our theorems are
that the $\snr$ requirement for correct permutation recovery
can be greatly reduced by making multiple measurements, i.e., from $\Omega(n^{c})$ to
$\Omega(n^{\frac{c\cdot \srank{\bBtrue}}{n}}\vcup \log^c n)$ or even $\Omega(1)$.

Additionally, we notice that the constraint on signal length $p$ has been removed
when multiple measurements are made. According to Theorem~\ref{thm:warm_up} which
concerns the single observation model $(m = 1)$,
our estimator only works when $p = 1$.
While in Theorem~\ref{thm:multi_snr_require_general_sub_gauss} and Theorem~\ref{thm:multi_snr_require_log_concave} where $\srank{\bBtrue} \gg 1$,
a much broader range of $p$ is permitted.

In summary, the take-home message is that
\textbf{diversity, i.e., large $\srank{\bBtrue}$,
helps in the permutation recovery}.

\paragraph{Statistical optimality.}
First, we argue that our estimator in Algorithm~\ref{alg:one_step_estim} is almost
mini-max optimal w.r.t. $\snr$ requirement.
For an easy comparison, we consider
the special case where $\srank{\bBtrue}$
is same as $\rank(\bBtrue)$, which
corresponds to the situation where
$\bBtrue$'s signal strength is uniformly spread
among all its eigenvalues. Theorem~\ref{thm:multi_statis_lb}
suggests that wrong permutation matrix $\wh{\bPi}$
will be returned with a high probability if
\[
\srank{\bBtrue} \cdot  \log\snr  \approx \
\srank{\bBtrue} \cdot \log(1 + \snr)
\lsim \log n,
\]
which holds regardless of the estimator form.
Comparing with Theorem~\ref{thm:multi_snr_require_log_concave}, we conclude that
our estimator has optimal $\snr$ requirement in the following two
regimes: $(i)$ $\srank{\bBtrue} \lsim \frac{\log n}{\log\log n}$
and $(ii)$ $\srank{\bBtrue} \gg \log^4 n$. Only in the
regime $\nfrac{\log n}{\log \log n} \ll \srank{\bBtrue} \ll \log^4 n$
our estimator experiences a loss in the $\snr$ requirement, which is up to $O(\log^c n)$.

Second, our requirement on sample number $n$ is almost optimal:~\cite{unnikrishnan2015unlabeled} claims
that $n\gsim p$ is required for correct permutation recovery while our estimator
only needs $n\gg p\cdot \log^3 n \cdot \log^2(n^2p^3)$.
In addition, we have almost minimal constraint on permuted rows' number, i.e.,
we allow $h_{\textup{max}}\asymp n$, where
 $h_{\textup{max}}$ is the maximum
allowed number of permuted rows.

\subsection{Proof outline}
To make the proof more digestible, we first outline the proof strategies
and technical challenges before delving into the technical details.
The rigorous proof is attached in the supplementary material, including all supporting Lemmas.
Here, we would like to explain the main technical challenges in the proof of
Theorem~\ref{thm:multi_snr_require_general_sub_gauss} and Theorem~\ref{thm:multi_snr_require_log_concave}, which lies in the proof that
\begin{align}
\label{eq:error_event}
\set{\la \bPi, \bY \bY^{\rmt}\bX\bX^{\rmt} \ra \geq  \langle \bPitrue, \bY \bY^{\rmt}\bX\bX^{\rmt}\rangle,~~\exists~\bPi \neq \bPitrue}
\end{align}
holds with probability near zero given the assumptions in Theorem~\ref{thm:multi_snr_require_general_sub_gauss} and Theorem~\ref{thm:multi_snr_require_log_concave}. Two noticeable challenges in analyzing the above events are
$(i)$ combinatorial nature of the optimization problem; and $(ii)$ high-order moments of $\bX$
in the product $\bY\bY^{\rmt}\bX\bX^{\rmt}$. To address these challenges, we construct our solutions
with two building blocks: $(i)$ relaxations of error event and $(ii)$ modified leave-one-out techniques.
The following context presents a more detailed explanation.

To begin with, we define $\wh{\bB}$ and $\wt{\bB}$ respectively as
\[
\wt{\bB} &= \bracket{n-h}^{-1}\bX^{\rmt}\bPitrue\bX\bBtrue, \\
\wh{\bB} &= \bracket{n-h}^{-1}\bX^{\rmt}\bY = \wt{\bB} + \
\bracket{n-h}^{-1}\bX^{\rmt}\bW,
\]
where $h$ is denoted as the Hamming distance between identity matrix $\bI$ and
the ground truth permutation matrix $\bPitrue$,
i.e., $h = \dH(\bI, \bPitrue)$.

\paragraph{Stage I. Relaxation of error event.}
To combat the combinatorial nature of \eqref{eq:error_event},
we first relax the error event to make it amenable for analysis.
Associated with different regimes comes
different forms of relaxations.

\begin{itemize}
\item
\textbf{Easy regime}.
We relax the error event $\{\bPiopt\neq \bPitrue\}$ as
\begin{align}
\label{eq:easy_error_event}	
\set{ \la \bY_{i, :},  \wh{\bB}^{\rmt}\bX_{\pi^{\natural}(i), :} \ra
\leq \la \bY_{i, :}, \wh{\bB}^{\rmt} \bX_{j, :} \ra,~\exists~1\leq \pi^{\natural}(i) \neq  j \leq n}.
\end{align}
With this relaxation method, we find the constraint
$\srank{\bBtrue} \gg \log^2 n$ to be inevitable. However, as compensation, we can show $\snr\geq c$ is sufficient for the correct permutation reconstruction.

\item
\textbf{Medium $\&$ hard regime}.
To get rid of the constraint on $\srank{\bBtrue}$, we first exploit the energy-preserving property permutation matrix, i.e., $\Fnorm{\bPi \bM} = \Fnorm{\bM}$ where $\bM \in \RR^{n\times (\cdot)}$ is an arbitrary matrix.
Then we adopt the relaxation
\begin{align}
\label{eq:medium_hard_error_event}	
\set{\big\|\bY_{i, :} - \wh{\bB}^{\rmt}\bX_{\pi^{\natural}(i), :} \big\|_{2}^2 \geq \big\|\bY_{i, :} - \wh{\bB}^{\rmt}\bX_{j, :} \big\|_{2}^2,~~\exists~i, j}.
\end{align}
In this way, we can relax the constraint on $\srank{\bBtrue}$ from
$\srank{\bBtrue} \gg \log^2 n$ to $\srank{\bBtrue}\gg \log n$ if
$\bX_{ij}$ is i.i.d. isotropic sub-gaussian RV. If the log-concavity assumption is further put on $\bX_{ij}$'s distribution, we can relax the constraint to $\srank{\bBtrue}\gg c$.

However, compared with the above relaxation method, this method will experience some
loss in the $\snr$ requirement. For instance, in the \textbf{Easy Regime},
this approach will require $\snr$ to satisfy  $\snr \gsim \log^c n$
for the correct permutation recovery while the previous approach only needs
$\snr\geq c$.
\end{itemize}

Bonuses of the above relaxations include $(i)$ we reduce the moments of $\bX$
from fourth order in \eqref{eq:error_event} to third order, see \eqref{eq:easy_error_event} and \eqref{eq:medium_hard_error_event}; and $(ii)$ the dependences between $\wh{\bB}$ and other
terms in \eqref{eq:easy_error_event} and \eqref{eq:medium_hard_error_event} are restricted to the rows $\bX_{\pi^{\natural}(i), :}$ and $\bX_{j, :}$. The latter bonus is essential for the leave-one-out technique.

\paragraph{Stage II. Dependence decoupling with leave-one-out technique.}
To further reduce the moments of $\bX$ involved in the analysis,
we  modify the \emph{leave-one-out}
technique~\citep{karoui2013asymptotic, el2013robust, karoui2018impact, chen2020noisy, sur2019likelihood, zhang2020optimal}.

As mentioned above, $\wh{\bB}$ is weakly correlated with
rows $\bX_{\pi^{\natural}(i), :}, \bX_{j, :}$,
$\bY_{i, :}$, and $\bY_{j, :}$.
The basic idea of the leave-one-out technique is to
replace the correlated rows in $\wh{\bB}$ with their i.i.d. copies, by which the dependence is decoupled.
Since only a limited number of rows are replaced, we expect its statistical properties should remain almost identical. A detailed explanation comes as follows.

First, we draw an independent copy $\bX_{s, :}^{'}$ for each row
$\bX_{s, :}$ ($s$th row of the sensing matrix $\bX$). With these
independent copies, we construct leave-one-out samples $\wtminus{\bB}{s}$ $(1\leq s \leq n)$
by replacing the $s$th row in $\wt{\bB}$ with its independent copy $\bX_{s, :}^{'}$.
In formulae:
\[
\wtminus{\bB}{s} =
(n-h)^{-1}\bigg(
\sum_{\substack{k \neq s \\ \pi^{\natural}(k) \neq s}} \bX_{\pi^{\natural}(k), :}\bX_{k, :}^{\rmt}
+ \sum_{\substack{k = s \textup{ or} \\  \pi^{\natural}(k) = s}} {\bX}^{'}_{\pi^{\natural}(k), :}{\bX}^{'\rmt}_{k, :}\bigg)\bBtrue.
\]
Easily we can verify that $\wtminus{\bB}{s}$ is independent of
$\bX_{s, :}$. Similarly, we construct matrices
$\{\wtminus{\bB}{s, t}\}_{1\leq s\neq t \leq n}$ as
\[
\wtminus{\bB}{s, t} =
(n-h)^{-1}\bigg(
\sum_{\substack{k \neq s, t\\ \pi^{\natural}(k) \neq s,t }} \bX_{\pi^{\natural}(k), :}\bX_{k, :}^{\rmt}
+ \sum_{\substack{k = s \textup{ or } k = t \textup{ or} \\ \pi^{\natural}(k) =s
  \textup{ or }\pi^{\natural}(k) = t}} {\bX}^{'}_{\pi^{\natural}(k), :}{\bX}^{'\rmt}_{k, :}\bigg)\bBtrue,
\]
and can verify the independence between $\wtminus{\bB}{s, t}$ and
the rows $\bX_{s, :}$ and $\bX_{t, :}$.

\newpage

\subsection{Discussion of related work}
\label{sec:multi_obs_related_work}
This subsection compares our results with
the previous works on multiple observations model
($m > 1$)~\citep{pananjady2017denoising, zhang2022permutation, slawski2020two}.
\citet{pananjady2017denoising} consider a similar
setting as ours while their focus is on
the product $\bPitrue\bX\bBtrue$ rather than the individual values of
$\bPitrue$ and $\bBtrue$.

In~\citet{zhang2022permutation}, statistical
limits w.r.t. the $\snr$ are presented and
the ML estimator is revisited under the setting of multiple observations model.
It is suggested that the ML estimator can reach the statistical limits
under certain regimes. In addition, a projected gradient descent-based algorithm with monotonic descent property is proposed for practical use.
However, whether this algorithm will yield the ground truth $\bPitrue$ still remains a mystery.

In~\citet{slawski2020two}, they follow a similar
idea of~\citet{slawski2017linear} and take the
viewpoint of denoising for permutation recovery. Assuming only a limited proportion of rows are permuted, they first
obtain an estimate of $\bBtrue$ and
plug it
into the ML estimator in \eqref{eq:optim_estim_pi} to reconstruct $\bPitrue$.
Compared with our estimator in Algorithm~\ref{alg:one_step_estim}, their estimator can be applied to a broader
family of matrices (i.e., selection matrices);
but $(i)$ allow a smaller number of permuted rows and $(ii)$ require a larger sample number $n$.
A detailed comparison is referred to Table~\ref{tab:compare}.

\section{Simulation Results}
\label{sec:simul}
This section presents the numerical results. Since our estimator cannot guarantee the correct permutation recovery under the single observation model, our simulations focus on the multiple observations model, i.e., $m > 1$.

We investigate the impact of the $\nfrac{p}{n}$ ratio and
number of permuted rows $h$ on the permutation recovery
when $\bX_{ij}$ is $(i)$ Gaussian distributed,
$(ii)$ uniformly distributed within $[-1, 1]$, and
$(iii)$ Rademacher distributed such that
$\Prob(\bX_{ij} = \pm 1) = \nfrac{1}{2}$.

First, we present our experiment settings. We set the $i$th column
$\bBtrue_{:, i}$ $(1\leq i \leq \min(m,p))$
to be the $i$th canonical basis, which has $1$ on the
$i$th entry and $0$ elsewhere.
One benefit of this setting is that
the stable rank $\srank{\bBtrue}$
can be easily calculated, i.e., $m\vcap p$.
In the following context, we will use $\srank{\bBtrue}$ and
$m$ $(m < p)$ interchangeably.

\begin{figure}[!h]

\begin{center}
\mbox{
\includegraphics[width=2.2in]{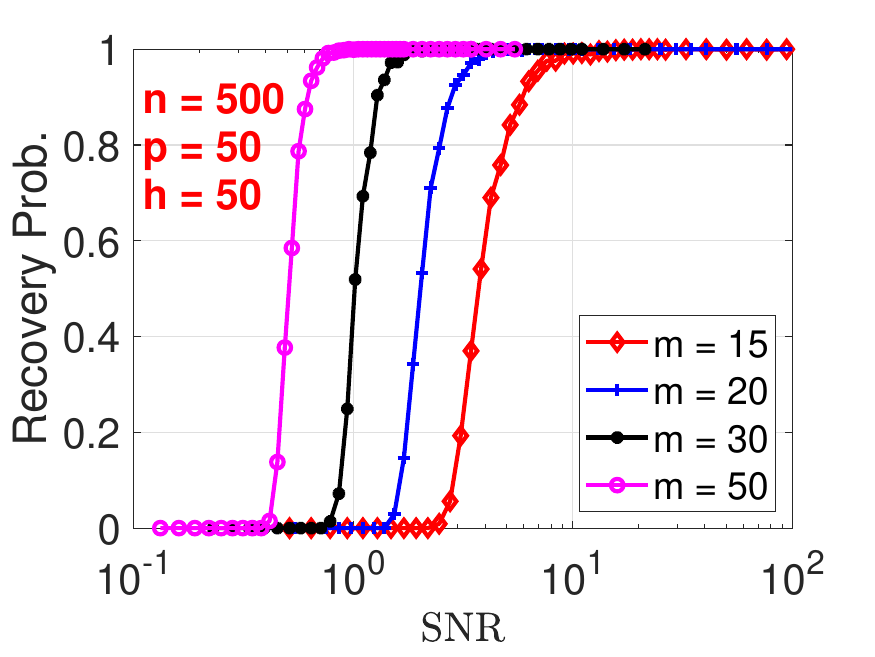}
\includegraphics[width=2.2in]{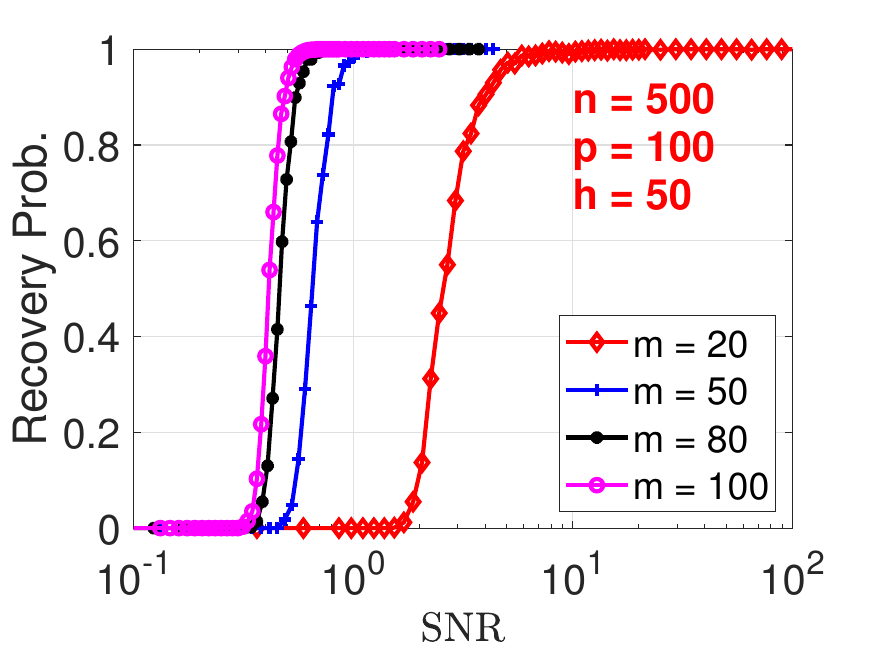}	
\includegraphics[width=2.2in]{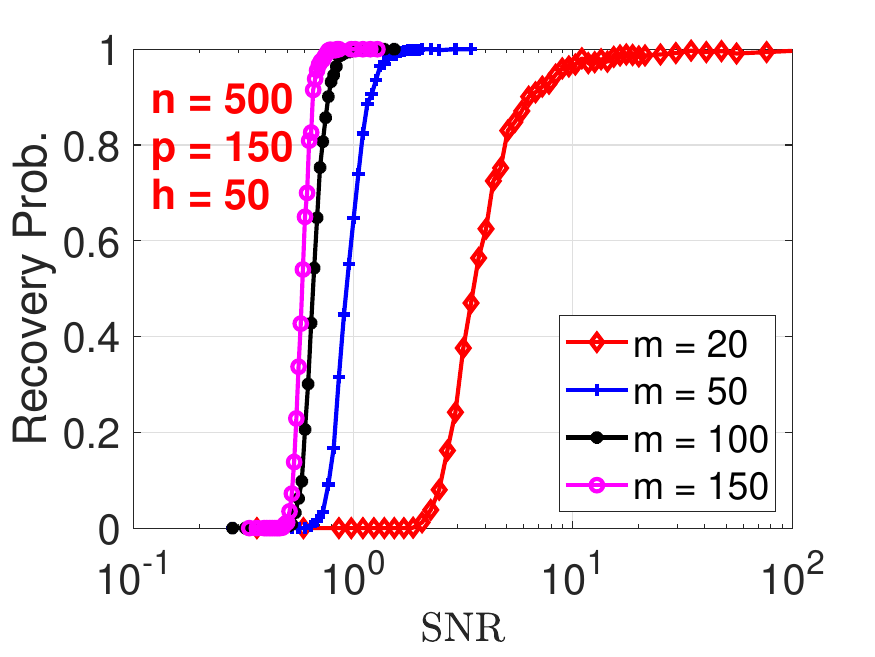}
}

\mbox{	
\includegraphics[width=2.2in]{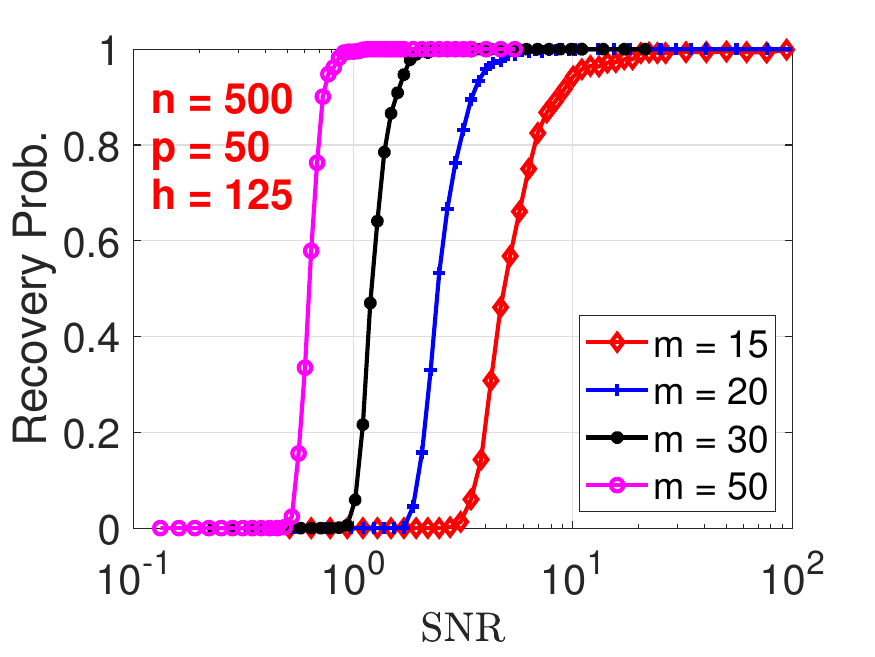}
\includegraphics[width=2.2in]{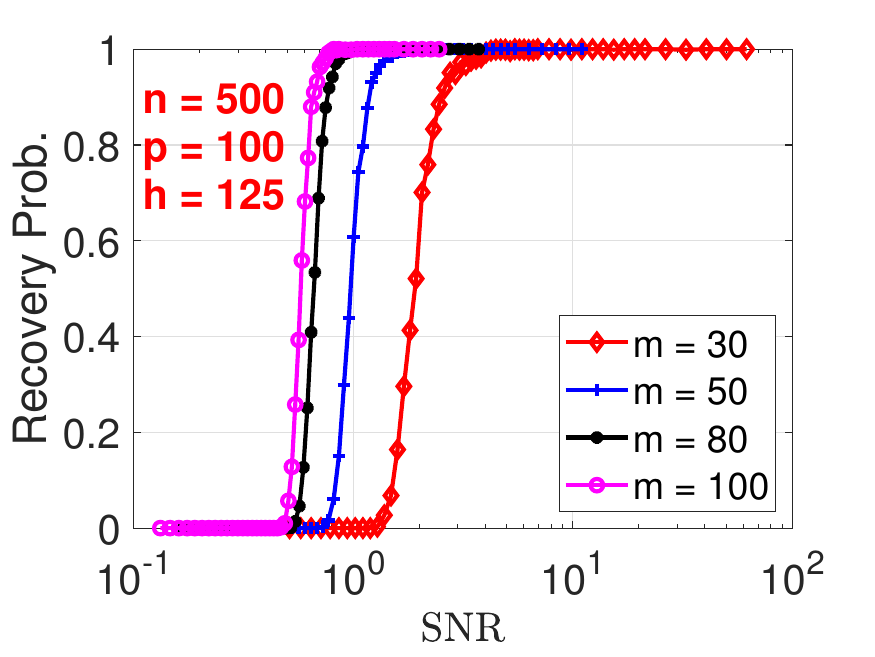}	
\includegraphics[width=2.2in]{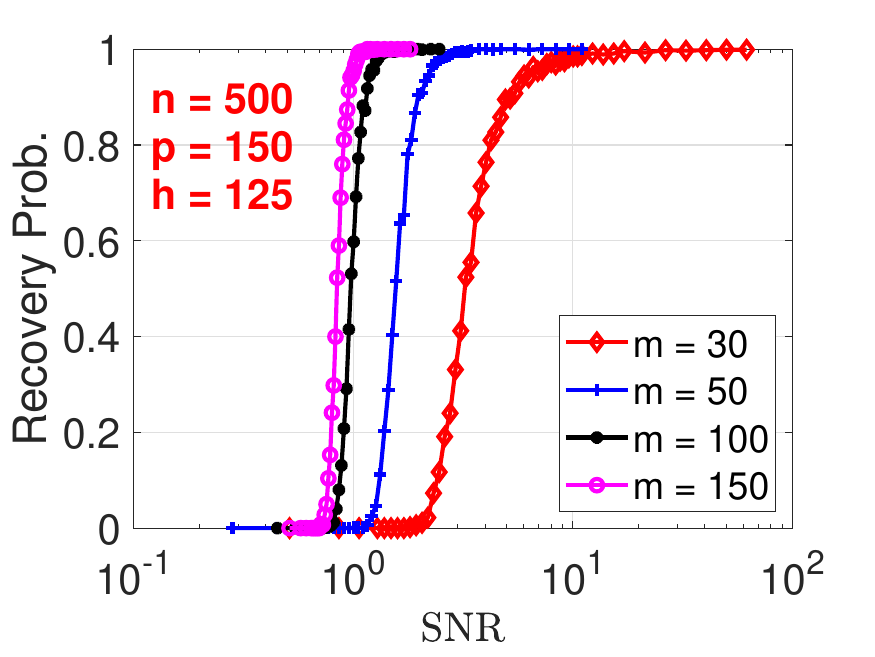}	
}

\mbox{
\includegraphics[width=2.2in]{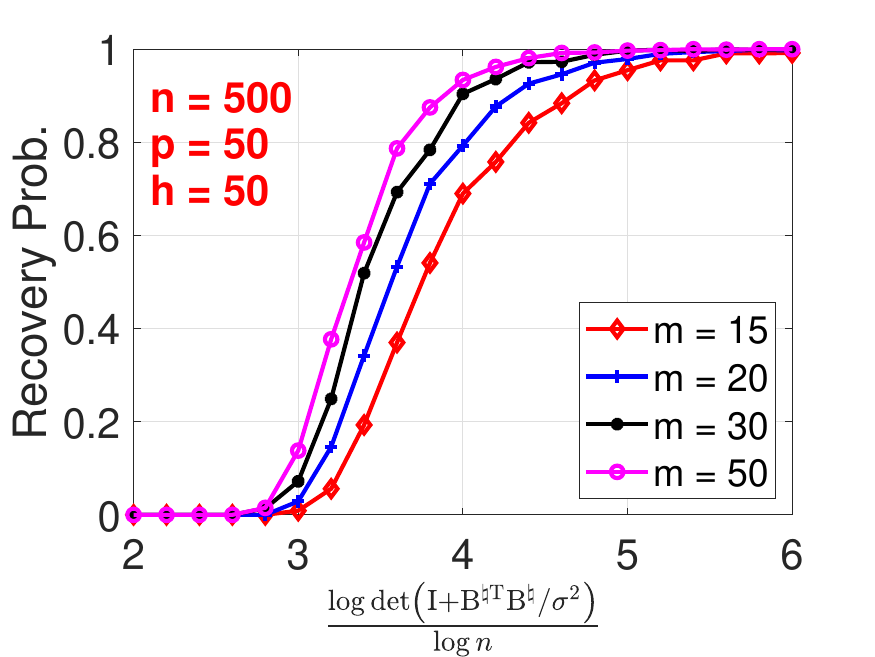}
\includegraphics[width=2.2in]{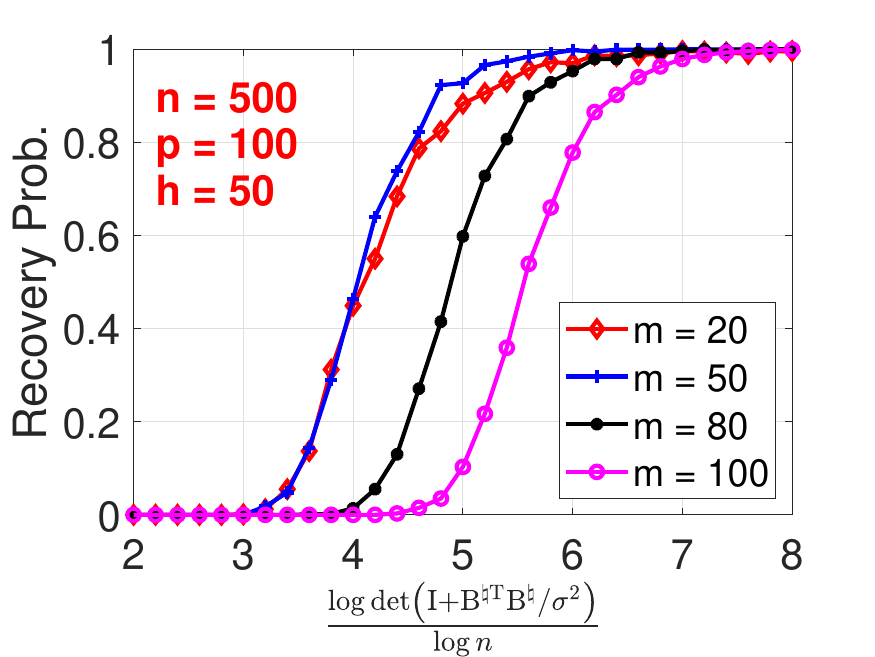}	
\includegraphics[width=2.2in]{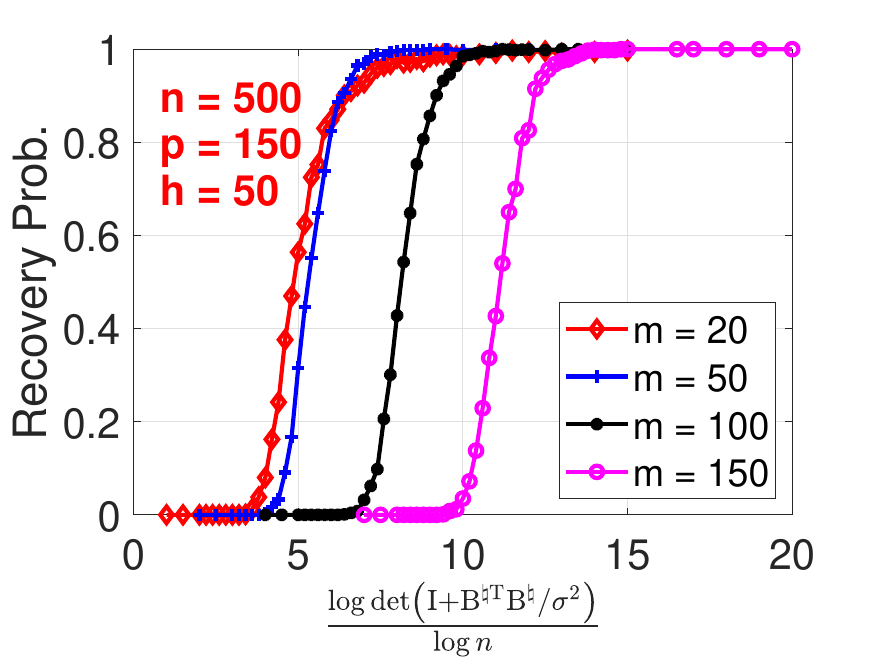}	
}

\mbox{
\includegraphics[width=2.2in]{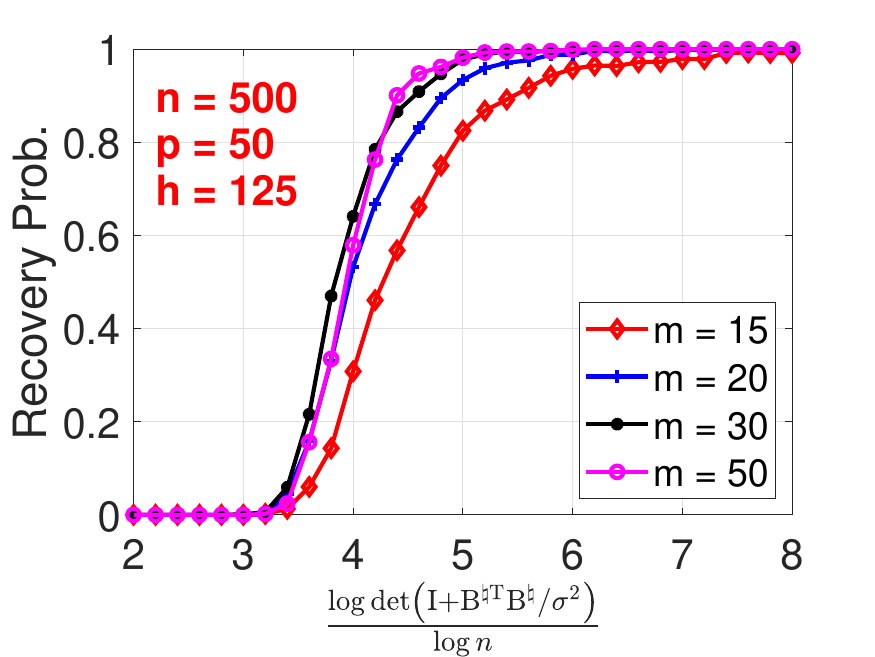}
\includegraphics[width=2.2in]{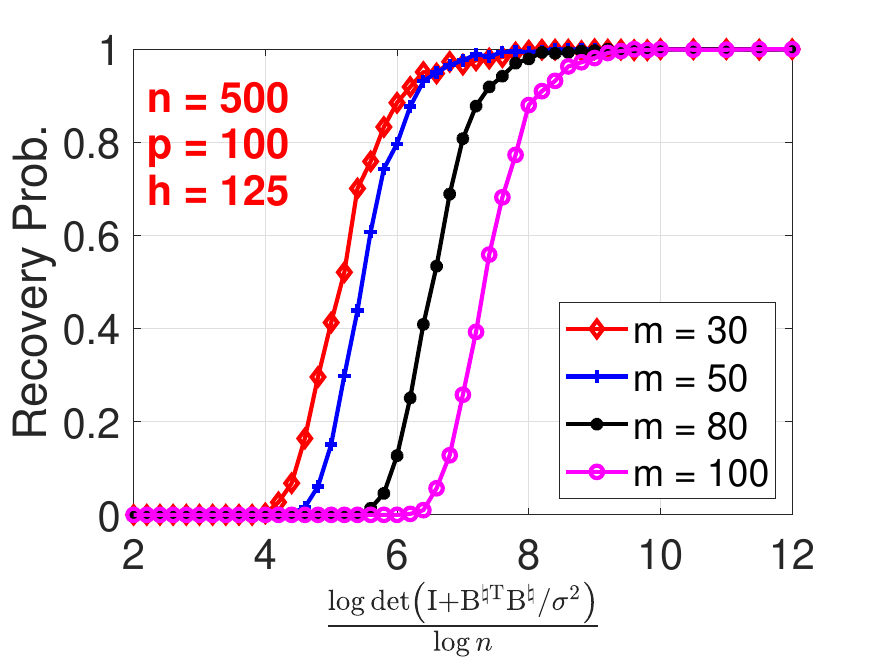}	
\includegraphics[width=2.2in]{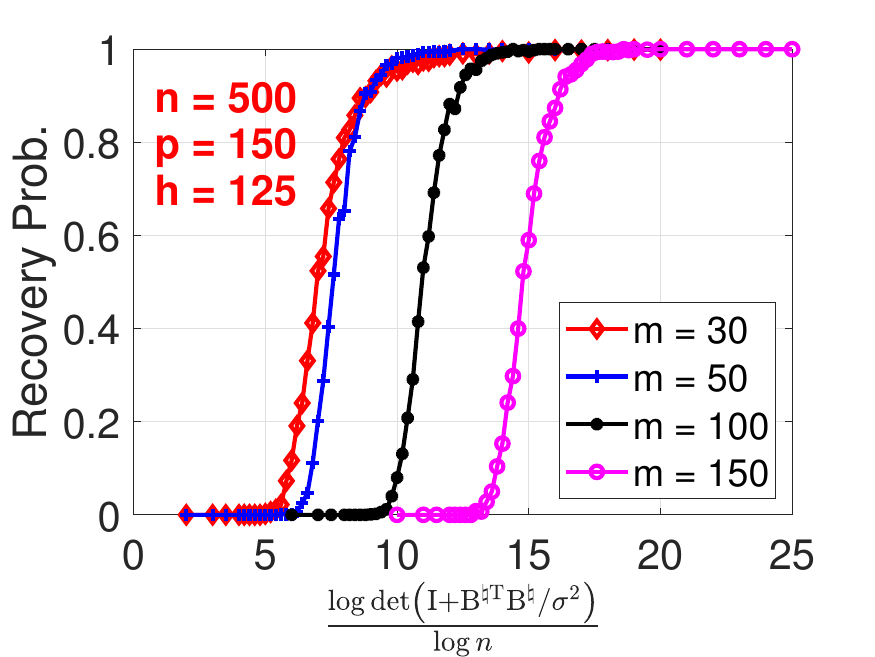}	
}

\end{center}
\caption{Simulated recovery rate $\Prob(\wh{\bPi} = \bPitrue)$, with $n = 500$, $p \in\{50, 100, 150\}$, and $h \in\{50, 125\}$,
versus $\snr$ (\textbf{upper panels}) and $\frac{\logdet\bracket{\bI + \bB^{\natural\rmt}\bBtrue/\sigma^2}}{\log n}$ (\textbf{lower panels}).}
\label{fig:recover_n500}
\end{figure}

\subsection{Gaussian distribution}
We assume $\bX_{ij}$ to be i.i.d. standard normal distribution $\normdist(0, 1)$.
We fix $\nfrac{p}{n}$ to be $\set{0.1, 0.2, 0.3}$ and $\nfrac{h}{n}$ to be
$\set{0.1, 0.25}$. We vary the sample number $n$ to be $\set{500, 1000, 2000, 4000}$.
The corresponding results are put in Figures~\ref{fig:recover_n500},~\ref{fig:recover_n1000},~\ref{fig:recover_n2000}, and~\ref{fig:recover_n4000}, respectively.
Apart from evaluating the permutation recovery w.r.t. $\snr$, we also evaluate it w.r.t.
$\nfrac{\logdet\bracket{\bI + \frac{\bB^{\natural\rmt}\bBtrue}{\sigma^2}}}{\log n}$.
This ratio appears in the statistical lower bounds in Theorem~\ref{thm:multi_statis_lb},
which claims
\[
\logdet\bigg(\bI + \frac{\bB^{\natural\rmt}\bBtrue}{\sigma^2}\bigg) \gsim \log n,
\]
should hold for correct permutation recovery.

First, we notice the numerical results are well aligned with our theoretical results,
say Theorem~\ref{thm:multi_snr_require_log_concave}, which concludes the optimality of our estimator.
Moreover, we find that smaller $\nfrac{p}{n}$ and $h$ facilitate the permutation recovery.
The former factor leads to more concentrated behavior of $\bX^{\rmt}\bY$ around its
means $\Expc\bX^{\rmt}\bY = (n-h) \bBtrue \parallel \bBtrue$; while the latter
factor brings more energy in $\Expc\bX^{\rmt}\bY$.

\begin{figure}[!h]
\begin{center}
\mbox{
\includegraphics[width=2.2in]{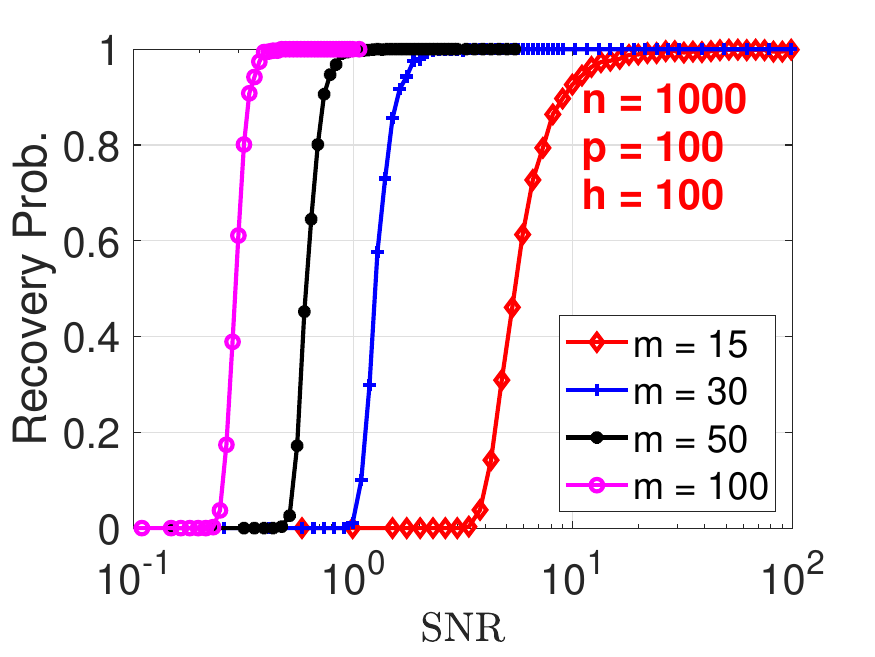}
\includegraphics[width=2.2in]{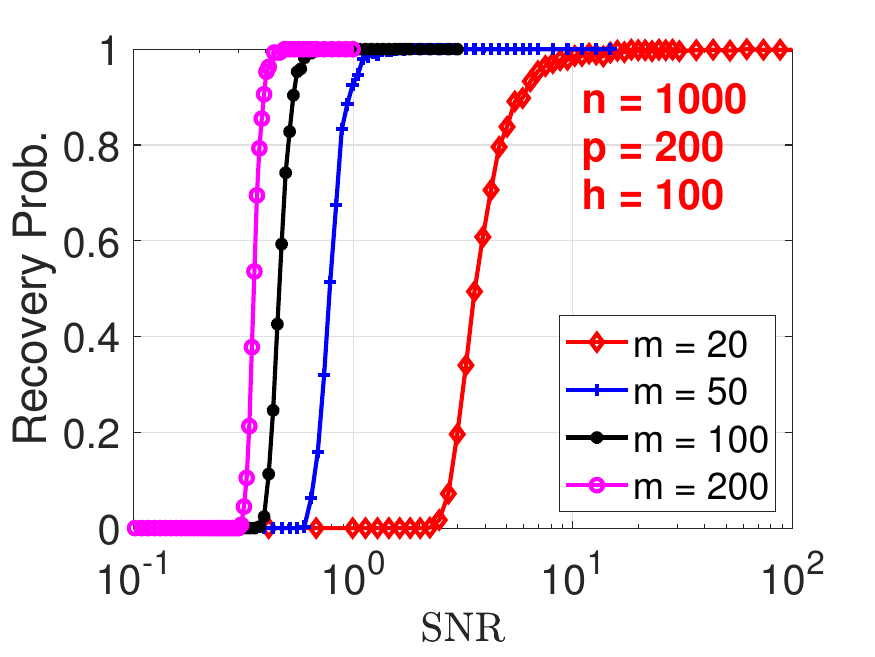}	
\includegraphics[width=2.2in]{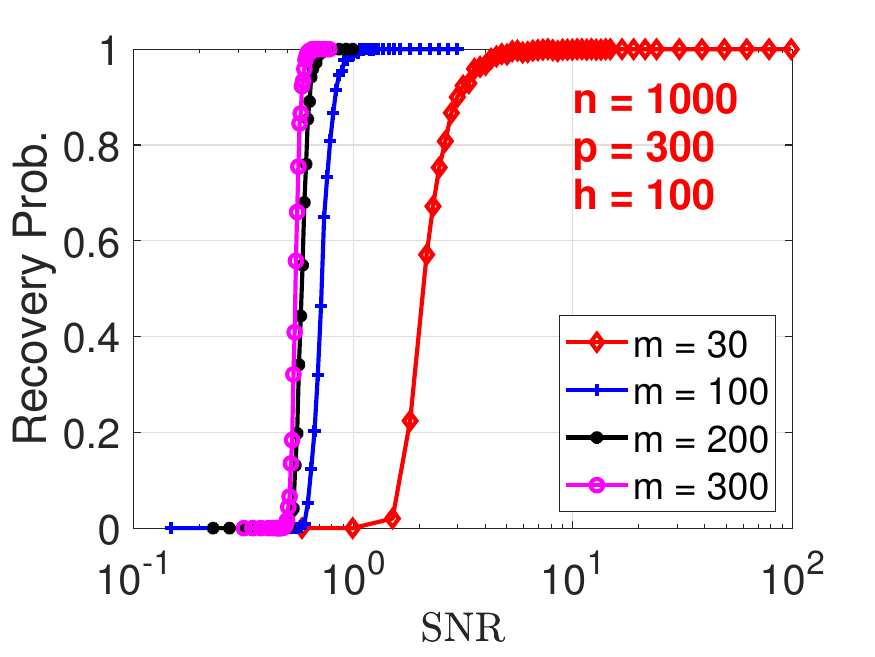}	
}

\mbox{
\includegraphics[width=2.2in]{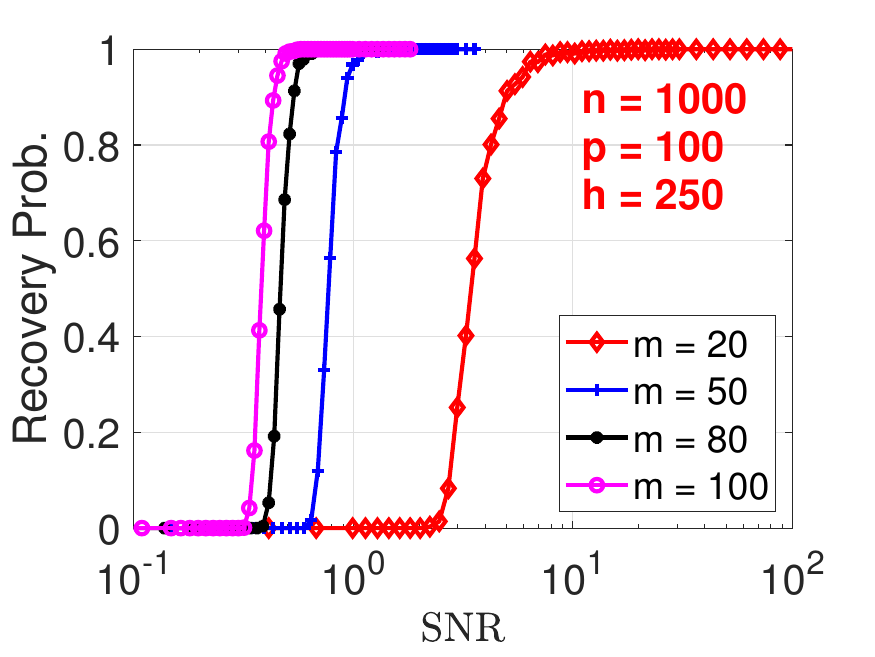}
\includegraphics[width=2.2in]{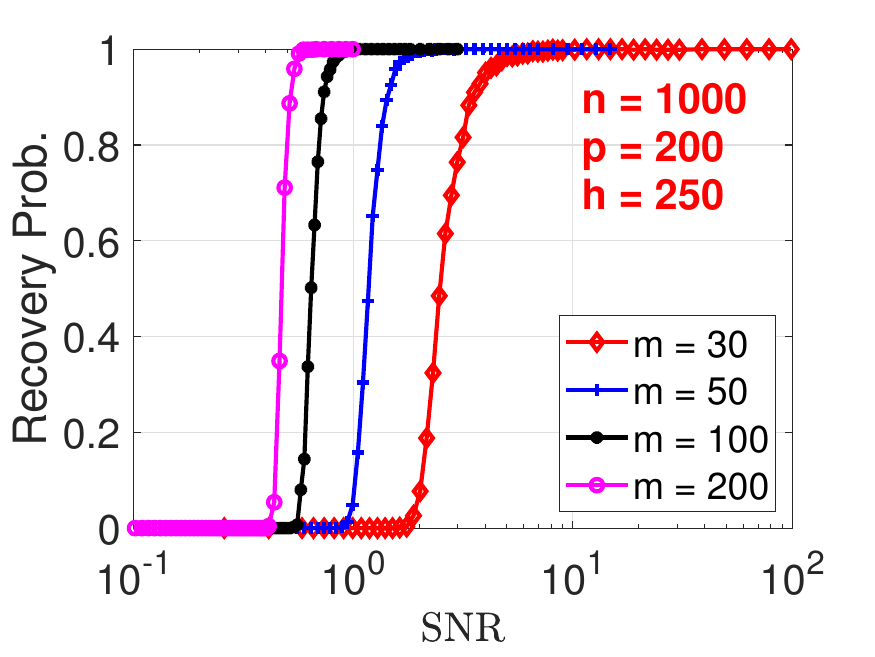}	
\includegraphics[width=2.2in]{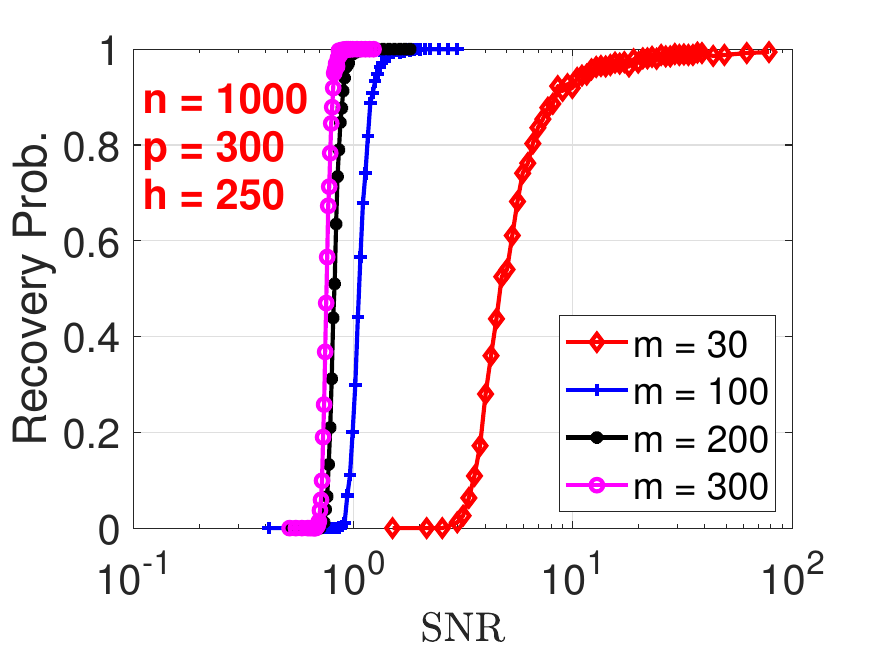}	
}

\mbox{
\includegraphics[width=2.2in]{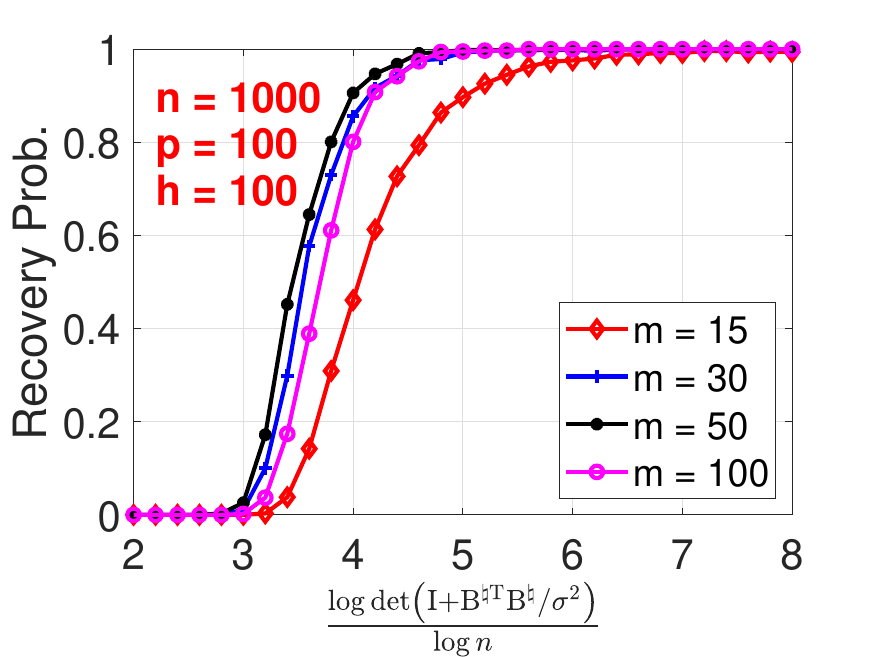}
\includegraphics[width=2.2in]{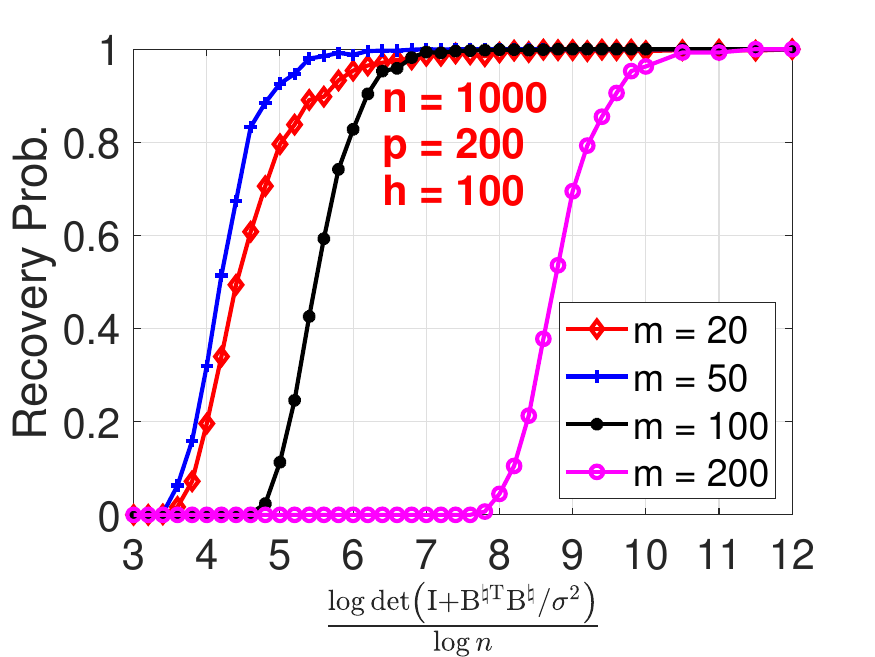}	
\includegraphics[width=2.2in]{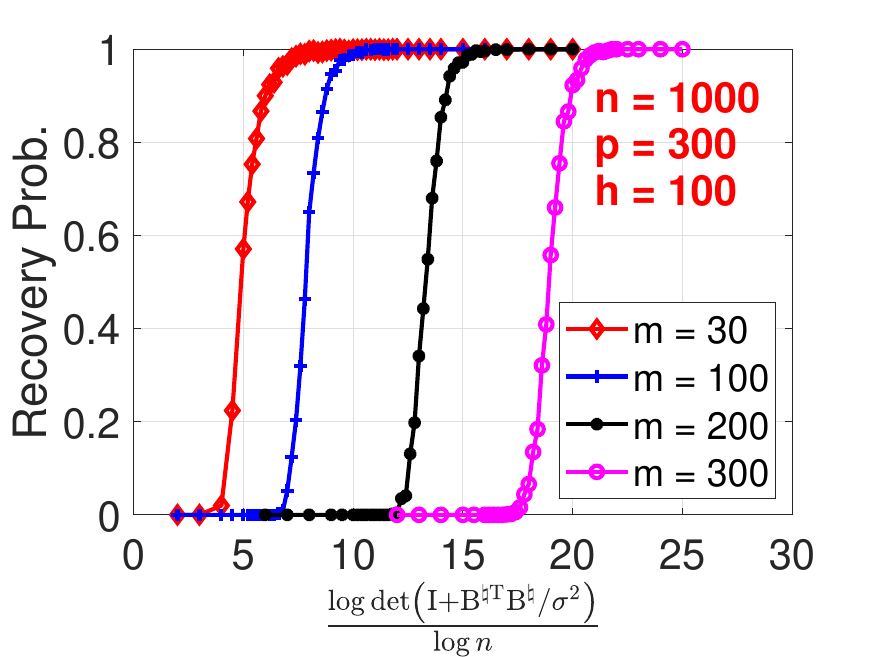}	
}

\mbox{
\includegraphics[width=2.2in]{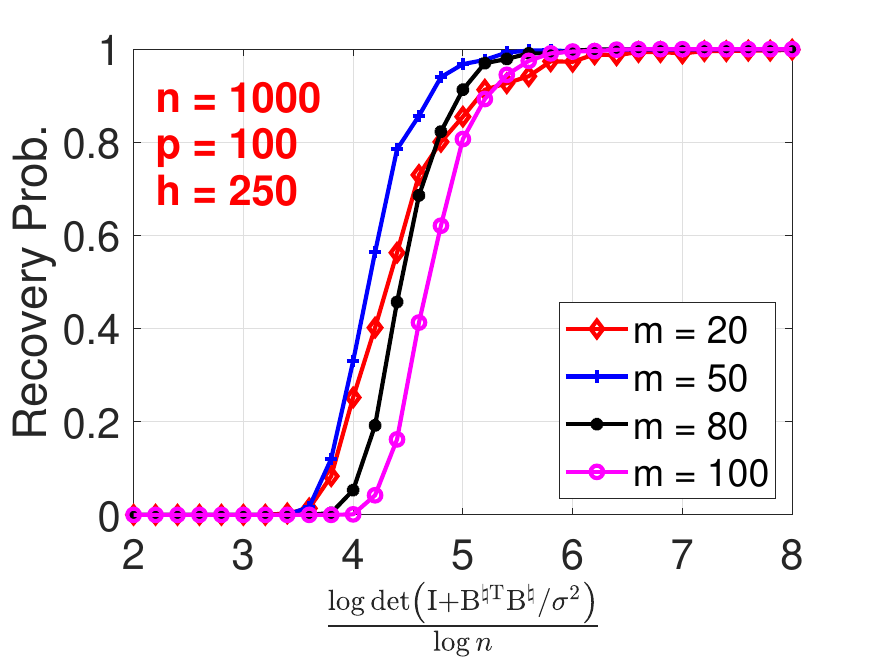}
\includegraphics[width=2.2in]{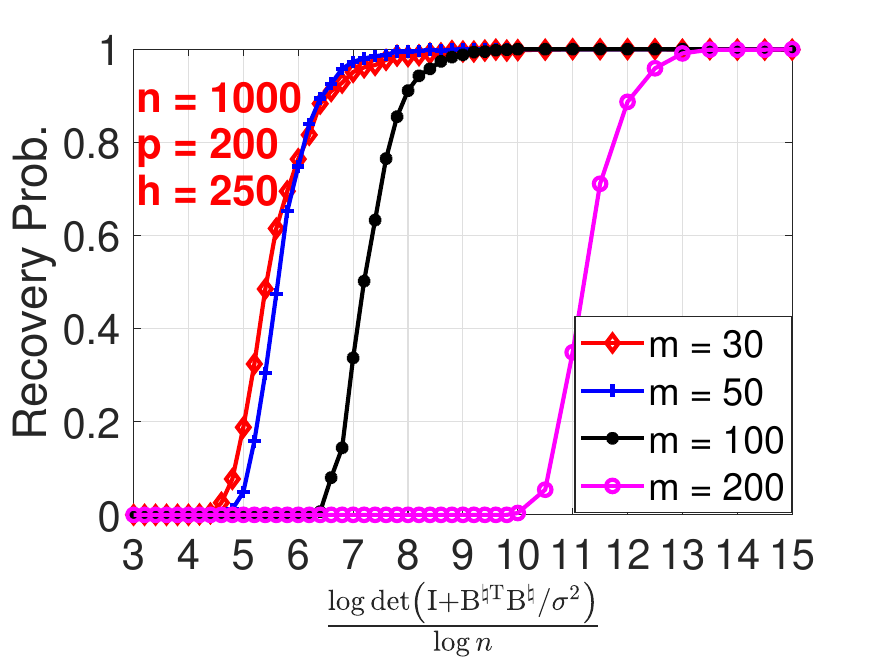}	
\includegraphics[width=2.2in]{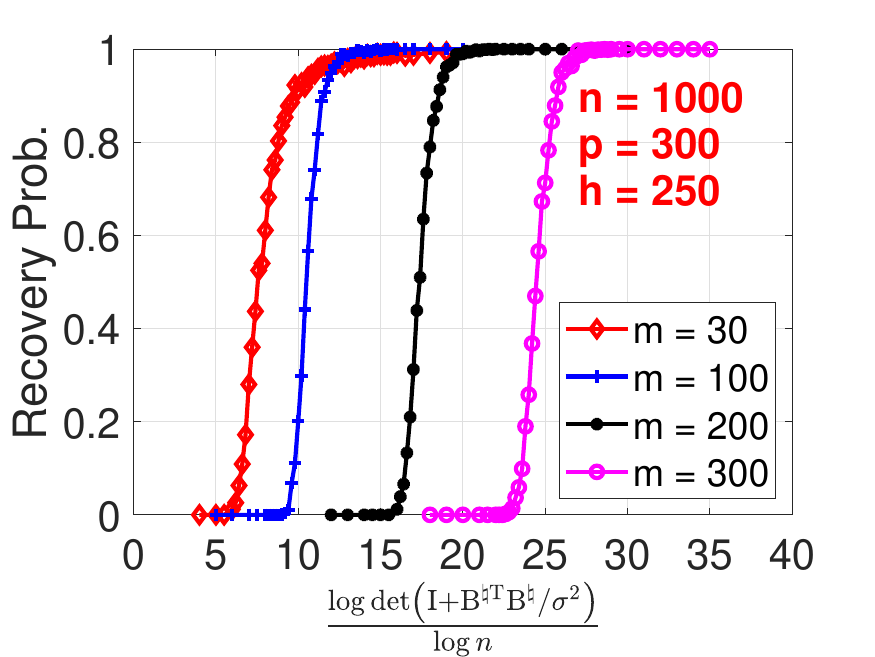}	
}

\end{center}
\caption{Simulated recovery rate $\Prob(\wh{\bPi} = \bPitrue)$, with $n = 1000$, $p\in\{100,200, 300\}$, and $h\in\{100,250\}$,
versus $\snr$ (\textbf{upper panels}) and $\frac{\logdet\bracket{\bI + \bB^{\natural\rmt}\bBtrue/\sigma^2}}{\log n}$ (\textbf{lower panels}).}
\label{fig:recover_n1000}
\end{figure}

\begin{figure}[!h]
\begin{center}
\mbox{
\includegraphics[width=2.2in]{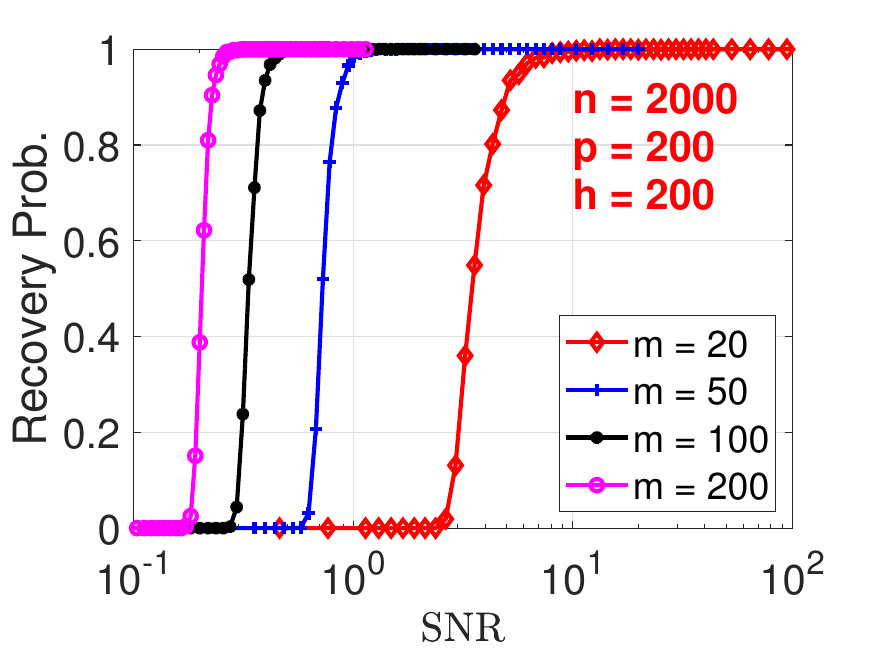}
\includegraphics[width=2.2in]{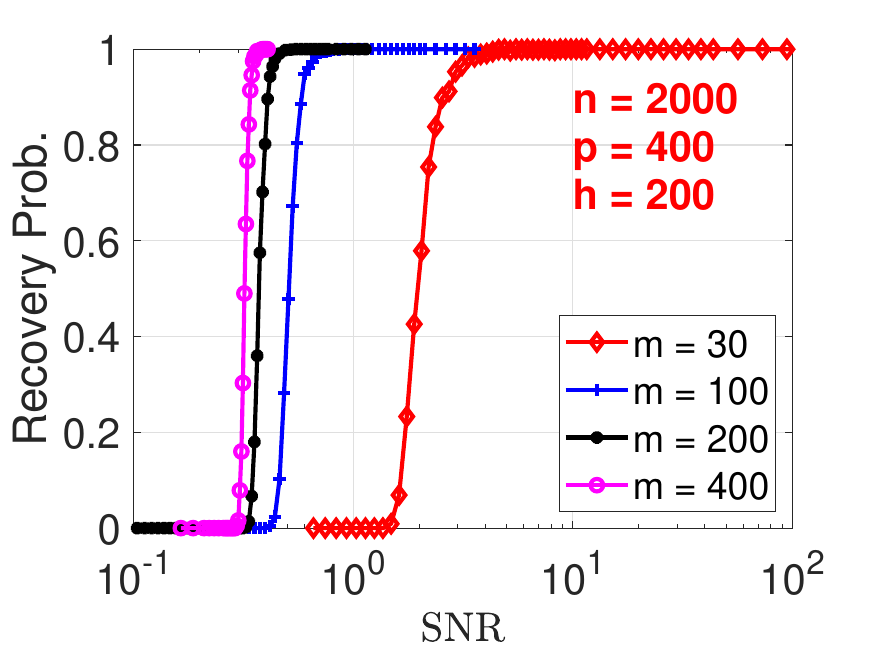}	
\includegraphics[width=2.2in]{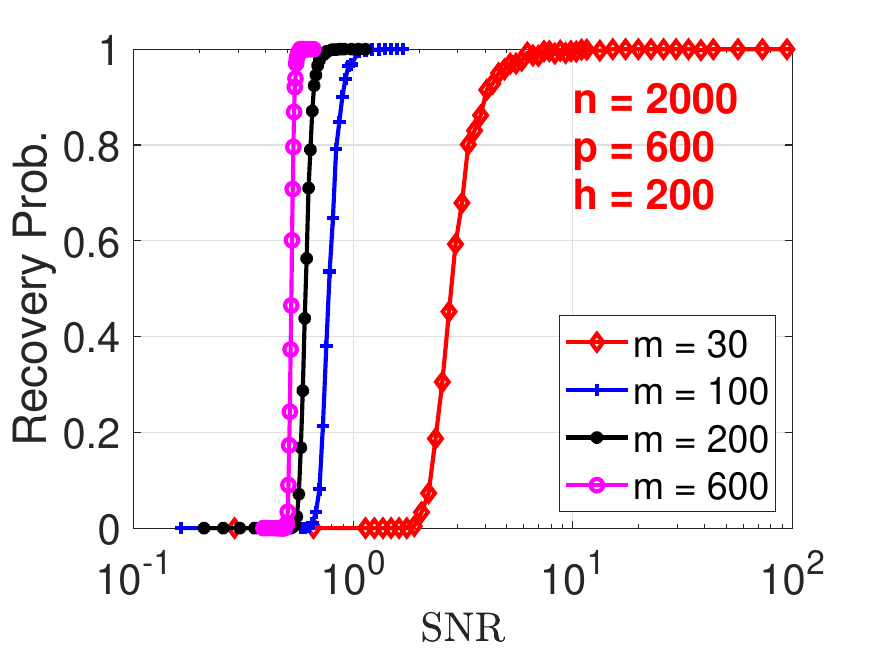}	
}

\mbox{
\includegraphics[width=2.2in]{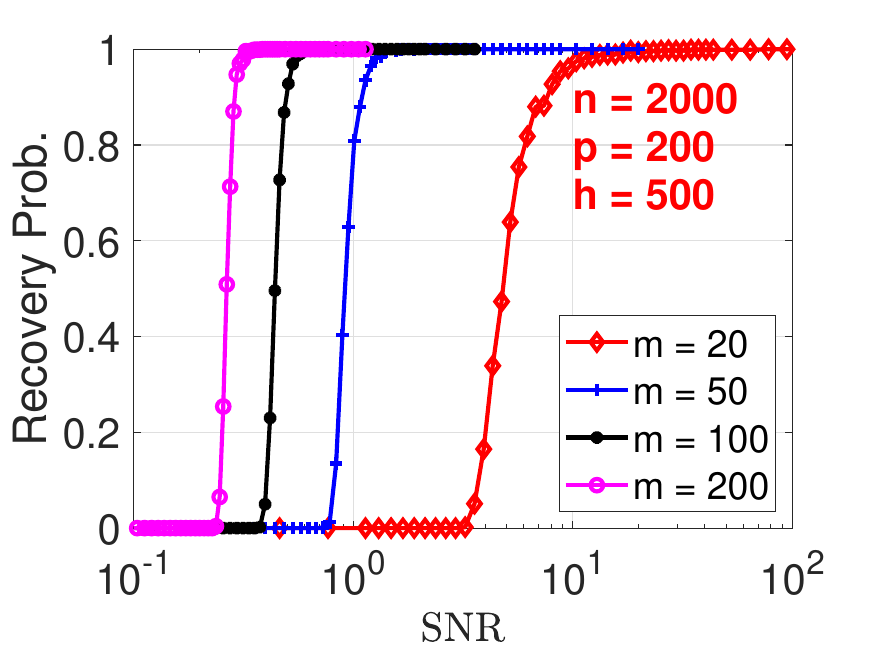}
\includegraphics[width=2.2in]{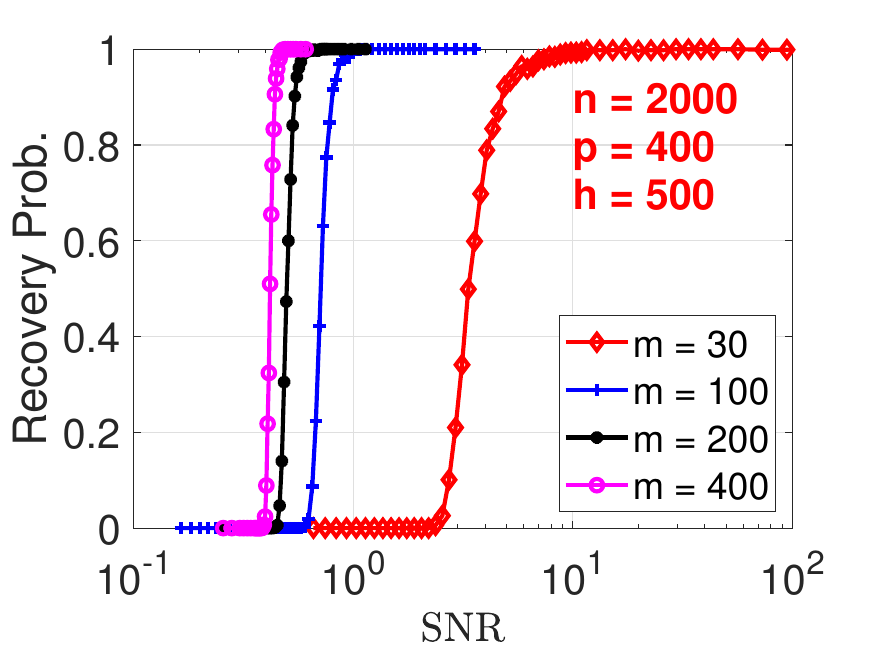}	
\includegraphics[width=2.2in]{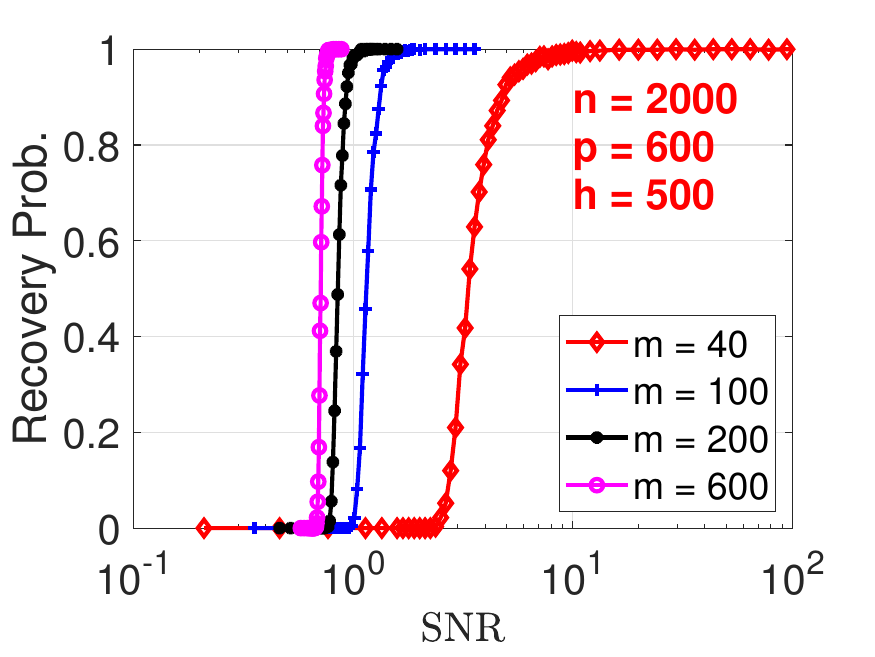}	
}

\mbox{
\includegraphics[width=2.2in]{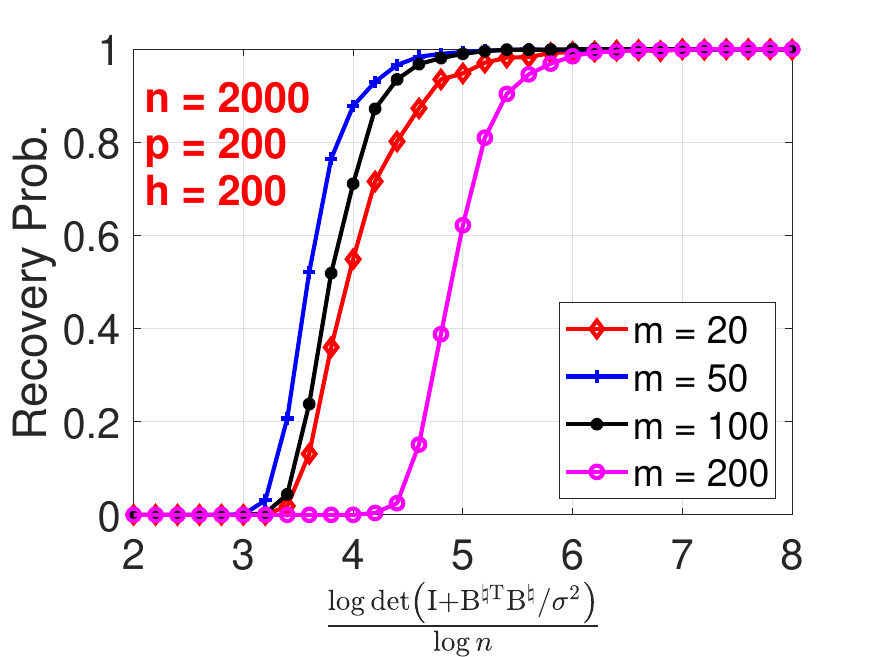}
\includegraphics[width=2.2in]{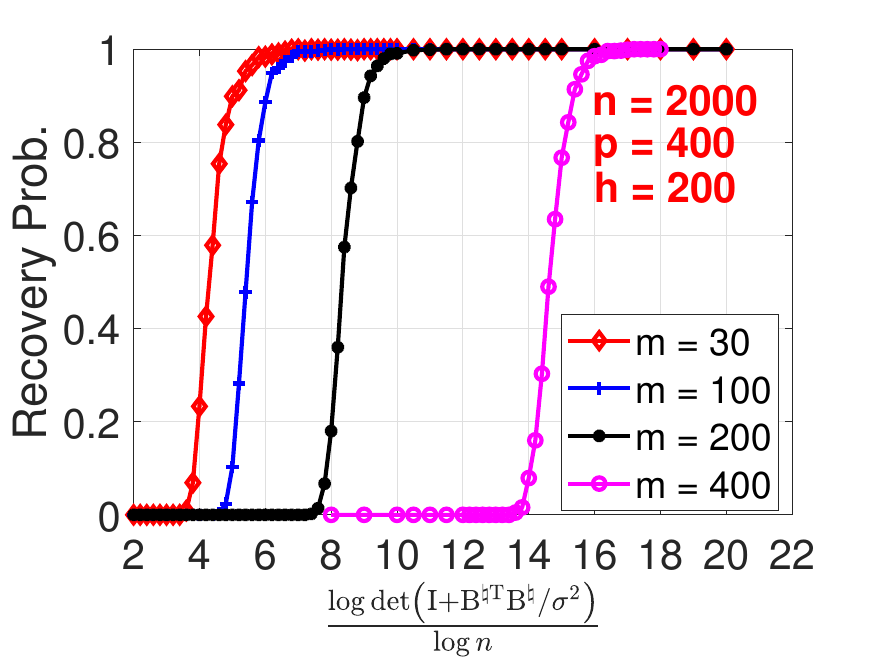}	
\includegraphics[width=2.2in]{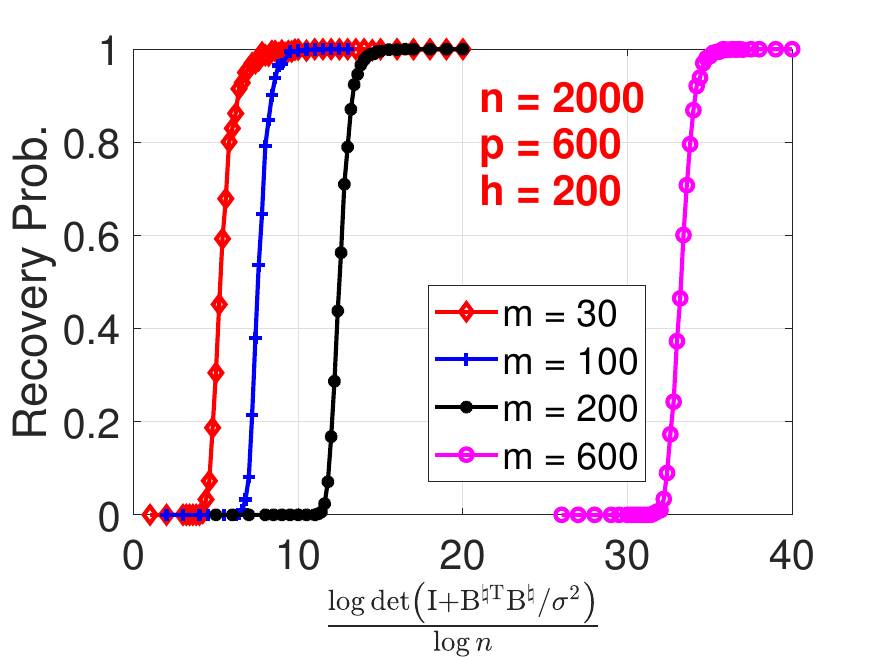}	
}

\mbox{
\includegraphics[width=2.2in]{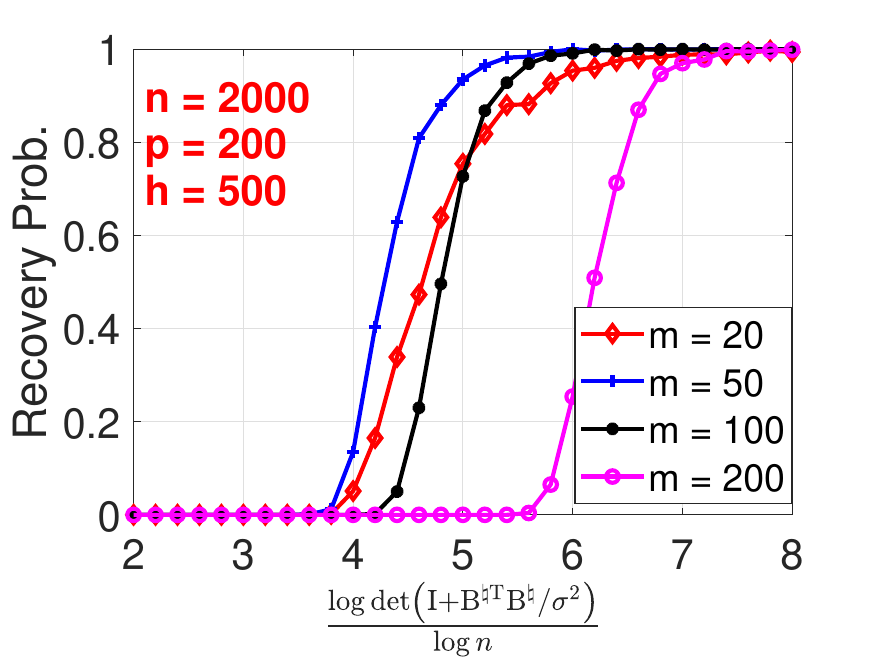}
\includegraphics[width=2.2in]{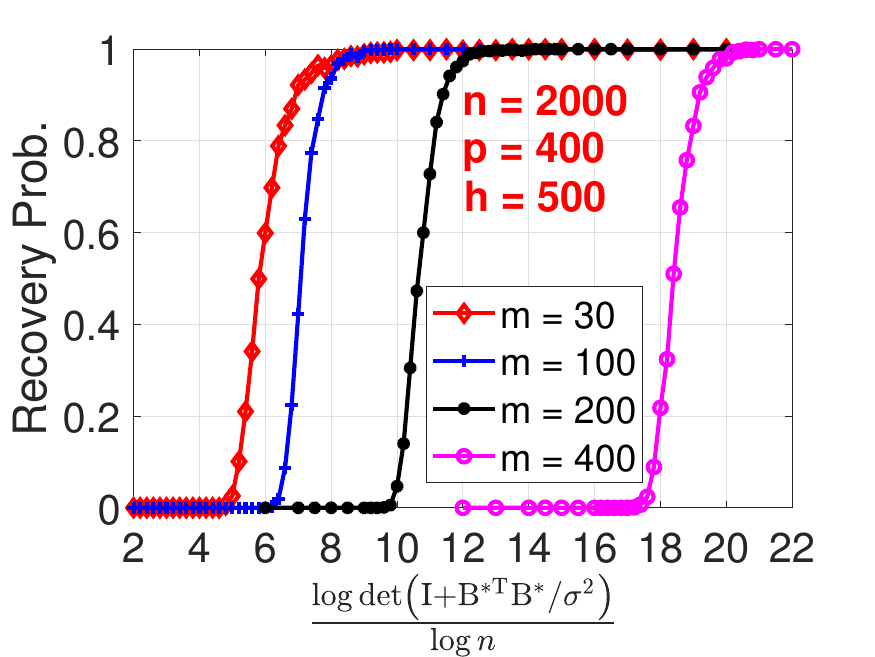}	
\includegraphics[width=2.2in]{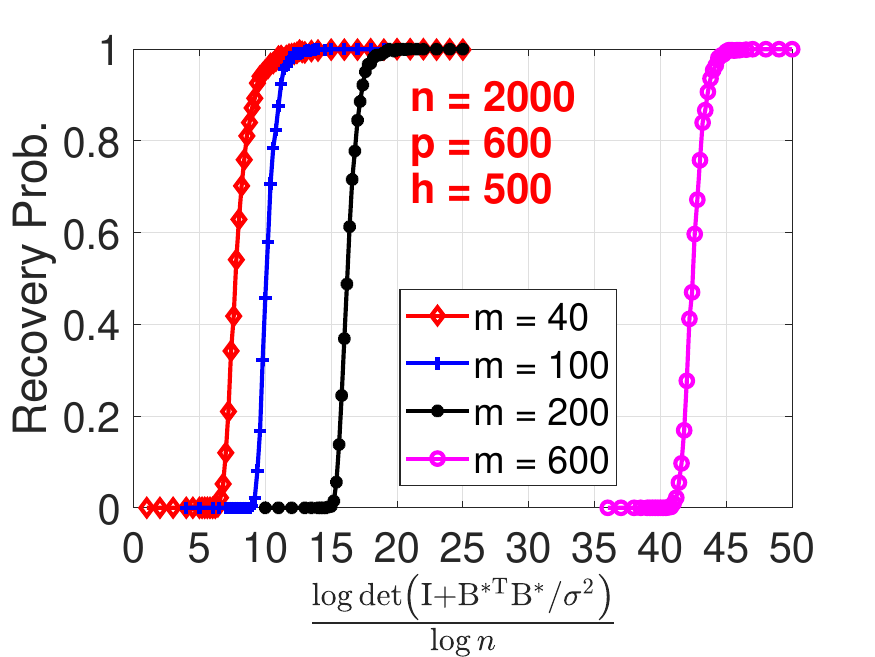}	
}

\end{center}
\caption{Simulated recovery rate $\Prob(\wh{\bPi} = \bPitrue)$, with $n = 2000$, $p\in\{200,400,600\}$, and $h\in\{200,500\}$,
versus $\snr$ (\textbf{upper panels}) and $\frac{\logdet\bracket{\bI + \bB^{\natural\rmt}\bBtrue/\sigma^2}}{\log n}$ (\textbf{lower panels}).}
\label{fig:recover_n2000}
\end{figure}

\begin{figure}[!h]
\begin{center}
\mbox{
\includegraphics[width=2.2in]{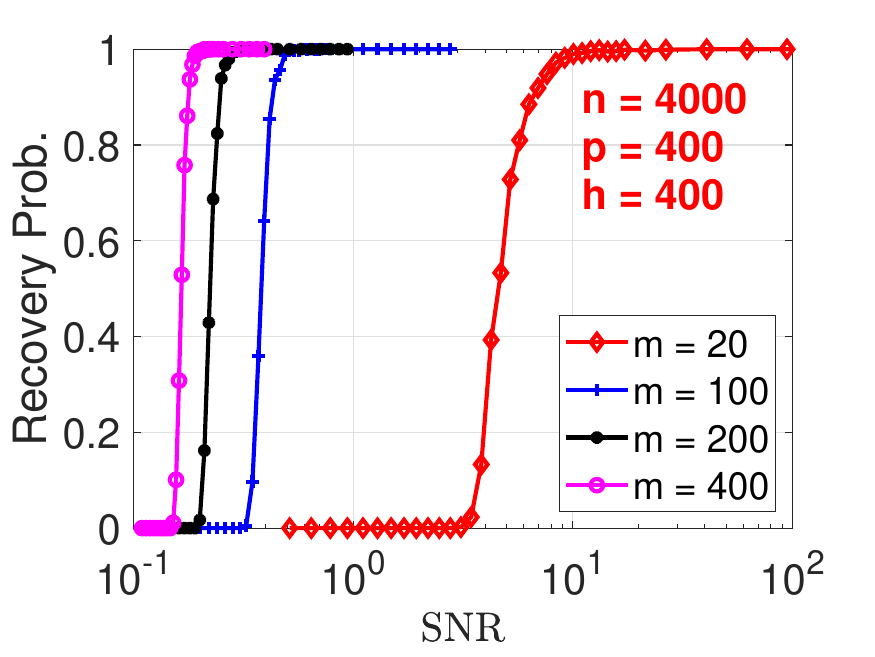}
\includegraphics[width=2.2in]{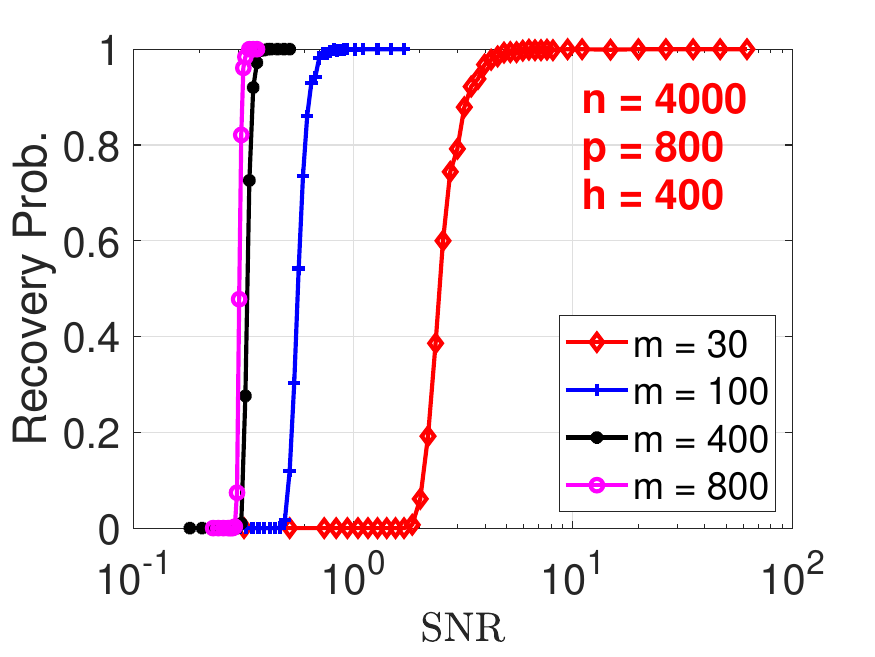}	
\includegraphics[width=2.2in]{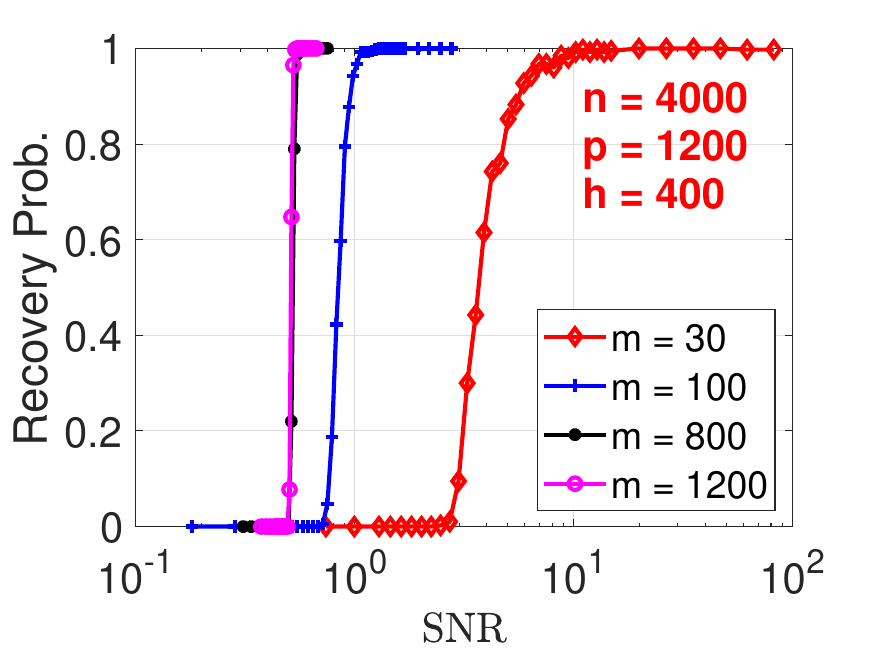}	
}

\mbox{
\includegraphics[width=2.2in]{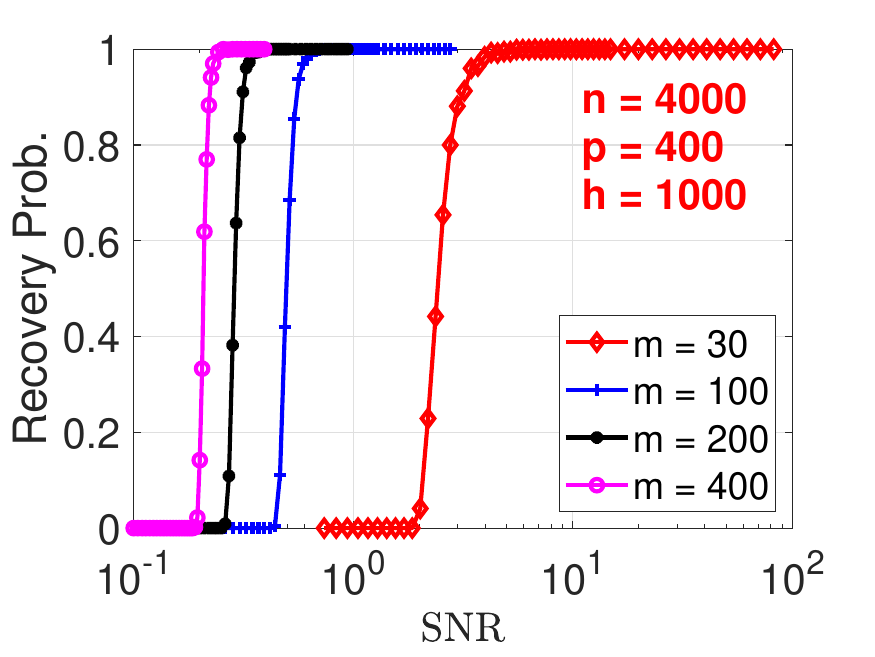}
\includegraphics[width=2.2in]{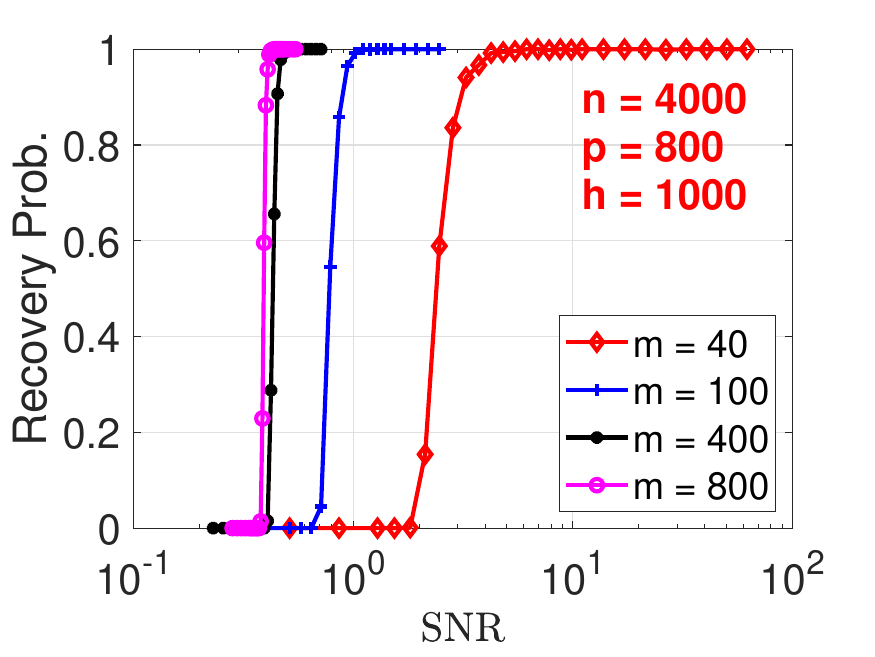}	
\includegraphics[width=2.2in]{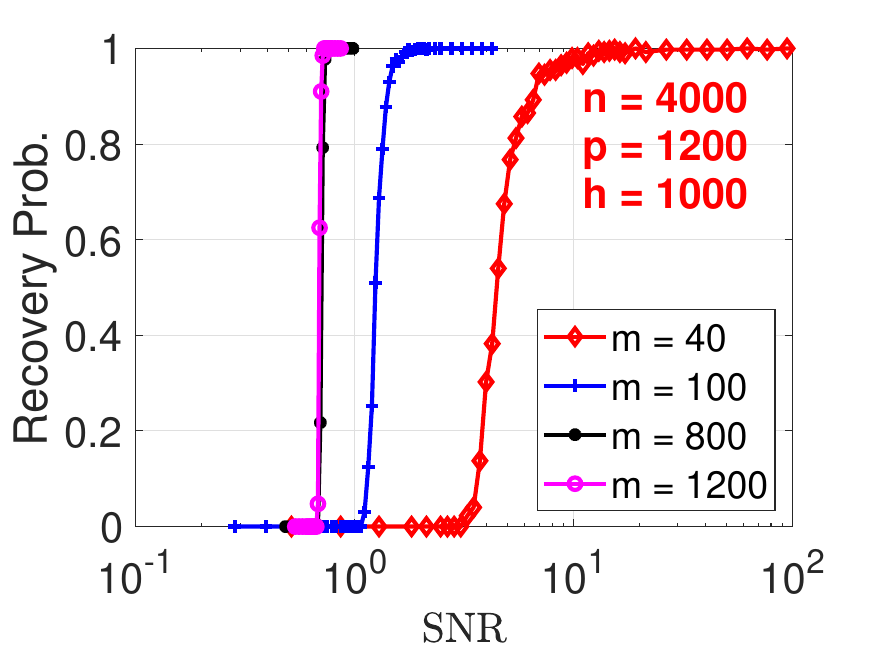}	
}

\mbox{
\includegraphics[width=2.2in]{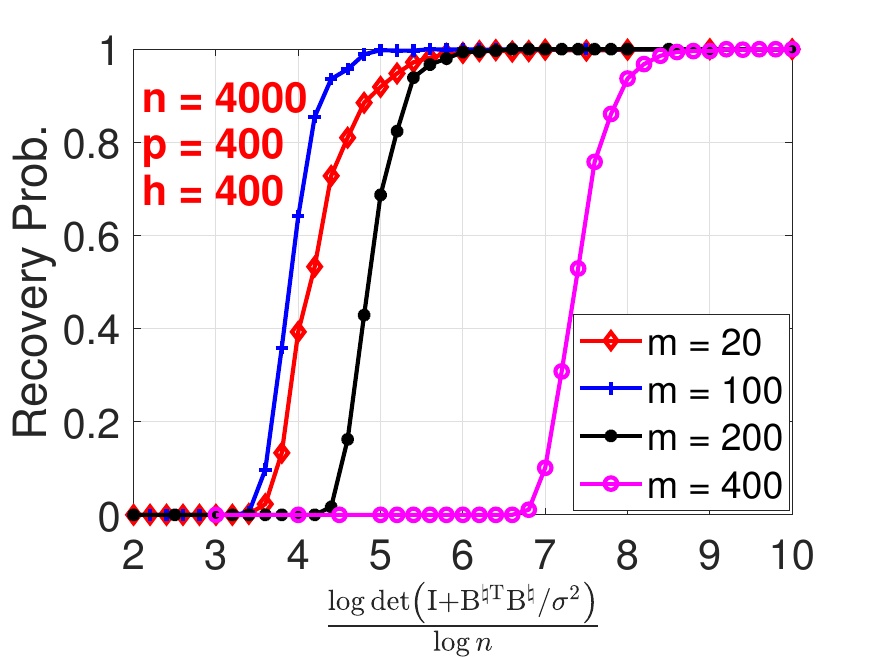}
\includegraphics[width=2.2in]{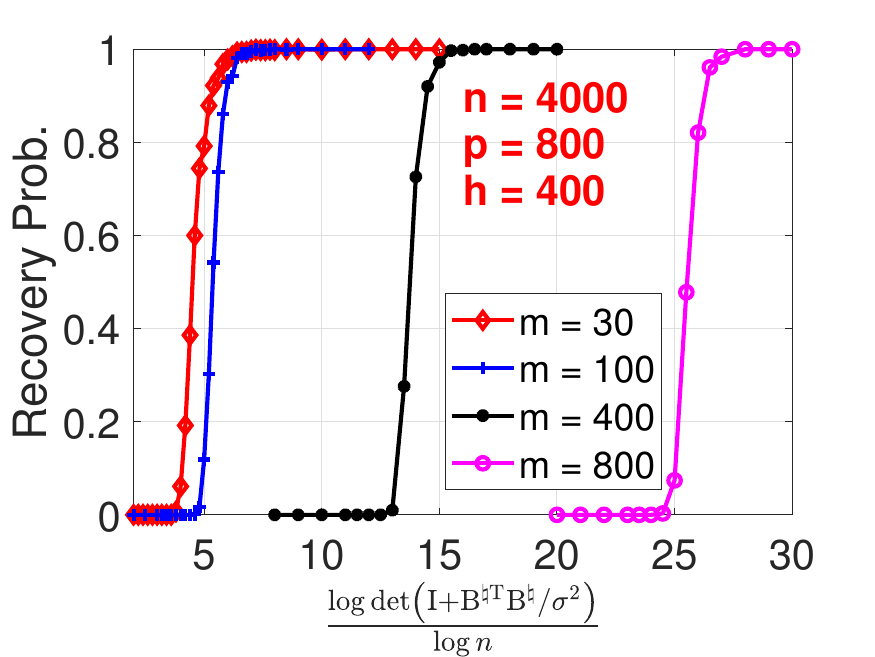}	
\includegraphics[width=2.2in]{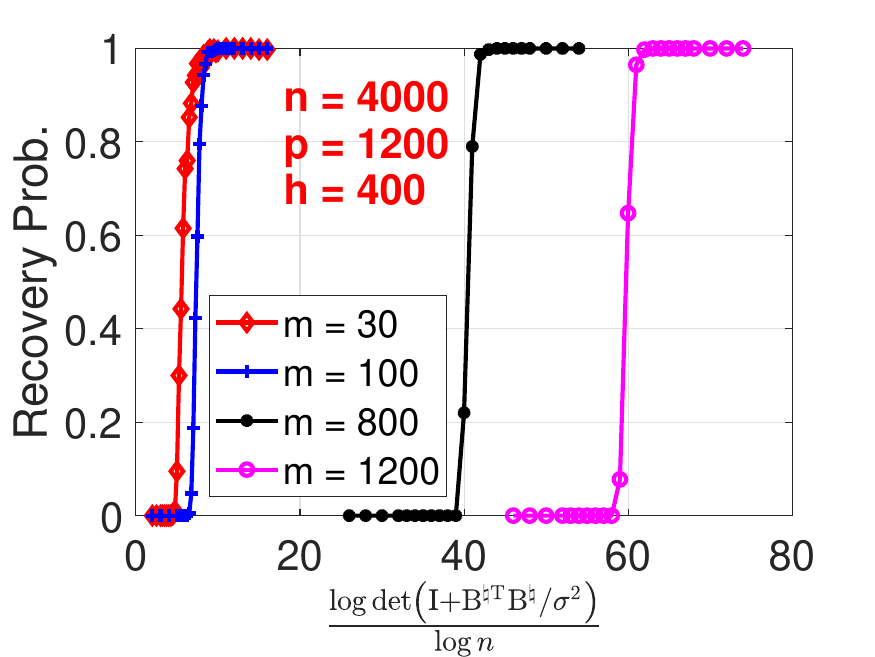}	
}

\mbox{
\includegraphics[width=2.2in]{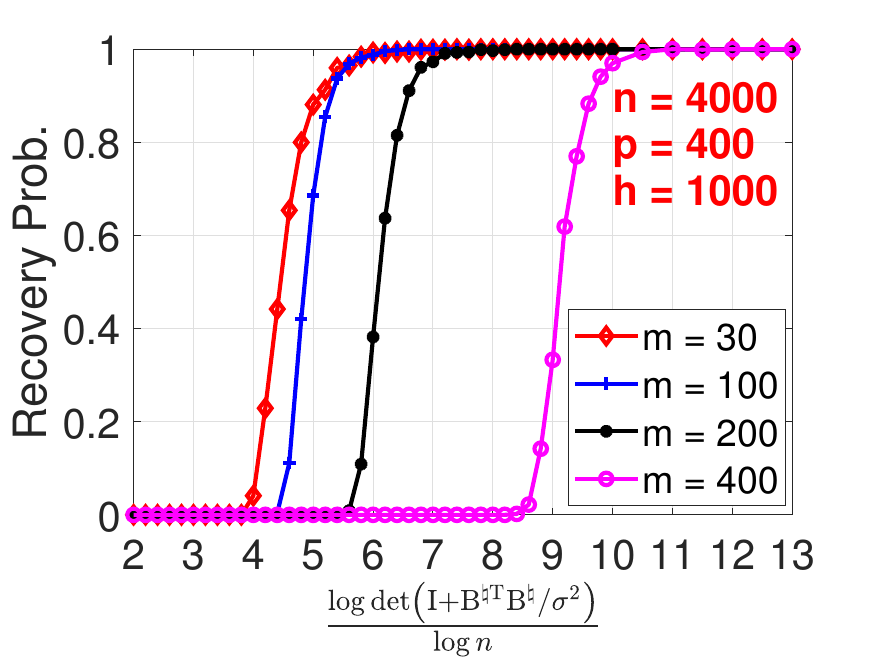}
\includegraphics[width=2.2in]{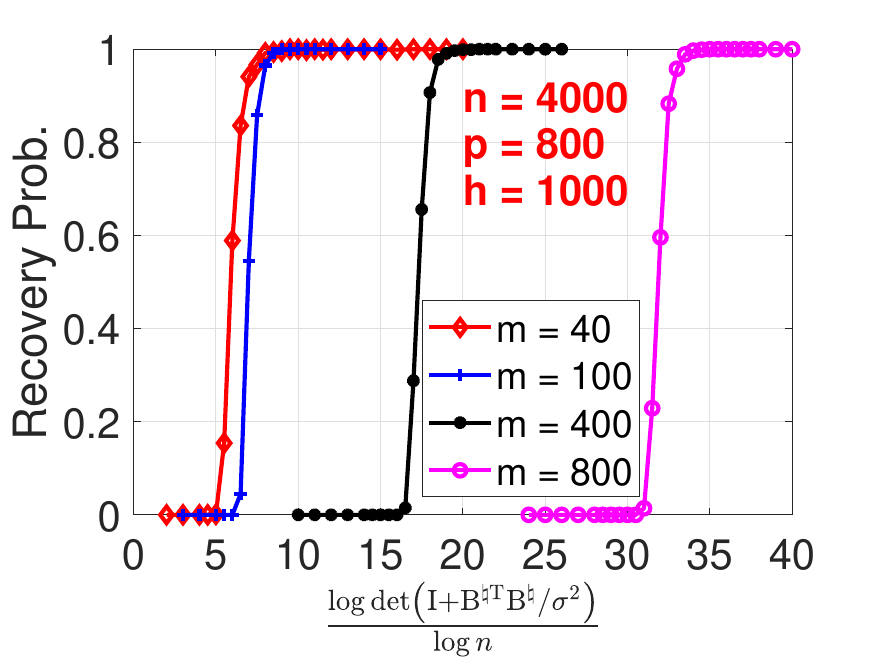}	
\includegraphics[width=2.2in]{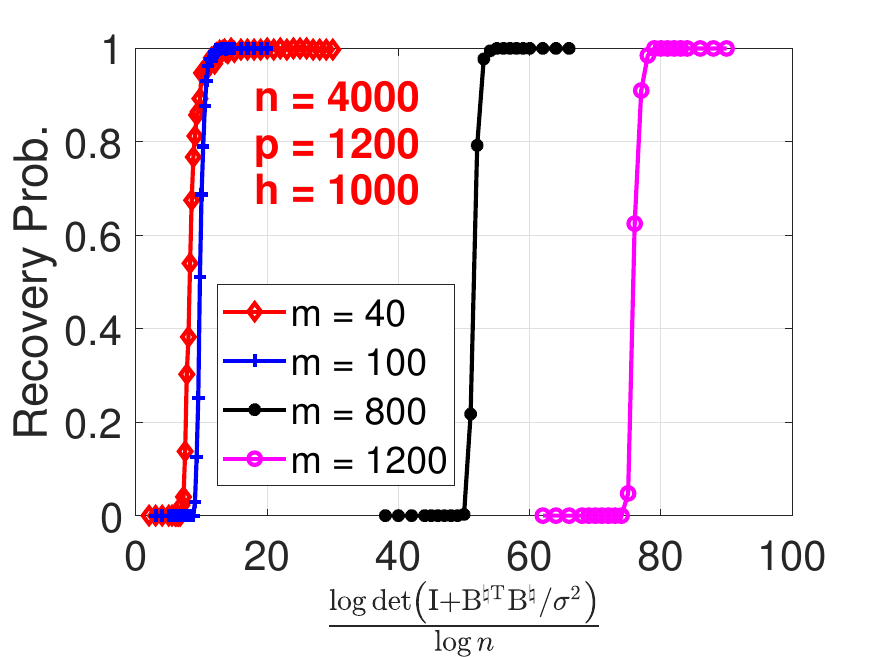}	
}

\end{center}
\caption{Simulated recovery rate $\Prob(\wh{\bPi} = \bPitrue)$, with $n = 4000$, $p\in\{400,800,1200\}$, and $h\in\{400,1000\}$,
versus $\snr$ (\textbf{upper panels}) and $\frac{\logdet\bracket{\bI + \bB^{\natural\rmt}\bBtrue/\sigma^2}}{\log n}$ (\textbf{lower panels}).}
\label{fig:recover_n4000}
\end{figure}

\newpage\clearpage

\begin{figure}[b!]

\vspace{-0.1in}

\centering
\mbox{
\includegraphics[width=2.2in]{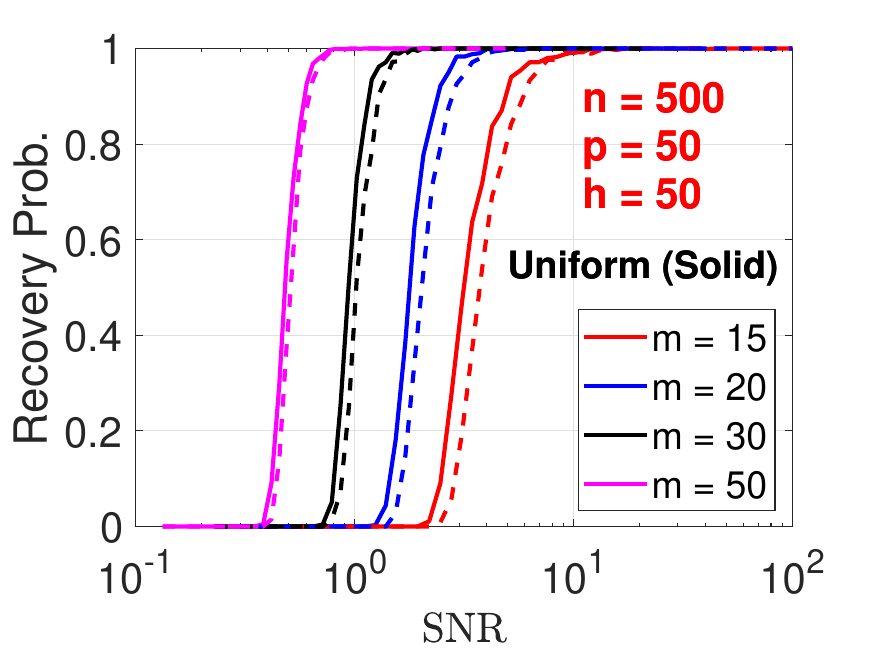}
\includegraphics[width=2.2in]{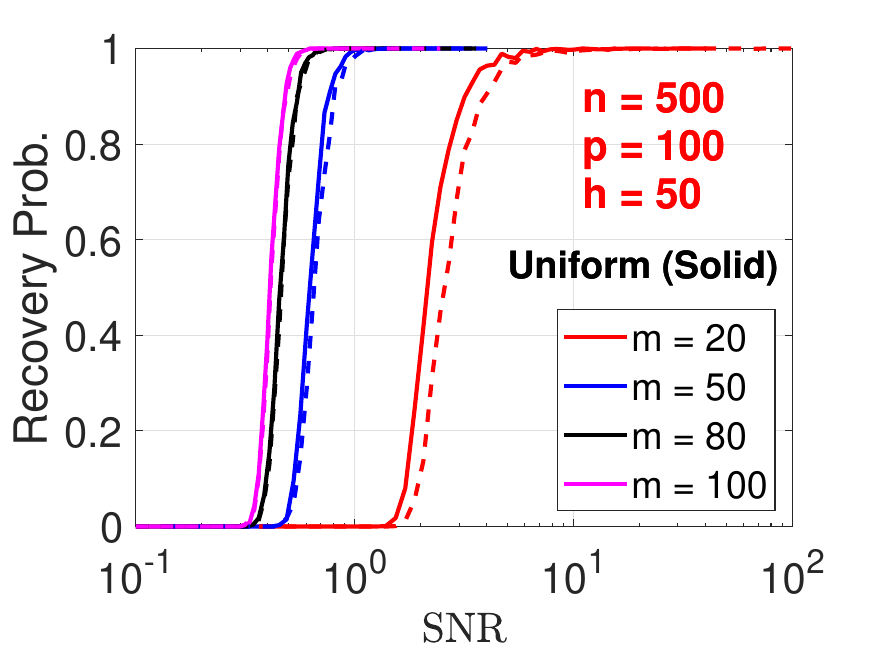}	
\includegraphics[width=2.2in]{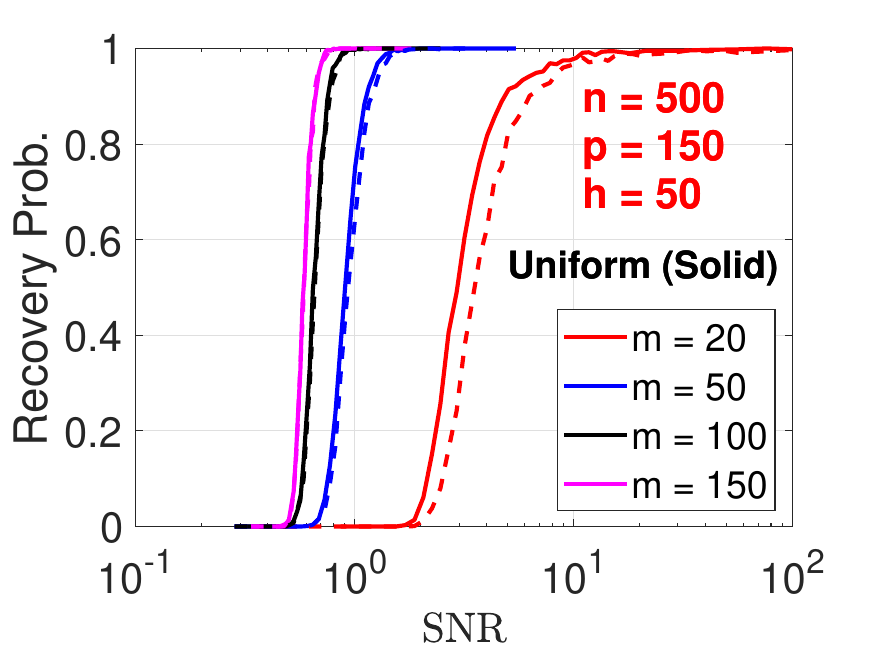}	
}

\mbox{
\includegraphics[width=2.2in]{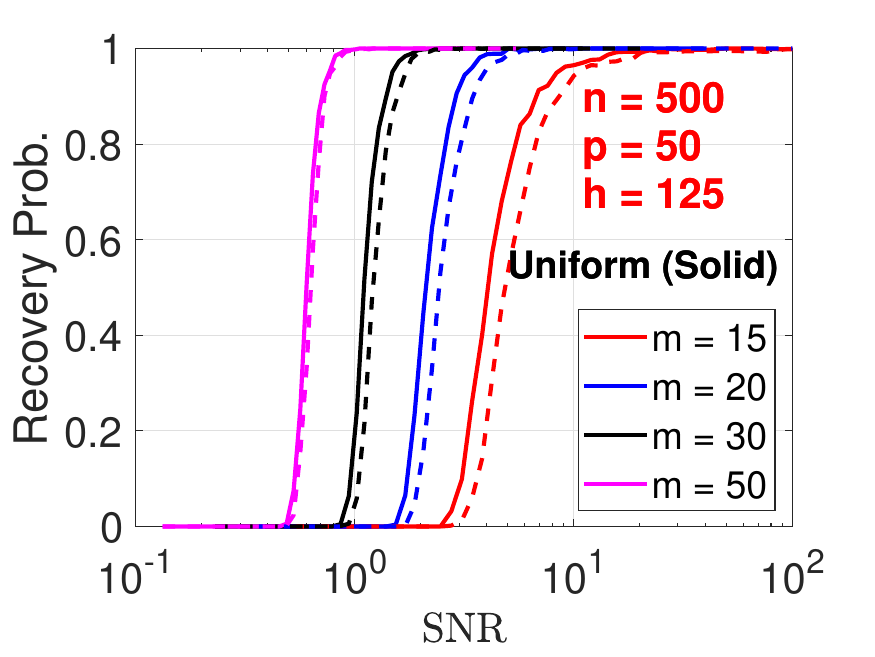}
\includegraphics[width=2.2in]{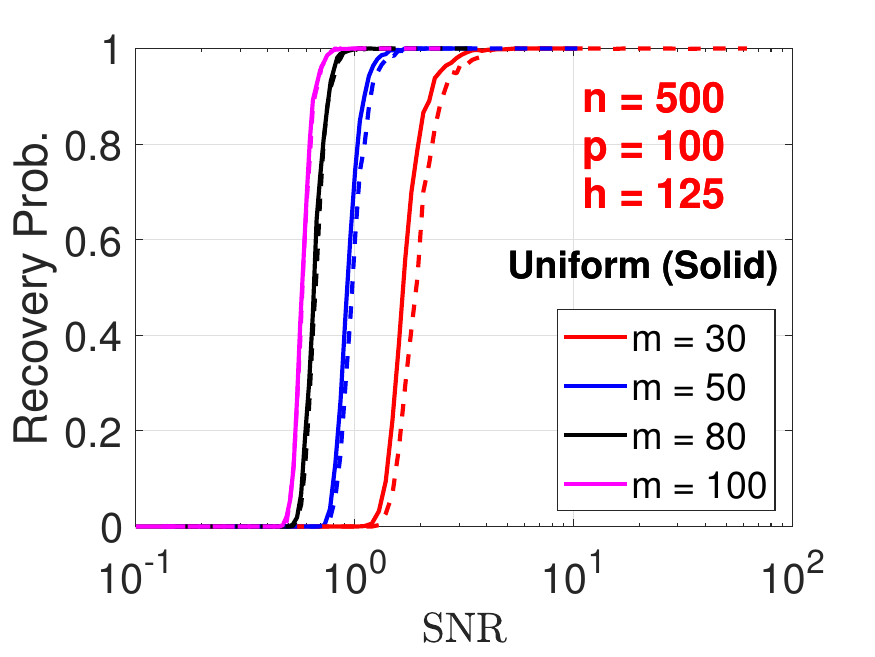}	
\includegraphics[width=2.2in]{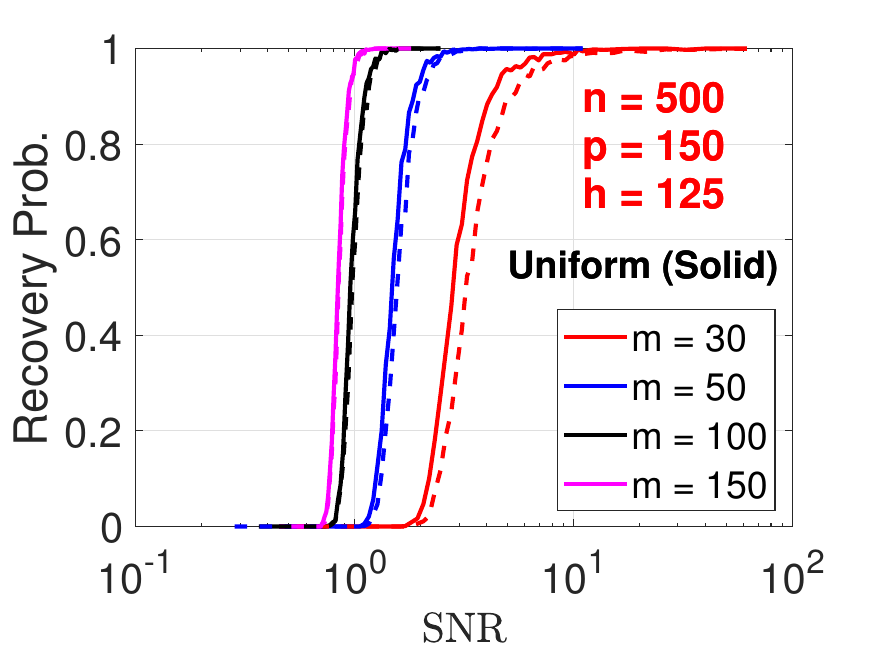}	
}

\vspace{-0.15in}

\caption{Simulated recovery rate $\Prob(\wh{\bPi} = \bPitrue)$ with
$n = 500$, $p \in \set{50, 100, 150}$, $h \in \set{50, 125}$, and $\bX_{ij}\iid \unif[-1, 1]$, with respect to $\snr$.}
\label{fig:unif_n500}
\end{figure}

\begin{figure}[b!]
\centering
\mbox{
\includegraphics[width=2.2in]{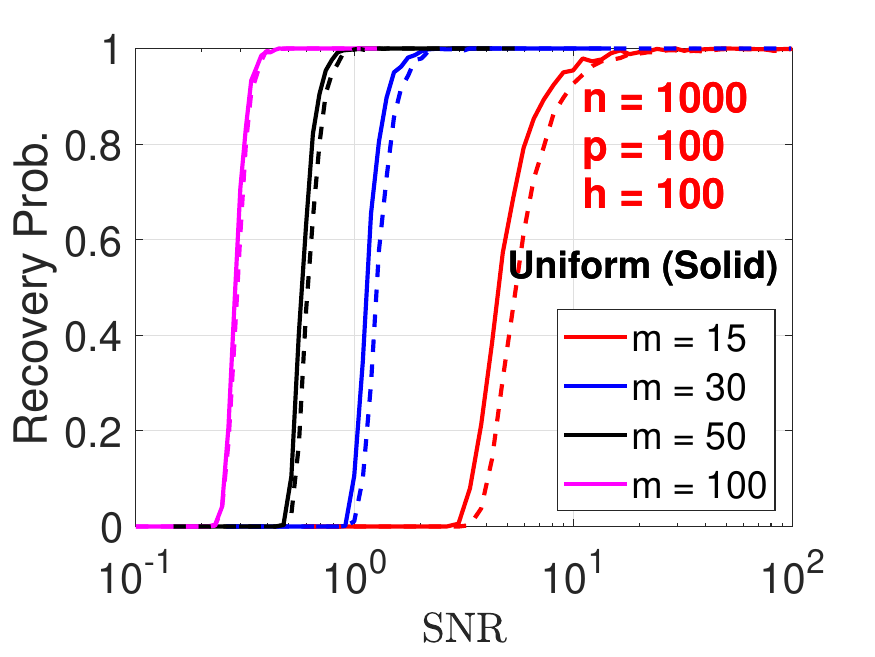}
\includegraphics[width=2.2in]{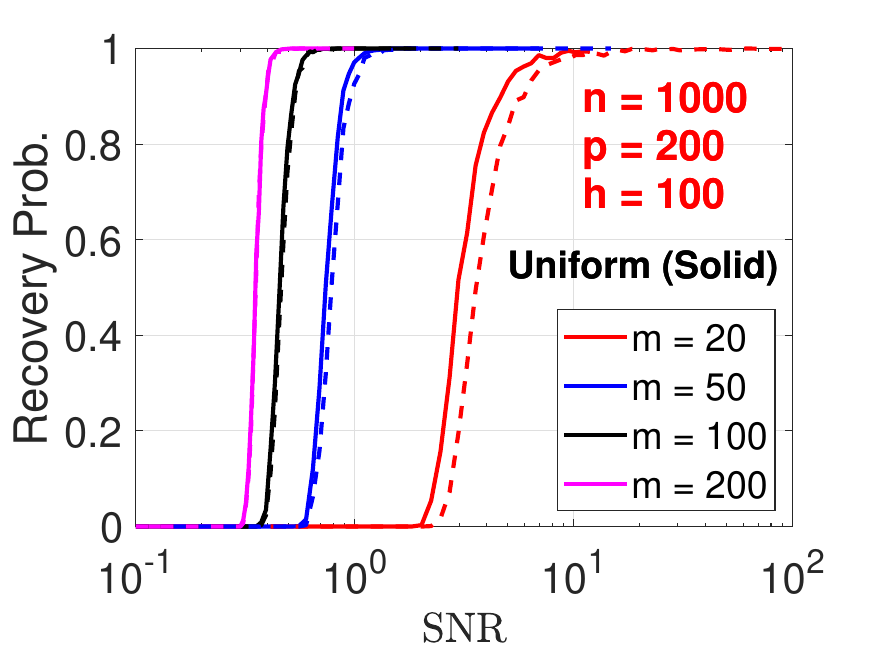}	
\includegraphics[width=2.2in]{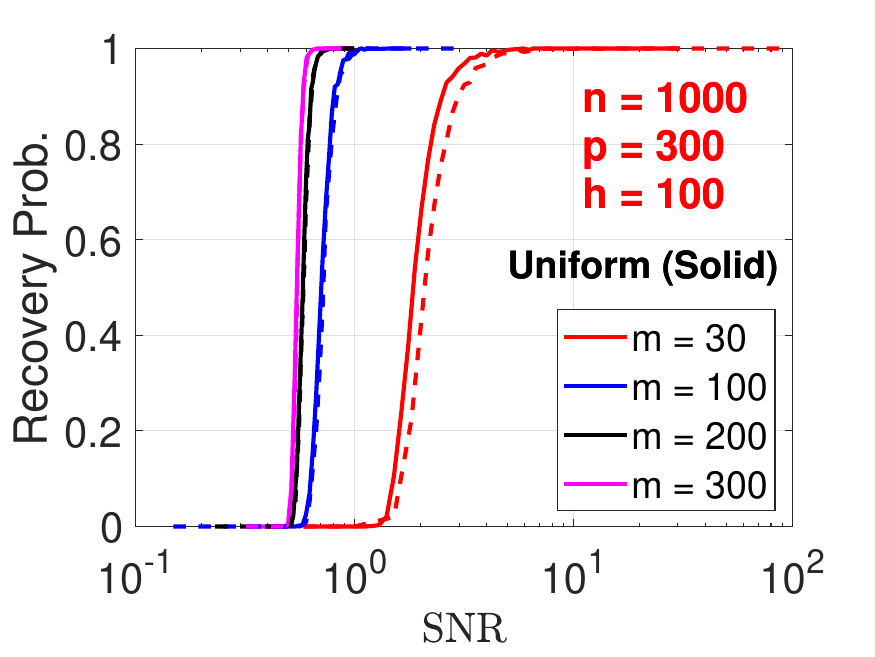}	
}

\mbox{
\includegraphics[width=2.2in]{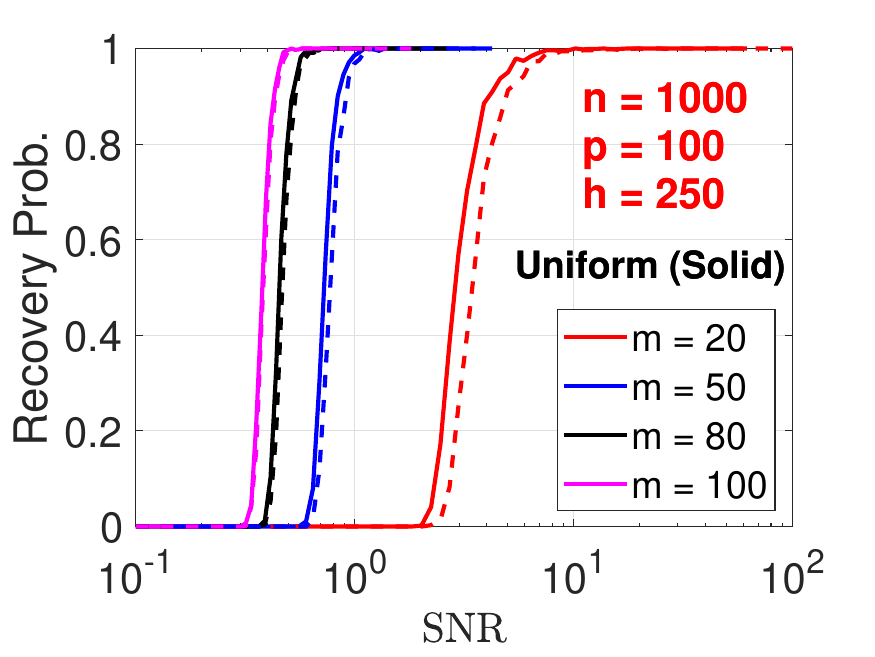}
\includegraphics[width=2.2in]{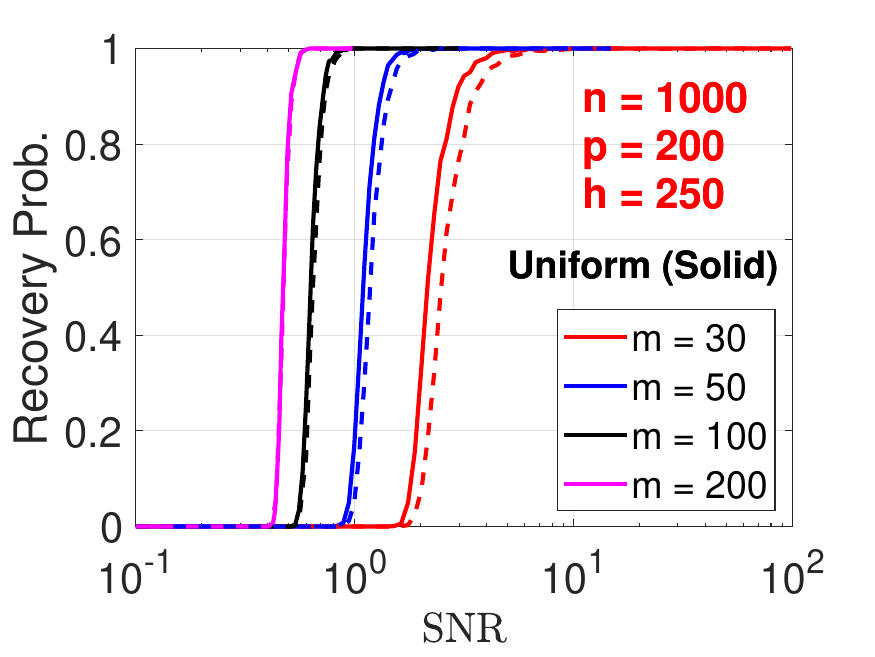}	
\includegraphics[width=2.2in]{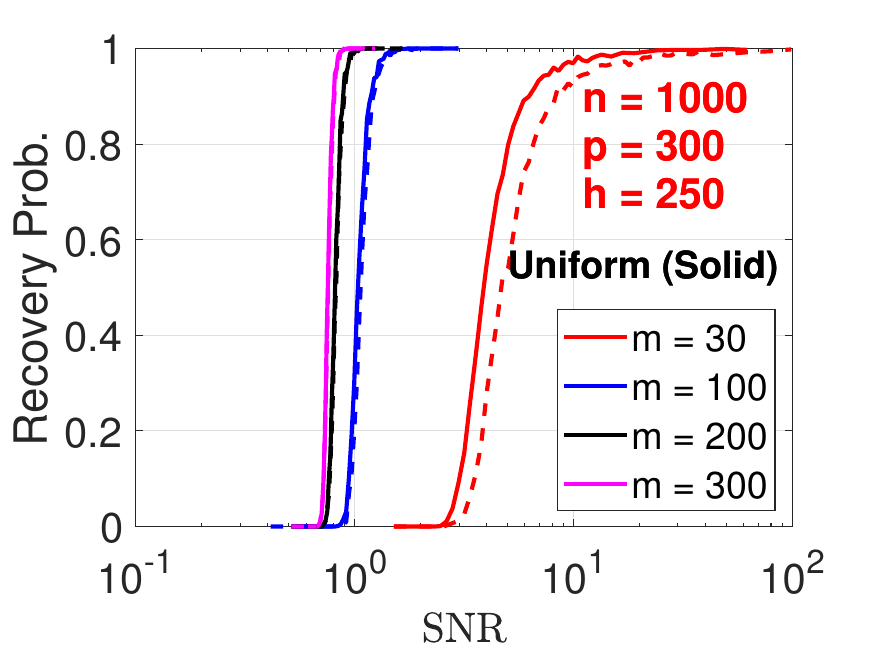}	
}

\vspace{-0.15in}

\caption{Simulated recovery rate $\Prob(\wh{\bPi} = \bPitrue)$ with
$n = 1000$, $p \in \set{100, 200, 300}$, $h \in \set{100, 250}$, and $\bX_{ij}\iid \unif[-1, 1]$, with respect to $\snr$.}
\label{fig:unif_n1000}\vspace{-0.1in}
\end{figure}

\subsection{Uniform distribution}

This subsection investigates the recovery performance
when $\bX_{ij}\iid\unif[-1, 1]$. Similar to above, we fix
$\nfrac{p}{n}$ to be $\set{0.1, 0.2, 0.3}$ and $\nfrac{h}{n}$
to be $\set{0.1, 0.25}$. Since the performance does not change drastically with increasing $n$,
we limit sample number $n$ to be $\set{500, 1000}$ and put the corresponding results
in Figure~\ref{fig:unif_n500} and Figure~\ref{fig:unif_n1000}, respectively.
For an easy comparison, we also put the
results with Gaussian distributed $\bX$ in
Figure~\ref{fig:unif_n500} and Figure~\ref{fig:unif_n1000}, which are
shown in dashed curves.
One noteworthy fact is that uniform distribution seems to
be more friendly for permutation recovery. A plausible
reason is that uniform distribution has a smaller variance,
namely, $\Var(\cdot) = (\nfrac{1}{2})\int_{-1}^1 z^2 dz = \nfrac{1}{3} < 1$.

Apart from that, we see its behavior is very similar to
that of the Gaussian distribution. For the conciseness of presentation, we omit the plot of recovery rate in terms
of $\frac{\logdet(\bI + \nfrac{\bB^{\natural \rmt}\bBtrue}{\sigma^2})}{\log n}$.

\subsection{Rademacher distribution}

This subsection considers the Rademacher distribution, i.e.,
$\Prob(\bX_{ij} = \pm 1) = \nfrac{1}{2}$. To begin with, we briefly comment on the
$\srank{\bBtrue}$ requirement. In Theorem~\ref{thm:multi_snr_require_log_concave}, we claim that log-concavity is required to remove the requirement
$\srank{\bBtrue}\gsim \log n$.
Here we give an example to illustrate its necessity.
For details see
Figure~\ref{fig:perform_compare}.
In the left panel, we assume $\bX_{ij} \iid \unif[-1,1]$, which is a log-concave
sub-gaussian RV; while in the right panel, we assume $\bX_{ij} \iid \textup{Rademacher}$, which is sub-gaussian but not log-concave.
When the stable rank $\srank{\bBtrue}$ is not sufficiently large, to put more specifically, $\srank{\bBtrue} \gsim \log n$,
we conclude the correct permutation cannot be reconstructed for Rademacher distribution even when $\snr$ is sufficiently large.

\begin{figure}[!h]

\begin{center}
\mbox{
\includegraphics[width=3in]{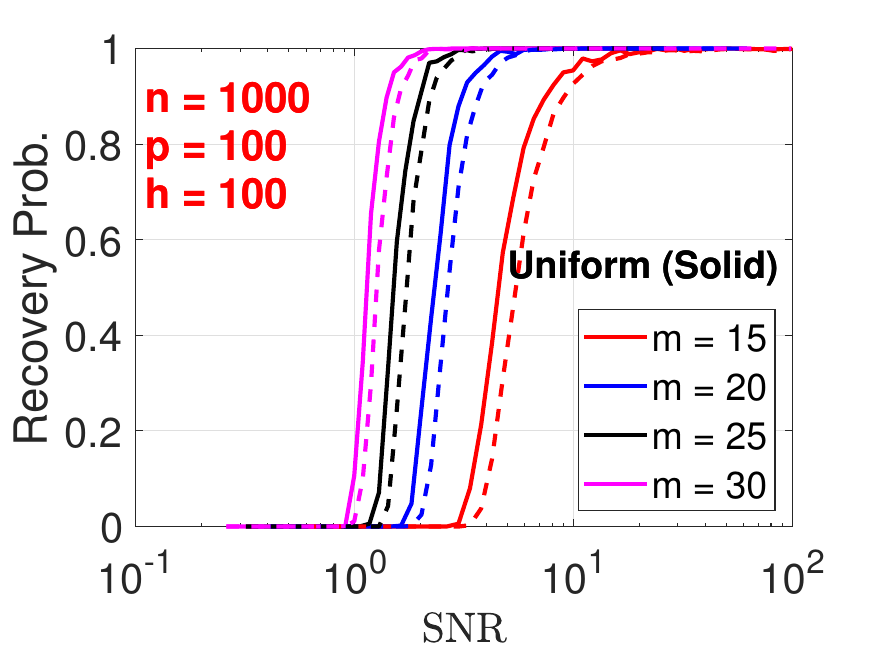}
\includegraphics[width=3in]{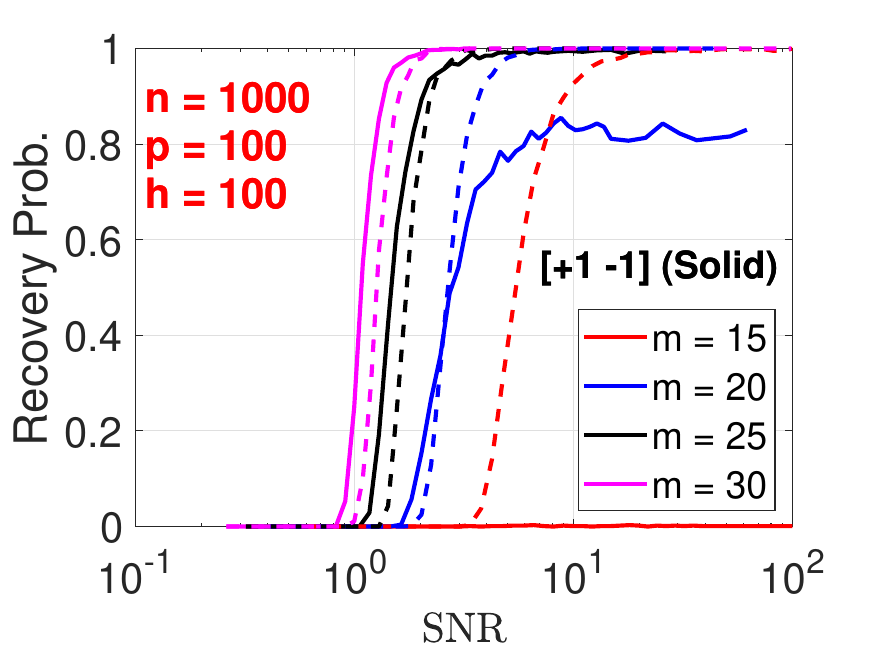}	
}

\end{center}

\vspace{-0.15in}

\caption{Comparison of permutation recovery performance
between \textbf{Uniform distribution among $[-1,1]$}
(\textbf{left panel})
and \textbf{Rademacher distribution} (\textbf{right panel}).
The dashed line corresponds to the
performance when $\bX_{ij}\iid \normdist(0, 1)$.}
\label{fig:perform_compare}
\end{figure}

In addition, we investigate the impact of $\nfrac{p}{n}$
and number of permuted rows $h$. Similar to above,
we fix $\nfrac{p}{n}$ to be $\set{0.1, 0.2, 0.3}$ and
$\nfrac{h}{n}$ to be $\set{0.1, 0.25}$. In Figure~\ref{fig:benoulli_n500},
we fix $n$ to be $500$; while in Figure~\ref{fig:benoulli_n1000}, we increase
it to be $1000$.
One noticeable phenomenon is that
the permutation recovery performance under
the Rademacher setting is very similar to
that of Gaussian distribution and Uniform distribution
provided that $\srank{\bBtrue}$ is large enough. This again
suggests that \textbf{log-concavity assumption is unavoidable
if the large stable rank requirement is to be mitigated}.

\begin{figure}[!h]

\begin{center}
\mbox{
\includegraphics[width=2.2in]{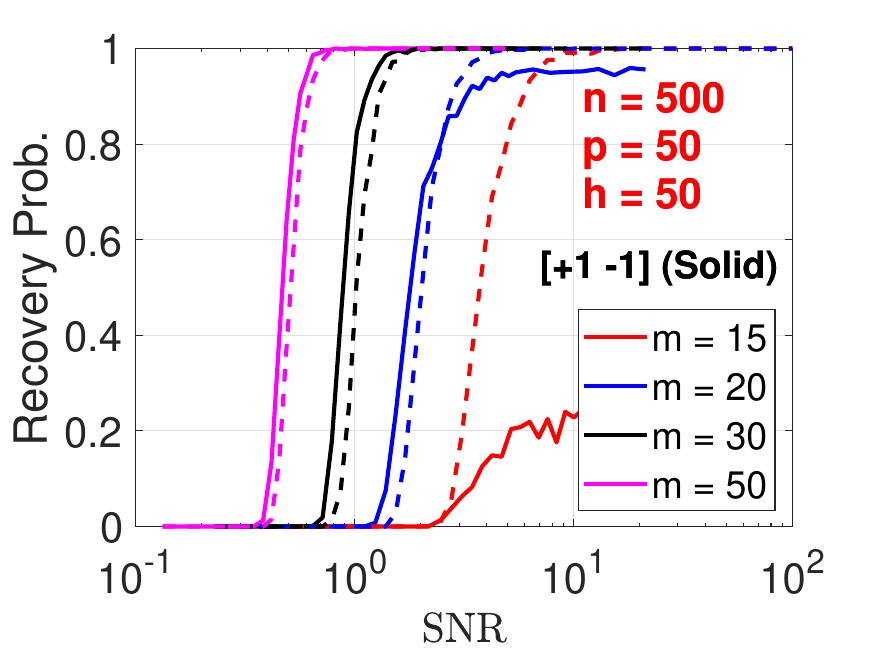}
\includegraphics[width=2.2in]{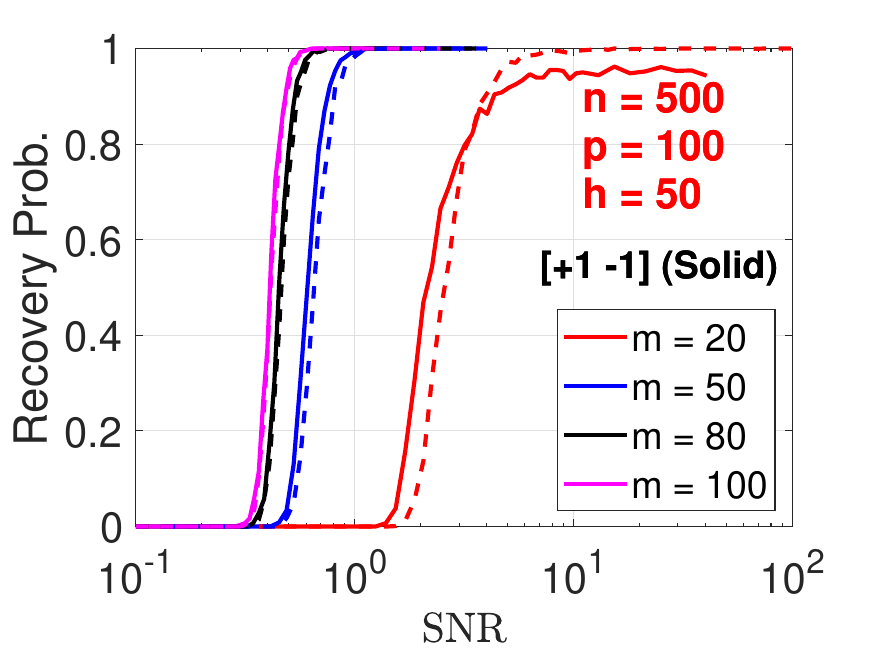}	
\includegraphics[width=2.2in]{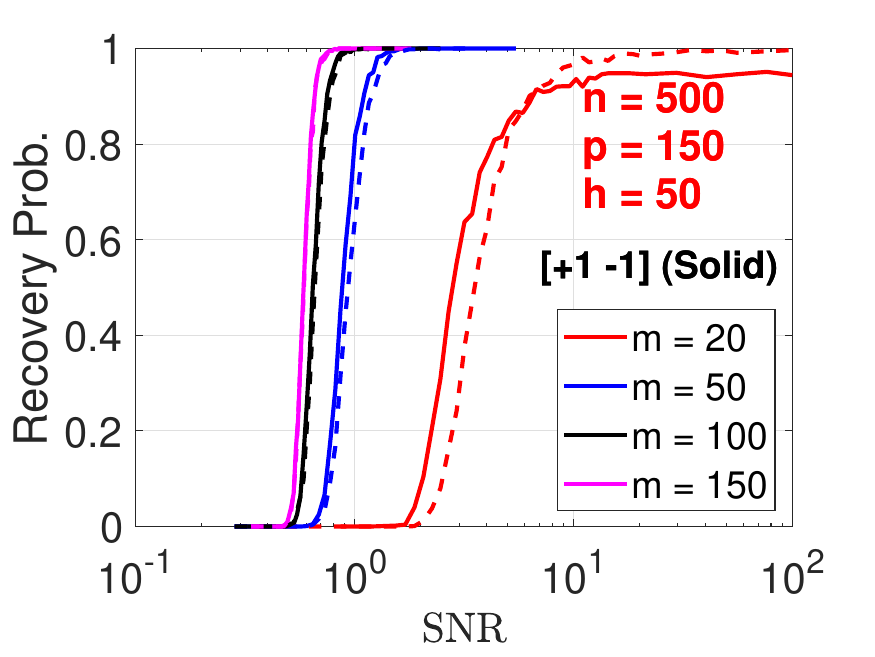}	
}

\mbox{
\includegraphics[width=2.2in]{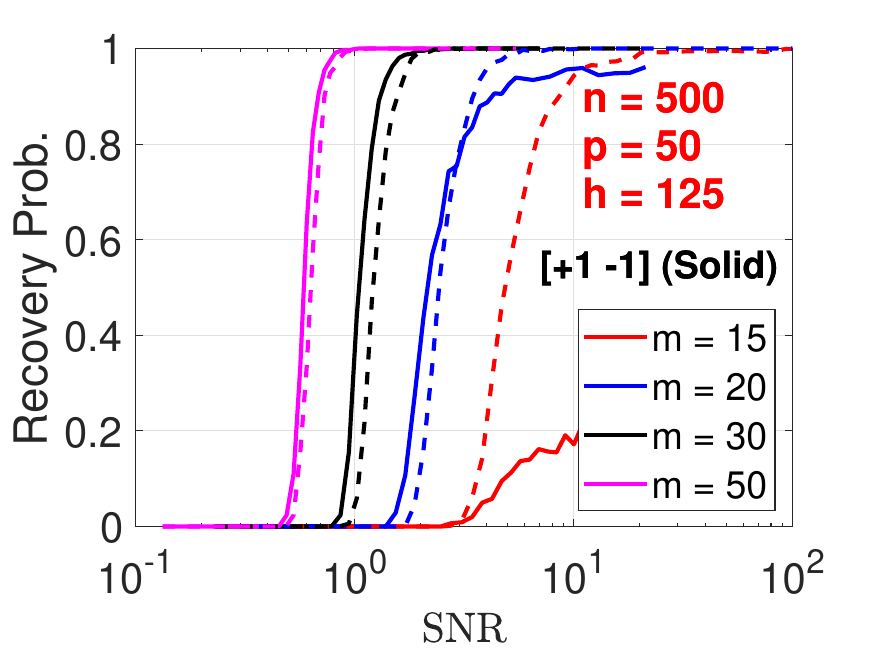}
\includegraphics[width=2.2in]{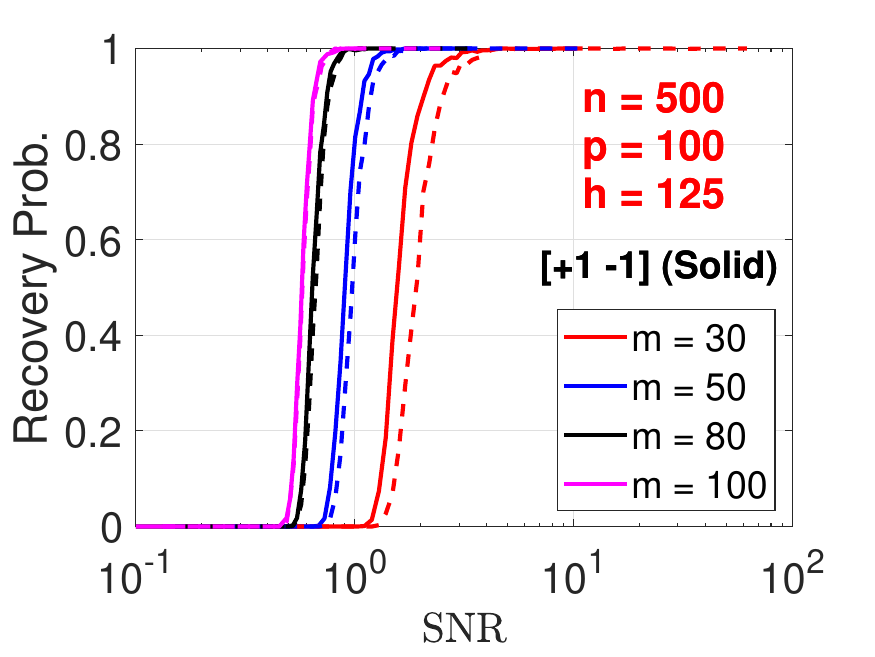}	
\includegraphics[width=2.2in]{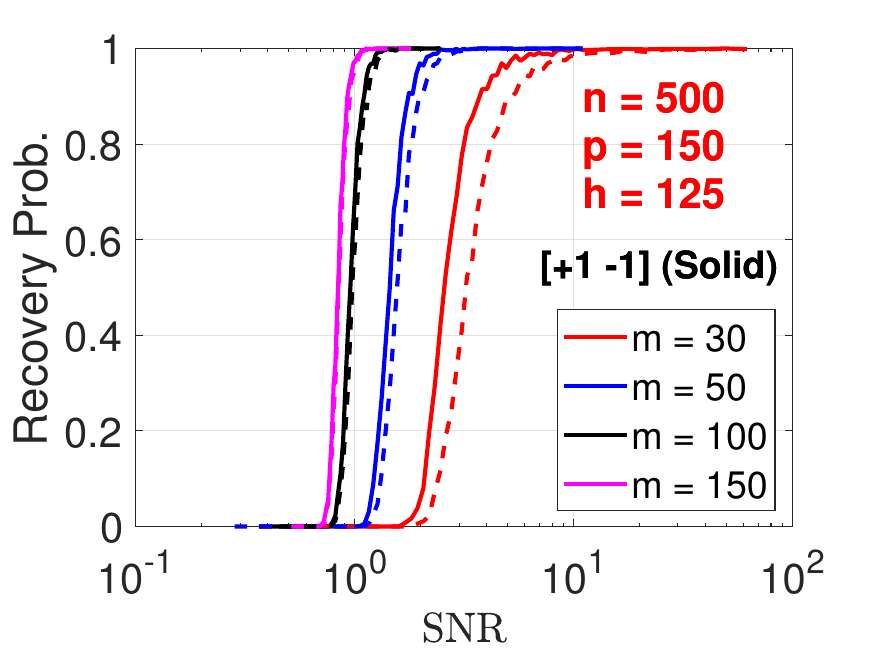}	
}

\end{center}
\caption{Simulated recovery rate $\Prob(\wh{\bPi} = \bPitrue)$ with
$n = 500$, $p\in \set{50, 100, 150}$, $h\in \set{50, 125}$, and $\bX_{ij}\iid \textup{Rademacher}$, with respect to $\snr$. The dashed line corresponds to the
performance when $\bX_{ij}\iid \normdist(0, 1)$.}
\label{fig:benoulli_n500}
\end{figure}

\begin{figure}[!h]
\begin{center}

\mbox{
\includegraphics[width=2.2in]{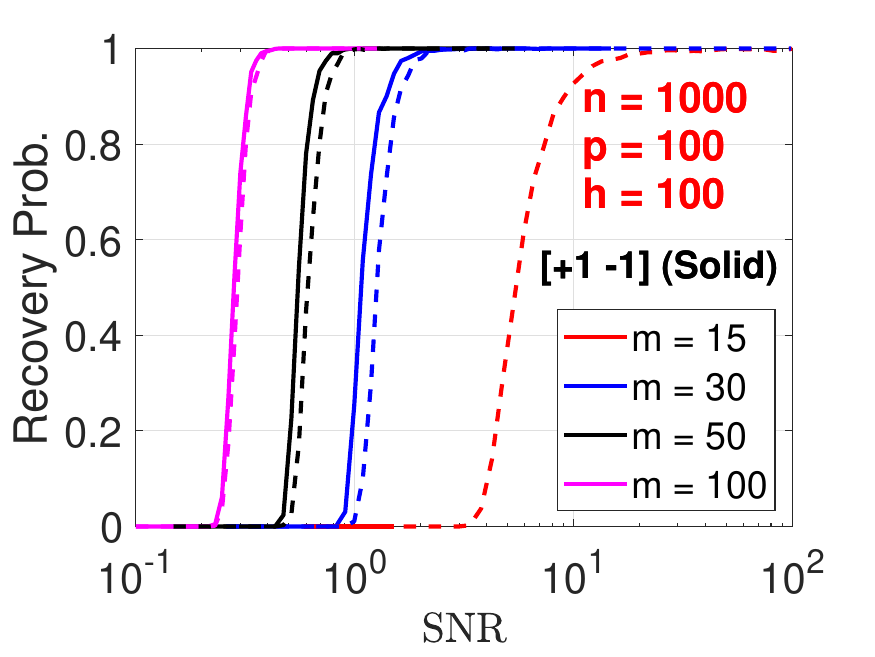}
\includegraphics[width=2.2in]{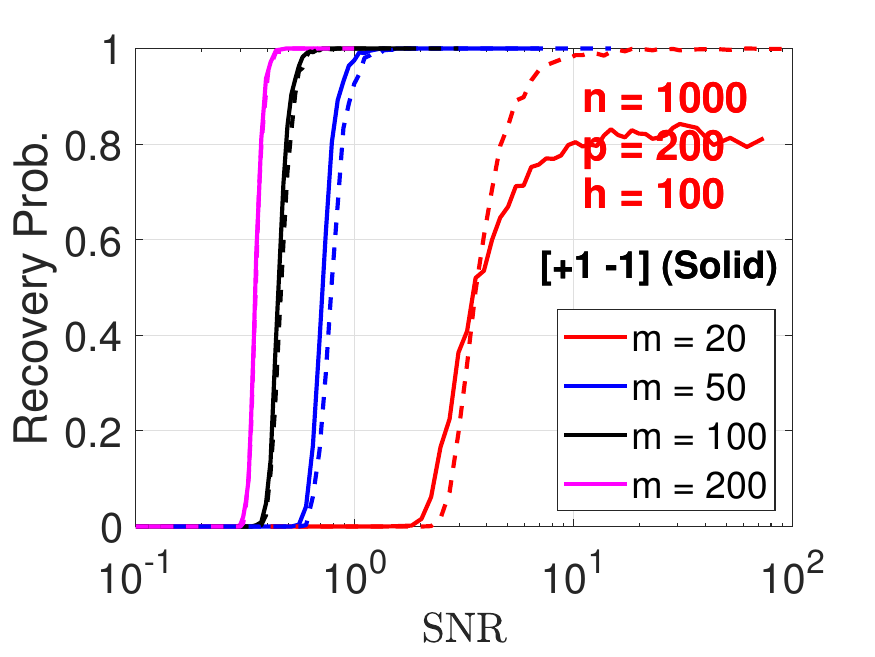}	
\includegraphics[width=2.2in]{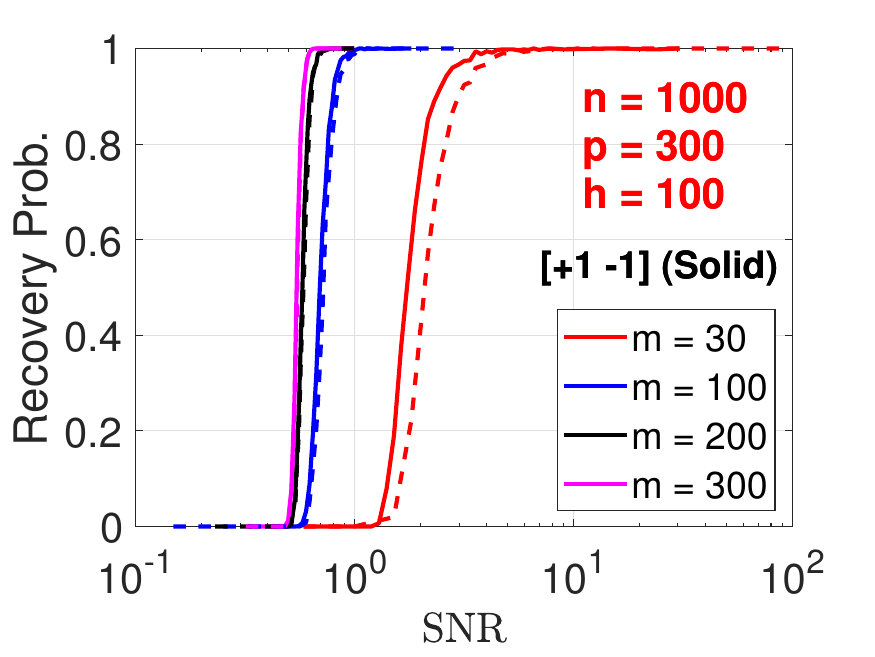}	
}

\mbox{
\includegraphics[width=2.2in]{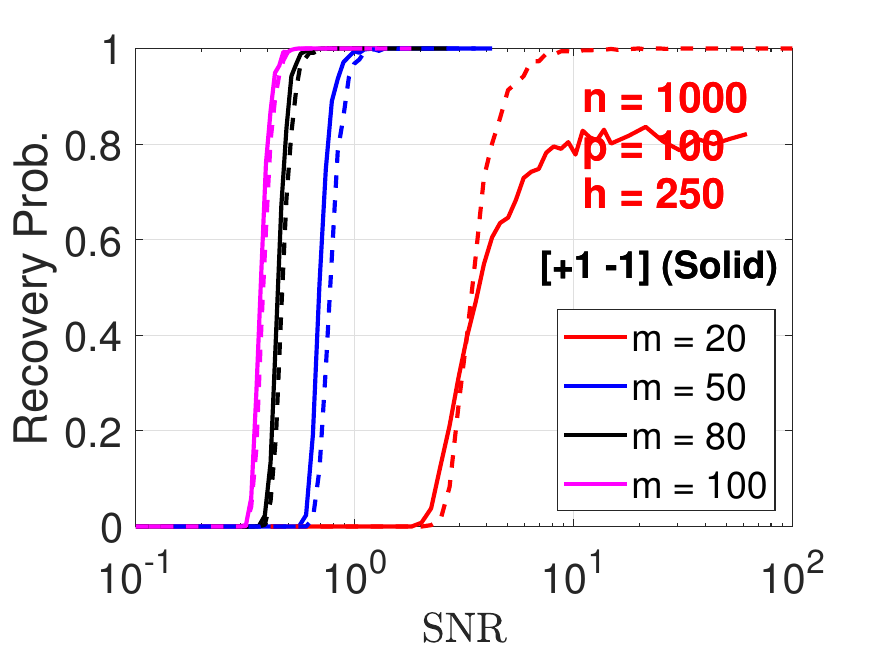}
\includegraphics[width=2.2in]{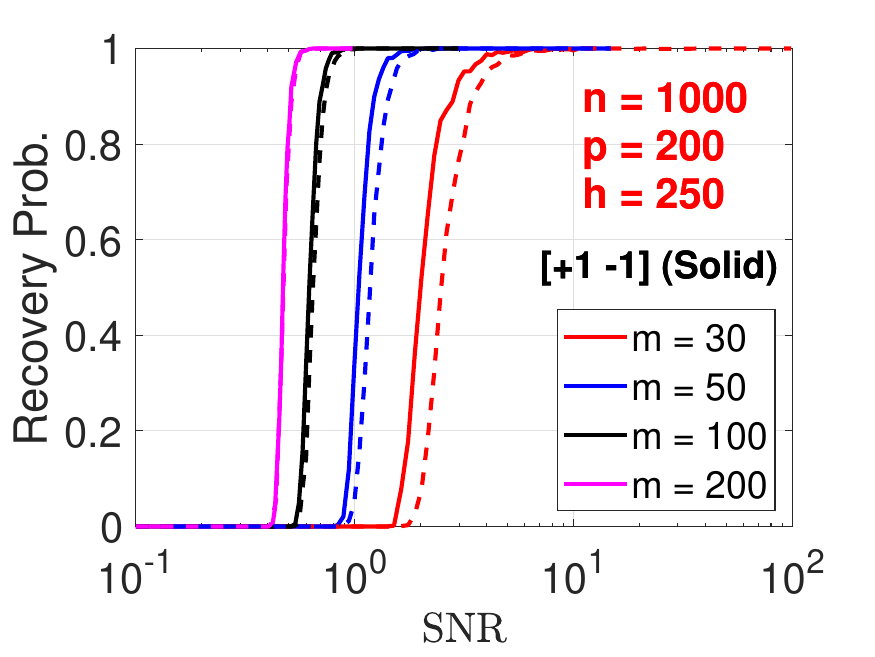}	
\includegraphics[width=2.2in]{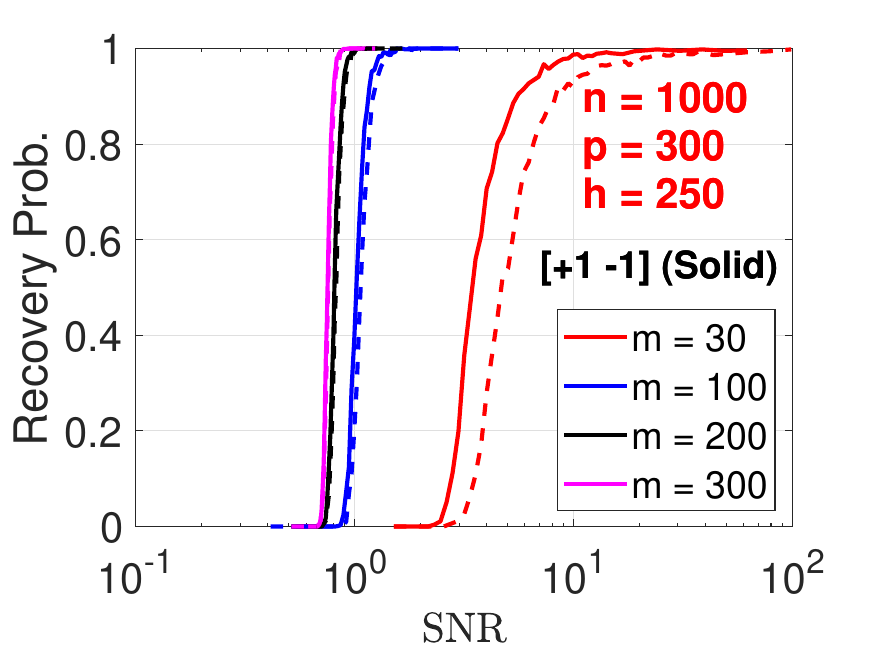}	
}

\end{center}
\caption{Simulated recovery rate $\Prob(\wh{\bPi} = \bPitrue)$ with
$n = 1000$, $p\in \set{100, 200, 300}$, $h\in \set{100, 250}$, and $\bX_{ij}\iid \textup{Rademacher}$, with respect to $\snr$. The dashed line corresponds to the
performance when $\bX_{ij}\iid \normdist(0, 1)$.}
\label{fig:benoulli_n1000}
\end{figure}

\newpage\clearpage

\section{Concluding Remarks}\label{sec:conclusion}

This paper considers unlabeled linear regression
and proposes a one-step estimator,
which is optimal in both
computational and statistical perspectives.
First, we show
our estimator has the same computational complexity
as that of oracle estimators.
Afterwards, we separately investigate its statistical properties
under the single observation model ($m = 1$) and multiple observations
model ($m > 1$).
For the single observation model, our estimator can obtain the
ground truth permutation matrix when the signal is with length one, i.e.,
$p = 1$. Moreover, its $\snr$ requirement matches the minimax lower bound, to put
more specifically, $\log \snr \gsim \log n$.
For the multiple observations model,
our estimator exhibits much richer behavior, which can be broadly
divided into three regimes.
In the Easy Regime ($\srank{\bBtrue}\gg \log^4 n$), our estimator is minimax optimal
and can obtain the correct permutation
once $\snr \geq \Omega(1)$; in the Medium Regime ($\log n \ll \srank{\bBtrue}\ll \log^4 n$),
our estimator experiences a performance loss up to
a multiplicative polynomial of $\log n$ as it
requires $\snr  \gsim \log n$ rather than $\snr \gsim \Omega(1)$
for correct permutation recovery;
and in the \textbf{hard regime} ($1\ll \srank{\bBtrue} \ll \log n$), our estimator
needs $\snr \gsim (\log n)^{c_0} \cdot n^{\nfrac{c_1}{\srank{\bBtrue}}}$, which
matches the statistical limits when
$\srank{\bBtrue} \ll \frac{\log n}{\log \log n}$ and
experiences a loss up to a multiplicative polynomial
of $\log n$ when $\frac{\log n}{\log\log n} \ll \srank{\bBtrue} \ll \log n$.
Comparing with the single observation model, we conclude that
high diversity, namely, large $\srank{\bBtrue}$, can
greatly facilitate the permutation recovery.
Ultimately, we provide numerical experiments to support all claims thereof.

\vspace{0.5in}

\bibliographystyle{plainnat}
\bibliography{refs_scholar}

\begin{thebibliography}{38}
\providecommand{\natexlab}[1]{#1}
\providecommand{\url}[1]{\texttt{#1}}
\expandafter\ifx\csname urlstyle\endcsname\relax
  \providecommand{\doi}[1]{doi: #1}\else
  \providecommand{\doi}{doi: \begingroup \urlstyle{rm}\Url}\fi

\bibitem[Abid et~al.(2017)Abid, Poon, and Zou]{abid2017linear}
Abubakar Abid, Ada Poon, and James Zou.
\newblock Linear regression with shuffled labels.
\newblock \emph{arXiv preprint arXiv:1705.01342}, 2017.

\bibitem[Bai and Hsing(2005)]{bai2005broken}
Zhidong Bai and Tailen Hsing.
\newblock The broken sample problem.
\newblock \emph{Probability Theory and Related Fields}, 131\penalty0
  (4):\penalty0 528--552, 2005.

\bibitem[Balakrishnan et~al.(2017)Balakrishnan, Wainwright, and
  Yu]{balakrishnan2017statistical}
Sivaraman Balakrishnan, Martin~J. Wainwright, and Bin Yu.
\newblock Statistical guarantees for the {EM} algorithm: From population to
  sample-based analysis.
\newblock \emph{The Annals of Statistics}, 45\penalty0 (1):\penalty0 77--120,
  02 2017.

\bibitem[Bertsekas and Casta{\~{n}}{\'{o}}n(1992)]{bertsekas1992forward}
Dimitri~P. Bertsekas and David~A. Casta{\~{n}}{\'{o}}n.
\newblock A forward/reverse auction algorithm for asymmetric assignment
  problems.
\newblock \emph{Comput. Optim. Appl.}, 1\penalty0 (3):\penalty0 277--297, 1992.

\bibitem[Cand{\`{e}}s and Plan(2010)]{candes2010matrix}
Emmanuel~J. Cand{\`{e}}s and Yaniv Plan.
\newblock Matrix completion with noise.
\newblock \emph{Proc. {IEEE}}, 98\penalty0 (6):\penalty0 925--936, 2010.

\bibitem[Chen et~al.(2020)Chen, Chi, Fan, Ma, and Yan]{chen2020noisy}
Yuxin Chen, Yuejie Chi, Jianqing Fan, Cong Ma, and Yuling Yan.
\newblock Noisy matrix completion: Understanding statistical guarantees for
  convex relaxation via nonconvex optimization.
\newblock \emph{{SIAM} J. Optim.}, 30\penalty0 (4):\penalty0 3098--3121, 2020.

\bibitem[Chi et~al.(2019)Chi, Lu, and Chen]{chi2019nonconvex}
Yuejie Chi, Yue~M. Lu, and Yuxin Chen.
\newblock Nonconvex optimization meets low-rank matrix factorization: An
  overview.
\newblock \emph{{IEEE} Trans. Signal Process.}, 67\penalty0 (20):\penalty0
  5239--5269, 2019.

\bibitem[DeGroot and Goel(1976)]{degroot1976matching}
Morris~H. DeGroot and Prem~K. Goel.
\newblock The matching problem for multivariate normal data.
\newblock \emph{Sankhy{\=a}: The Indian Journal of Statistics, Series B
  (1960-2002)}, 38\penalty0 (1):\penalty0 14--29, 1976.

\bibitem[DeGroot and Goel(1980)]{degroot1980estimation}
Morris~H. DeGroot and Prem~K. Goel.
\newblock Estimation of the correlation coefficient from a broken random
  sample.
\newblock \emph{The Annals of Statistics}, 8\penalty0 (2):\penalty0 264--278,
  03 1980.

\bibitem[Dokmanic(2019)]{domankic2018permutations}
Ivan Dokmanic.
\newblock Permutations unlabeled beyond sampling unknown.
\newblock \emph{{IEEE} Signal Process. Lett.}, 26\penalty0 (6):\penalty0
  823--827, 2019.

\bibitem[Fang and Li(2023)]{fang2023regression}
Guanhua Fang and Ping Li.
\newblock Regression with label permutation in generalized linear model.
\newblock In \emph{Proceedings of the International Conference on Machine
  Learning (ICML)}, pages 9716--9760, Honolulu, HI, 2023.

\bibitem[Goel(1975)]{goel1975re}
Prem~K. Goel.
\newblock On re-pairing observations in a broken random sample.
\newblock \emph{The Annals of Statistics}, 3\penalty0 (6):\penalty0 1364--1369,
  11 1975.

\bibitem[Golub and Van~Loan(2012)]{golub2012matrix}
Gene~H Golub and Charles~F Van~Loan.
\newblock \emph{Matrix computations}, volume~3.
\newblock JHU press, 2012.

\bibitem[Haghighatshoar and Caire(2018)]{haghighatshoar2018signal}
Saeid Haghighatshoar and Giuseppe Caire.
\newblock Signal recovery from unlabeled samples.
\newblock \emph{{IEEE} Trans. Signal Process.}, 66\penalty0 (5):\penalty0
  1242--1257, 2018.

\bibitem[Hsu et~al.(2017)Hsu, Shi, and Sun]{hsu2017linear}
Daniel~J. Hsu, Kevin Shi, and Xiaorui Sun.
\newblock Linear regression without correspondence.
\newblock In \emph{Advances in Neural Information Processing Systems (NIPS)},
  pages 1531--1540, Long Beach, CA, 2017.

\bibitem[Jeong et~al.(2020)Jeong, Dytso, Cardone, and
  Poor]{jeong2020recovering}
Min{-}Oh Jeong, Alex Dytso, Martina Cardone, and H.~Vincent Poor.
\newblock Recovering data permutations from noisy observations: The linear
  regime.
\newblock \emph{{IEEE} J. Sel. Areas Inf. Theory}, 1\penalty0 (3):\penalty0
  854--869, 2020.

\bibitem[Karoui(2013)]{karoui2013asymptotic}
Noureddine~El Karoui.
\newblock Asymptotic behavior of unregularized and ridge-regularized
  high-dimensional robust regression estimators: rigorous results.
\newblock \emph{arXiv preprint arXiv:1311.2445}, 2013.

\bibitem[Karoui(2018)]{karoui2018impact}
Noureddine~EL Karoui.
\newblock On the impact of predictor geometry on the performance on
  high-dimensional ridge-regularized generalized robust regression estimators.
\newblock \emph{Probability Theory and Related Fields}, 170\penalty0
  (1-2):\penalty0 95--175, 2018.

\bibitem[Karoui et~al.(2013)Karoui, Bean, Bickel, Lim, and Yu]{el2013robust}
Noureddine~El Karoui, Derek Bean, Peter~J Bickel, Chinghway Lim, and Bin Yu.
\newblock On robust regression with high-dimensional predictors.
\newblock \emph{Proceedings of the National Academy of Sciences}, 110\penalty0
  (36):\penalty0 14557--14562, 2013.

\bibitem[Kuhn(1955)]{kuhn1955hungarian}
Harold~W Kuhn.
\newblock {The Hungarian method for the assignment problem}.
\newblock \emph{Naval research logistics quarterly}, 2\penalty0 (1-2):\penalty0
  83--97, 1955.

\bibitem[Latala et~al.(2007)Latala, Mankiewicz, Oleszkiewicz, and
  Tomczak{-}Jaegermann]{latala2007banach}
Rafal Latala, Piotr Mankiewicz, Krzysztof Oleszkiewicz, and Nicole
  Tomczak{-}Jaegermann.
\newblock Banach-mazur distances and projections on random subgaussian
  polytopes.
\newblock \emph{Discret. Comput. Geom.}, 38\penalty0 (1):\penalty0 29--50,
  2007.

\bibitem[Pananjady et~al.(2017)Pananjady, Wainwright, and
  Courtade]{pananjady2017denoising}
Ashwin Pananjady, Martin~J Wainwright, and Thomas~A Courtade.
\newblock Denoising linear models with permuted data.
\newblock In \emph{Proceedings of the 2017 IEEE International Symposium on
  Information Theory (ISIT)}, pages 446--450, Aachen, Germany, 2017.

\bibitem[Pananjady et~al.(2018)Pananjady, Wainwright, and
  Courtade]{pananjady2018linear}
Ashwin Pananjady, Martin~J. Wainwright, and Thomas~A. Courtade.
\newblock Linear regression with shuffled data: Statistical and computational
  limits of permutation recovery.
\newblock \emph{{IEEE} Trans. Inf. Theory}, 64\penalty0 (5):\penalty0
  3286--3300, 2018.

\bibitem[Paouris(2012)]{paouris2012small}
Grigoris Paouris.
\newblock Small ball probability estimates for log-concave measures.
\newblock \emph{Transactions of the American Mathematical Society},
  364\penalty0 (1):\penalty0 287--308, 2012.

\bibitem[Peng et~al.(2021)Peng, Wang, and Tsakiris]{peng2021homomorphic}
Liangzu Peng, Boshi Wang, and Manolis~C. Tsakiris.
\newblock Homomorphic sensing: Sparsity and noise.
\newblock In \emph{Proceedings of the 38th International Conference on Machine
  Learning (ICML)}, pages 8464--8475, Virtual Event, 2021.

\bibitem[Slawski and Ben-David(2019)]{slawski2017linear}
Martin Slawski and Emanuel Ben-David.
\newblock Linear regression with sparsely permuted data.
\newblock \emph{Electronic Journal of Statistics}, 1:\penalty0 1--36, 2019.

\bibitem[Slawski and Sen(2022)]{slawski2022permuted}
Martin Slawski and Bodhisattva Sen.
\newblock Permuted and unlinked monotone regression in $\mathbb{R}^d$: an
  approach based on mixture modeling and optimal transport.
\newblock \emph{arXiv preprint arXiv:2201.03528}, 2022.

\bibitem[Slawski et~al.(2020)Slawski, Ben{-}David, and Li]{slawski2020two}
Martin Slawski, Emanuel Ben{-}David, and Ping Li.
\newblock Two-stage approach to multivariate linear regression with sparsely
  mismatched data.
\newblock \emph{J. Mach. Learn. Res.}, 21:\penalty0 204:1--204:42, 2020.

\bibitem[Sur et~al.(2019)Sur, Chen, and Cand{\`e}s]{sur2019likelihood}
Pragya Sur, Yuxin Chen, and Emmanuel~J Cand{\`e}s.
\newblock The likelihood ratio test in high-dimensional logistic regression is
  asymptotically a rescaled chi-square.
\newblock \emph{Probability Theory and Related Fields}, 175\penalty0
  (1):\penalty0 487--558, 2019.

\bibitem[Tang et~al.(2021)Tang, Chang, Ye, and Zha]{tang2021low}
Zhiwei Tang, Tsung-Hui Chang, Xiaojing Ye, and Hongyuan Zha.
\newblock Low-rank matrix recovery with unknown correspondence.
\newblock \emph{arXiv preprint arXiv:2110.07959}, 2021.

\bibitem[Tropp(2015)]{tropp2015introduction}
Joel~A. Tropp.
\newblock An introduction to matrix concentration inequalities.
\newblock \emph{Found. Trends Mach. Learn.}, 8\penalty0 (1-2):\penalty0 1--230,
  2015.

\bibitem[Tsakiris and Peng(2019)]{tsakiris2019homomorphic}
Manolis~C. Tsakiris and Liangzu Peng.
\newblock Homomorphic sensing.
\newblock In Kamalika Chaudhuri and Ruslan Salakhutdinov, editors,
  \emph{Proceedings of the 36th International Conference on Machine Learning
  (ICML)}, pages 6335--6344, Long Beach, CA, 2019.

\bibitem[Unnikrishnan et~al.(2015)Unnikrishnan, Haghighatshoar, and
  Vetterli]{unnikrishnan2015unlabeled}
Jayakrishnan Unnikrishnan, Saeid Haghighatshoar, and Martin Vetterli.
\newblock Unlabeled sensing: Solving a linear system with unordered
  measurements.
\newblock In \emph{Proceedings of the 53rd Annual Allerton Conference on
  Communication, Control, and Computing (Allerton)}, pages 786--793,
  Monticello, IL, 2015.

\bibitem[Vershynin(2018)]{vershynin2018high}
Roman Vershynin.
\newblock \emph{High-dimensional probability: An introduction with applications
  in data science}, volume~47.
\newblock Cambridge university press, 2018.

\bibitem[Zhang and Li(2020)]{zhang2020optimal}
Hang Zhang and Ping Li.
\newblock Optimal estimator for unlabeled linear regression.
\newblock In \emph{Proceedings of the 37th International Conference on Machine
  Learning (ICML)}, pages 11153--11162, Virtual Event, 2020.

\bibitem[Zhang and Li(2023{\natexlab{a}})]{zhang2023greed}
Hang Zhang and Ping Li.
\newblock Greed is good: correspondence recovery for unlabeled linear
  regression.
\newblock In \emph{Proceedings of the Conference on Uncertainty in Artificial
  Intelligence (UAI)}, pages 2509--2518, Pittsburgh, PA, 2023{\natexlab{a}}.

\bibitem[Zhang and Li(2023{\natexlab{b}})]{zhang2023one}
Hang Zhang and Ping Li.
\newblock One-step estimator for permuted sparse recovery.
\newblock In \emph{Proceedings of the International Conference on Machine
  Learning (ICML)}, pages 41244--41267, Honolulu, HI, 2023{\natexlab{b}}.

\bibitem[Zhang et~al.(2022)Zhang, Slawski, and Li]{zhang2022permutation}
Hang Zhang, Martin Slawski, and Ping Li.
\newblock The benefits of diversity: Permutation recovery in unlabeled sensing
  from multiple measurement vectors.
\newblock \emph{{IEEE} Trans. Inf. Theory}, 68\penalty0 (4):\penalty0
  2509--2529, 2022.

\end{thebibliography}

\newpage
\pagebreak
\appendix
\section{Appendix for Section~\ref{sec:single_obser}}
\label{sec:single_proof_appendix}

This section focuses on the special case where
$p=1$ and  $m =1$.
Consider $\bx\in \RR^n$ to be an
isotropic log-concave random vector
with zero mean and $\norm{\bx}{\psi_2} \lsim 1$.
Additionally, we assume that the permutation matrix $\bPitrue$
satisfies $\dH(\bI, \bPitrue) = h \leq \nfrac{n}{4}$.

\subsection{Notations: single observation model}
First, we define the following events $\calE_i$, $(1\leq i \leq 4)$,
\begin{align*}
\calE_{1} \defequal & \set{
\langle \bx, \bPitrue\bx \rangle \geq c_0 n };
\\
\calE_2 \defequal &
\set{
\bw^{\rmt}\bx\bx^{\rmt}(\bPitrue - \bPi ) \bw \lsim \sigma^2 n^2\log n,~\forall~\bPi \neq \bPitrue}; \\
\calE_3 \defequal &
\set{
\abs{\langle \bw, \bx\rangle  \langle \bPitrue \bx, (\bPitrue - \bPi)^{\rmt}\bx\rangle   + \
\langle \bw, (\bPitrue - \bPi)^{\rmt}\bx\rangle
\langle \bPitrue \bx, \bx \rangle
}\lsim \sigma n^2 \sqrt{\log n},~\forall~\bPi \neq \bPitrue
};\\
\calE_{4}\defequal &
\set{
\norm{\bx - \bPi\bx}{2}^2 \gsim n^{-20},~~\forall~\bPi\neq \bPitrue}.
\end{align*}

\subsection{Proof of Theorem~\ref{thm:warm_up}}
\label{subsec:appendix_proof_thm_warm_up}

\begin{proof}
Under the assumptions in Theorem~\ref{thm:warm_up},
we will prove that the ground truth permutation matrix $\bPitrue$ will be returned
with high probability with Algorithm~\ref{alg:one_step_estim}.
Ahead of the technical details, we give an outline of the
proof strategy, which can be divided into two stages.
\begin{itemize}
\item
\textbf{Stage I.}
We show the intersection of events
$\bigcap_{i=1}^5 \calE_i$ is a subset of the
event $\{\bPi^{\opt} = \bPitrue\}$ under the assumptions
of Theorem~\ref{thm:warm_up}.

\item
\textbf{Stage II}.
With the union bound,
we can upper-bound the error probability
$\Prob(\bPi^{\opt }\neq \bPitrue)$
by $\sum_{\ell =1}^4 \Prob(\br{\calE}_{\ell})$.
The proof is then completed by studying
probability $\Prob(\br{\calE}_{(\cdot)})$, respectively.
\end{itemize}
Technical details come as follows.

\vspace{0.1in}\noindent
\textbf{Stage I.}
We begin the proof by showing
$\bigcap_{\ell =1}^4 \calE_{\ell} \subseteq \{\bPiopt = \bPitrue\}$ under the assumptions of Theorem~\ref{thm:warm_up}.
First, we expand $\langle \bPi, \by \by^{\rmt}\bx\bx^{\rmt} \rangle$
as
\[
\langle \bPi, \by \by^{\rmt}\bx\bx^{\rmt} \rangle =
(\betatrue)^2 \calT_2(\bPi) +  \betatrue \calT_1(\bPi) +
\calT_0(\bPi),
\]
where $\bPi$ is an arbitrary permutation matrix, and
$\calT_{i}(\bPi)$ $(0\leq i \leq 2)$ are defined as
\[
\calT_2(\bPi) &=
\langle \bPitrue\bx, \bPi \bx\rangle \langle \bPitrue\bx, \bx \rangle; \\
\calT_1(\bPi) &=
\la \bw, \bx\ra \langle \bPitrue \bx, \bPi^{\rmt}\bx\rangle  + \
\langle \bw, \bPi^{\rmt}\bx\rangle  \langle \bPitrue \bx, \bx \rangle; \\
\calT_{0}(\bPi) &=
\langle  \bw, \bPi^{\rmt}\bx \rangle \langle \bx, \bw\rangle.
\]
Then we can express the difference
$\langle \bPitrue, \by \by^{\rmt}\bx\bx^{\rmt}\rangle -
\langle \bPi, \by \by^{\rmt}\bx\bx^{\rmt}\rangle$ as
\begin{align*}
& \langle \bPitrue, \by \by^{\rmt}\bx\bx^{\rmt}\rangle -
\langle \bPi, \by \by^{\rmt}\bx\bx^{\rmt}\rangle \\
=~& \
(\betatrue)^2 (\calT_2(\bPitrue) -  \calT_2(\bPi)) +
\betatrue (\calT_1(\bPitrue) -\calT_1(\bPi)) +
\calT_0(\bPitrue) - \calT_0(\bPi)
\\
\stackrel{\cirone}{=}~& \frac{(\betatrue)^2}{2}\langle \bPitrue\bx, \bx\rangle
\big\|\bx - \bPi^{\natural\rmt}\bPi\bx \big\|_{2}^2 +
\betatrue(\calT_1(\bPitrue)-\calT_1(\bPi) ) +
\calT_0(\bPitrue) - \calT_0(\bPi)
\end{align*}
where in $\cirone$ we exploit the fact such that
$\bx$ is a vector and hence
\[
\norm{\bx}{2}^2 - \langle \bPitrue\bx, \bPi\bx\rangle =
\bracket{\norm{\bx}{2}^2 +
\big\|\bPi^{\natural\rmt}\bPi \bx\big\|_{2}^2
-2\langle \bPitrue\bx, \bPi\bx\rangle }/2= \
\big\|\bx - \bPi^{\natural\rmt}\bPi \bx\big\|_{2}^2/2.
\]
Note that this relation generally does not hold except
for the vector case. Conditioning on
$\calE_1 \bigcap \calE_4$,
we have the relation
$\calT_2(\bPitrue) -\calT_2(\bPi) \gsim \ell \cdot n^{-19}$
and hence
\[
\langle \bPitrue, \by \by^{\rmt}\bx\bx^{\rmt}\rangle -
\langle \bPi, \by \by^{\rmt}\bx\bx^{\rmt}\rangle
\geq~& \frac{(\betatrue)^2}{2} c_0n n^{-20} - \
\betatrue |\calT_1(\bPitrue) -\calT_1(\bPi)|- \
|\calT_0(\bPitrue) - \calT_0(\bPi)| \\
\stackrel{\cirtwo}{\gsim} ~&  \frac{c_0(\betatrue)^2}{n^{19}} -
c_1\betatrue\sigma n^2 \sqrt{\log n} - c_2\sigma^2 n^2 \log n,
\]
where
in $\cirtwo$ we condition on $\calE_2 \bigcap \calE_3$.
Under the assumptions in Theorem~\ref{thm:warm_up}, i.e.,  $\log \snr \gsim \log n$, we have
\[
\langle \bPitrue, \by \by^{\rmt}\bx\bx^{\rmt}\rangle >
\langle \bPi, \by \by^{\rmt}\bx\bx^{\rmt}\rangle,
~~\forall~~\bPi\neq \bPitrue,
\]
which suggests the correct permutation $\bPitrue$ can always be obtained and
we can upper bound the error probability
$\Prob(\bPiopt \neq \bPitrue)$ by
$\sum_{\ell = 1}^4 \Prob(\br{\calE}_{\ell})$.

\vspace{0.1in}
\noindent
\textbf{Stage II.}
We upper bound the error probability $\Prob(\bPiopt \neq \bPitrue)$ by $\sum_{\ell=1}^4 \Prob(\br{\calE}_{\ell})$
and complete the proof with
Lemma~\ref{lemma:xpix_proximity},
Lemma~\ref{lemma:event2},
Lemma~\ref{lemma:event3}, and Lemma~\ref{lemma:event4}.
\end{proof}

\subsection{Supporting lemmas for Theorem~\ref{thm:warm_up}}
\label{subsec:appendix_proof_warm_up_support_lemma}

This subsection collects the supporting lemmas
for the proof of Theorem~\ref{thm:warm_up}.

\begin{lemma}
\label{lemma:xpix_proximity}
We have $\Prob\bracket{\calE_1}\geq 1-  c_0 e^{-c_1 n}$ when
$n$ is sufficiently large, where $c_0, c_1 > 0$ are some positive constants. 	
\end{lemma}

\begin{proof}
W.l.o.g, we assume the first $h$ entries are permuted and
 expand the inner product $\langle \bx, \bPitrue \bx\rangle$ as
\[
\langle \bx, \bPitrue\bx\rangle = \sum_{i =1}^h x_i x_{\pi^{\natural}(i)}
+ \sum_{i= h+1}^n x_i^2.
\]
With union bound, we can upper bound
$\Prob\bracket{\langle \bx, \bPitrue\bx \rangle \leq c_0 n}$ as
\begin{align}
\label{eq:xpix_proximity_tot}
& \Prob\bracket{\langle \bx, \bPitrue\bx \rangle \leq c_0 n} \stackrel{\cirone}{\leq} \
\underbrace{\Prob\bracket{\sum_{i=1}^h x_i x_{\pi^{\natural}(i)} \leq  -\frac{c_0 n}{2} }}_{\defequal~\zeta_1} +
\underbrace{\Prob\bracket{\sum_{i=h+1}^n x_i^2 \leq \frac{3c_0 n}{2}}}_{\defequal~\zeta_2}.
\end{align}
We finish the proof by separately upper-bounding
$\zeta_1$ and $\zeta_2$ as $\zeta_1, \zeta_2 \lsim e^{-c_1 n}$.
The detailed computation comes as follows.

\vspace{0.1in}

\noindent
\textbf{Analysis of $\zeta_1$}.
The technical difficulties stem from the correlation
between the terms $x_ix_{\pi^{\natural}(i)}$.
According to Lemma 8 in~\citet{pananjady2018linear} (restated
as Lemma~\ref{lemma:permute_decomp}),
we can divide
the index set $\set{i: i\neq \pi^{\natural}(i)}$
into $3$ disjoint categories $\calI_{\ell}$, $(1\leq \ell \leq 3)$,
such that (i) indices $i$ and $\pi^{\natural}(i)$
belongs to different categories; (ii) the cardinality $h_{\ell}$ of $\calI_{\ell}$ satisfies $h_{\ell}\geq \lfloor h/5 \rfloor$.
Then we make the decomposition
\[
\sum_{i=1}^h x_{i}x_{\pi^{\natural}(i)}
= \sum_{\ell = 1}^3 \sum_{i\in \calI_{\ell}}x_i x_{\pi^{\natural}(i)},
\]
and again, using the union bound, we have
\begin{align}
\label{eq:xpix_proximity_zetaone}
\zeta_1 \leq~& \sum_{\ell=1}^3 \Prob\bigg(\sum_{i\in \calI_{\ell}}x_i x_{\pi^{\natural}(i)}\leq \frac{-c_0 n}{6}\bigg)
\leq  \sum_{\ell=1}^3 \Prob\bigg(\bigg|\sum_{i\in \calI_{\ell}}x_i x_{\pi^{\natural}(i)}\bigg|\geq c_1 n\bigg).
\end{align}
Recalling the fact that $i$ and $\pi^{\natural}(i)$ belong
to different categories, we have
$\sum_{i\in \calI_{\ell}}x_i x_{\pi^{\natural}(i)}$
to be identically distributed as $\langle \bz_1, \bz_2\rangle$,
where $\bz_1,\bz_2 \in \RR^{h_{\ell}}$ are i.i.d
sub-gaussian random vectors.
Then we obtain
\[
\Prob\bigg(\bigg|\sum_{i\in \calI_{\ell}}x_i x_{\pi^{\natural}(i)}\bigg|\geq c_1 n\bigg) \leq~&  \Prob\big(\norm{\bz_2}{2}\gsim \sqrt{h_{\ell}\log n}\big)
+ \Prob\big(|\langle \bz_1, \bz_2 \rangle|\geq c_1 n,~\norm{\bz_2}{2}\lsim \sqrt{h_{\ell}\log n}\big) \\
\stackrel{\cirtwo}{\leq}~& e^{-cn} + 2\exp\bracket{-\frac{c_1^{2}n^2}{h_{\ell}\log n}} \stackrel{\cirthree}{\lsim} e^{-cn},
\]
where $\cirtwo$ is due to the tail bound of sub-gaussian RVs
and $\cirthree$ is because $h_{\ell} \leq h \lsim n$.
Thus we show $\zeta_1 \lsim e^{-c n}$.

\vspace{0.1in}
\noindent
\textbf{Analysis of $\zeta_2$}.
Its analysis is a direct consequence of Lemma~\ref{lemma:chi_square}
as $\sum_{i=h+1}^n x_i^2$ can be viewed as a
$\chi^2$ RV with freedom $n-h$. In formulae:
\begin{align}
\label{eq:xpix_proximity_zetatwo}
\zeta_2\leq~&
\Prob\bracket{\abs{\sum_{i=h+1}^n x_i^2 - (n-h)} \geq \frac{n-h}{2}}
\stackrel{\cirfour}{\leq}
\exp\Bracket{-c_0\bracket{\frac{n-h}{\opnorm{\bI_{(n-h)\times (n-h)}}}\vcap \frac{(n-h)^2}{\fnorm{\bI_{(n-h)\times (n-h)}}^2}}} \notag \\
\stackrel{\cirfive}{=}~& e^{-cn},
\end{align}
where in $\cirfour$ we use the Hanson-Wright inequality (Theorem~$6.2.1$ in~\citet{vershynin2018high}) and
in $\cirfive$ we use the fact $h \lsim n$.
The proof is thus completed by combining
\eqref{eq:xpix_proximity_tot}, \eqref{eq:xpix_proximity_zetaone}
and \eqref{eq:xpix_proximity_zetatwo}.
\end{proof}

\begin{lemma}
\label{lemma:event2}
We have $\Prob\bracket{\calE_2} \geq 1 - c_0 e^{-c_1 n} - c_2
n^{-n}$, where $c_0, c_1$ and $c_2$ are some positive constants.
\end{lemma}

\begin{proof}
We begin the proof with the union bound, which
proceeds as
\begin{align}
\label{eq:event2_tot_init}
\Prob(\br{\calE}_2) \leq
\Prob(\|\bx\|_2 \geq \sqrt{2n})
+ \Prob\bracket{\br{\calE}_2,~\|\bx\|_2 \leq \sqrt{2n}}
\leq c_0 e^{-c_1n} + \Prob\bracket{\br{\calE}_2,~\|\bx\|_2 \leq \sqrt{2n}}.
\end{align}
To upper-bound $\Prob\bracket{\br{\calE}_2,~\|\bx\|_2 \leq \sqrt{2n}}$,
we first consider a fixed permutation matrix $\bPi_0$
such that $\bPi_0 \neq \bPitrue$.
In addition, we define $\bM$ as
$\bx\bx^{\rmt}(\bPitrue - \bPi_0)$.
Due to the independence of the $\bx$ and $\bw$, we
have
\[
\Prob\bracket{\br{\calE}_2,~\|\bx\|_2 \leq \sqrt{2n}}
\stackrel{\cirone}{\leq}~&\Prob\bracket{
|\bw^{\rmt}\bM \bw -  \Expc \bw^{\rmt}\bM \bw|\geq
c \sigma^2 n^2 \log n,~\|\bx \|_2 \leq \sqrt{2 n}},
\]
where in $\cirone$ we condition on $\|\bx\|_2\leq \sqrt{2n}$
and use the fact
\begin{align*}
\Expc \bw^{\rmt}\bM\bw + c\sigma^2 n^2 \log n =\
\sigma^2\trace(\bM) +  c\sigma^2 n^2 \log n
\lsim \sigma^2 \|\bx\|^2_2 + c\sigma^2 n^2 \log n
\lsim
\sigma^2 n^2 \log n.
\end{align*}
Using  Hanson-Wright inequality (Theorem~$6.2.1$ in~\citet{vershynin2018high}), we obtain
\[
& \Prob\bracket{
|\bw^{\rmt}\bM \bw - \Expc \bw^{\rmt}\bM \bw | \geq
c \sigma^2 n^2 \log n,~\|\bx \|_2 \leq \sqrt{2 n}}  \\
\leq~& 2\Expc\set{\exp\Bracket{-\bracket{ \frac{c_0 n^4\log^2 n}{\fnorm{\bM}^2} \vcap \frac{c_1n^2\log n}{\opnorm{\bM}}}}
\Ind(\|\bx\|_2\leq \sqrt{2n})}.
\]
Since $\bM$ is a rank-$1$ matrix, we have
$\fnorm{\bM} = \opnorm{\bM}$.
Conditioning on the event $\|\bx\|_2\leq \sqrt{2n}$,
we have $\opnorm{\bM} \lsim \|\bx\|^2_2 \lsim n$ and
hence $\Prob\bracket{\bw^{\rmt}\bM \bw \gsim \sigma^2 n^2 \log n} \lsim n^{-3n}$ for a fixed $\bPi_0$. Iterating over all
possible $\bPi \neq \bPitrue$, we use \eqref{eq:event2_tot_init} and complete the proof as
\begin{align}
\label{eq:event2_tot}
\Prob(\br{\calE}_2) \leq c_0 e^{-c_1n}
+ |\calP_n|\cdot \Prob(\br{\calE}_2, \|\bx\|_2 \leq \sqrt{2n})
\leq c_0 e^{-c_1 n} + c_2 \cdot n!\cdot n^{-3n}
\stackrel{\cirtwo}{\lsim} e^{-cn} + n^{-n},
\end{align}
where we use the Stirling's approximation, namely,
$n!\sim \sqrt{2\pi n}(n/e)^n$, in $\cirtwo$.
\end{proof}

\begin{lemma}
\label{lemma:event3}
We have $\Prob\bracket{\calE_3} \geq 1 - c_0 e^{-c_1 n} - c_2 n^{-n}$, where $c_0, c_1$, and $c_2$ are some positive constants.
\end{lemma}

\begin{proof}
The proof follows a similar strategy as in
that of Lemma~\ref{lemma:event2}.
Defining $\bu_{\bPi} \in \RR^n$ as
\[
\bu_{\bPi} \defequal \langle \bPitrue\bx, (\bPitrue - \bPi)^{\rmt}\bx \rangle \cdot \bx + \langle \bPitrue\bx, \bx\rangle (\bPitrue -\bPi)^{\rmt}\bx,
\]
we first rewrite the $\Prob\bracket{\br{\calE}_3}$ as
\[
\Prob\bracket{\br{\calE}_3}
= \Prob\bracket{|\la \bu_{\bPi}, \bw \ra |\gsim \sigma n^2\sqrt{\log n},~\exists~\bPi\neq \bPitrue}.
\]
With the union bound, we can obtain
\[
\Prob\bracket{\br{\calE}_3}
\leq e^{-cn} + \Prob\bracket{\br{\calE}_3,~\|\bx\|_2\leq \sqrt{2n}}
\leq e^{-cn} + |\calP_n|\cdot
\Prob\bracket{|\la \bu_{\bPi}, \bw \ra |\gsim \sigma n^2\sqrt{\log n},~\|\bx\|_2\leq \sqrt{2n}}.
\]
Conditioning on $\{\|\bx\|_2 \leq \sqrt{2n}\}$, we have the relation
$\|\bu_{\bPi}\|_2^2 \lsim
\sigma^2 (4\norm{\bx}{2}^3)^2 =
c \sigma^2 n^3$ and hence
\[
\Prob\bracket{|\la \bu_{\bPi}, \bw \ra |\gsim \sigma n^2\sqrt{\log n},~\|\bx\|_2\leq \sqrt{2n}}
\leq ~& 2\cdot \Expc_{\bu_{\bPi}}
\Bracket{\exp\bracket{-\frac{c \sigma^2 n^4 \log n}{2\|\bu_{\bPi}\|^2_2}}\Ind(\| \bu_{\bPi}\|^2_2 \lsim \sigma^2 n^3)} \\
\leq~& 2\exp\bracket{-\frac{c \sigma^2 n^4 \log n}{2\sigma^2 n^3}}
= 2\cdot n^{-3n}.
\]
Following the same logic as \eqref{eq:event2_tot},
we complete the proof.

\end{proof}

\begin{lemma}
\label{lemma:event4}
We have $\Prob\bracket{\calE_4} \geq 1- c_0 \cdot n^{-c_1}$, where
 $c_0, c_1 > 0$ are some positive constants.
\end{lemma}

\begin{proof}
To begin with, we notice the relation
\begin{align}
\label{lemma:event4_tot}
\Prob(\br{\calE}_4)
\stackrel{\cirone}{\leq}~& \Prob\bracket{\min_{s\neq t} |\bx_s - \bx_t|^2 \lsim n^{-20}}
\leq n^2\cdot \Prob\bracket{|\bx_s - \bx_t|\lsim n^{-10}} \notag  \\
\leq~& n^2 \cdot \Bracket{\Prob(|\bx_s| \lsim n^{-10}) +  \Prob(|\bx_t| \lsim n^{-10})}
= 2n^2 \cdot \Prob(|\bx_s|\lsim n^{-20}),
\end{align}
where $\cirone$ is due to the relation
$\norm{\bx - \bPi\bx}{2}^2 \geq \min_{s\neq t}|\bx_s - \bx_t|^2$ as long as  $\bPi \neq \bI$.
Invoking \eqref{lemma:small_ball_log_concave}, we have
\[
\Prob(|\bx_s|\lsim n^{-20})\lsim \exp\bracket{-c\log n} = n^{-c},
\]
and complete the proof when combining with \eqref{lemma:event4_tot}.

\end{proof}

\section{Appendix for Section~\ref{sec:multi_observe}}
\label{sec:multi_proof_appendix}
This section provides a theoretical analysis
for the multiple observations model, i.e., $m > 1$.
Without specification, only the centered sub-gaussian
assumption is put on $\bX_{ij}$ and the
log-concavity is not assumed.

\subsection{Notations: multiple observations model}
\label{subsec:multi_proof_appendix_notation}
We begin the proof by defining the notations.
First, we define $\wt{\bB}$ and $\wh{\bB}$ as
\[
\wt{\bB} \defequal~& (n-h)^{-1} \bX^{\rmt} \bPitrue \bX \bBtrue; \\
\wh{\bB} \defequal~& (n-h)^{-1} \bX^{\rmt}\bY = \wt{\bB} + (n-h)^{-1} \bX^{\rmt}\bW.
\]
In addition, we define the following events $\calF_{\ell}$ $(1\leq \ell \leq 8)$ as
\[
\calF_1(\bM) &\defequal
\set{\norm{\bM^{\rmt}\bX_{s, :}}{2} \lsim \sqrt{\log n} \Fnorm{\bM}
\textup{ and }\big\|\bM^{\rmt}\bX^{'}_{s, :}\big\|_{2} \lsim \sqrt{\log n} \Fnorm{\bM}~~\forall~1\leq s \leq n}; \\
\calF_{2, 1} & \defequal
\set{\langle \bX_{s, :}, \bX_{t, :}^{'}\rangle \lsim (\log n)\sqrt{p},~~1\leq s, t \leq n}; \\
\calF_{2,2} & \defequal
\set{\langle \bX_{s, :}, \bX_{t, :}\rangle \lsim (\log n)\sqrt{p}, ~~1\leq s \neq t \leq n}; \\
\calF_{2, 3} & \defequal
\set{\langle \bX^{'}_{s, :}, \bX^{'}_{t, :}\rangle \lsim (\log n)\sqrt{p},~~1\leq s \neq t \leq n}; \\
\calF_2 &= \calF_{2,1}\bigcap \calF_{2, 2}\bigcap \calF_{2, 3};\\
\calF_3 &= \set{\norm{\bX_{s, :}}{2} \leq \sqrt{p\log n} \textup{ and }
\|\bX^{'}_{s, :}\|_{2} \leq \sqrt{p\log n},~~\forall~1\leq s \leq n }; \\
\calF_4 &= \set{
\Fnorm{\bX} \leq \sqrt{2np}~\textup{ and }\Fnorm{\sampminus{\bX}{s}}
\leq \sqrt{2np},~~\forall~1\leq s \leq n}; \\
\calF_5 &= \set{
\norm{\bX\bX_{s, :}}{2} \lsim (\log n)\sqrt{np},
~~\forall~1\leq s \leq n }; \\
\calF_{6, \textup{original}} &=
\set{\Fnorm{\bBtrue - \wt{\bB}} \lsim\frac{(\log n)(\log n^2 p^3)\sqrt{p}}{\sqrt{n}}\Fnorm{\bBtrue}}; \\
\calF_{6, \textup{single}} &=
\set{\Fnorm{\bBtrue - \wtminus{\bB}{s}} \lsim\frac{(\log n)(\log n^2 p^3)\sqrt{p}}{\sqrt{n}}\Fnorm{\bBtrue},~~\forall~1\leq s \leq n}; \\
\calF_{6, \textup{double}} &=
\set{\Fnorm{\bBtrue - \wtminus{\bB}{s, t}} \lsim \frac{(\log n)(\log n^2 p^3)\sqrt{p}}{\sqrt{n}}\Fnorm{\bBtrue},~~\forall~1\leq s \neq t \leq n}; \\
\calF_6 &= \calF_{6, \textup{original}}\bigcap \calF_{6,\textup{single}}\bigcap \calF_{6, \textup{double}}; \\
\calF_{7, \textup{single}} &= \
\set{\norm{ (\wt{\bB} - \wtminus{\bB}{s})^{\rmt} \bX_{s, :} }{2}
\lsim \frac{p\log^{\nfrac{3}{2}} n}{n}\Fnorm{\bBtrue},~\forall~1\leq s \leq n};  \\
\calF_{7, \textup{double}} &= \set{
\norm{(\wt{\bB} - \wtminus{\bB}{s, t})^{\rmt}\bX_{s, :} }{2}
\lsim \frac{p\log^{\nfrac{3}{2}} n}{n}\Fnorm{\bBtrue},~\forall~1\leq s\neq t \leq n}; \\
\calF_7 &= \calF_{7, \textup{single}}\bigcap \calF_{7, \textup{double}}; \\
\calF_8 &= \set{\big\|( \wt{\bB} - \bBtrue)^{\rmt} \bX_{s, :}\big\|_{2} \lsim
\frac{(\log n)^{\nfrac{3}{2}}(\log n^2 p^3)\sqrt{p}}{\sqrt{n}}\Fnorm{\bBtrue},~\forall~1\leq s\leq n},
\]
where $\bM$ in $\calF_1(\bM)$ is an arbitrary matrix independent of $\bX_{i, :}$.
In addition, we define the quantities $\Delta_1$,
$\Delta_2$, and $\Delta_3$ as
\begin{align}
\Delta_1 &=
c_0 \sigma(\log^{\nfrac{5}{2}} n)\sqrt{\frac{p}{n}}\Fnorm{\bBtrue}; \label{eq:Delta1_def}\\
\Delta_2 &= c_1\sigma (\log^2 n)\Fnorm{\bBtrue};\label{eq:Delta2_def} \\
\Delta_3 &= c_2\bracket{\frac{mp\sigma^2 (\log^2 n)}{n}
+ \sigma^2(\log^2 n)\sqrt{\frac{mp}{n}}}, \label{eq:Delta3_def}
\end{align}
respectively. Besides, we denote $\Delta_{(\textup{multi, 1})}$ as
$\Delta_1 + \Delta_2 + \Delta_3$ and
$\Delta_{(\textup{multi, 2})}$ as $\Delta_2 + \Delta_3$.

The following context presents the analysis corresponding to each regime, which is organized in
ascending order
of difficulties. To facilitate understanding, we put a diagram illustrating the dependence among lemmas in
Figure~\ref{fig:dependence_diagram}.

\vspace{0.1in}

\begin{figure}[!ht]
\centering
\begin{tikzpicture}

\node (ub1) [draw=black, rounded corners, fill=celadon, minimum width=0.5cm, align=center] at (-4, 1.5) {\textbf{Easy Regime} \\ {\footnotesize sub-gaussian $\&$} \\ {\footnotesize $\srank{\bBtrue}\gg \log^4 n$} };

\node (ub2) [draw=black, rounded corners, fill=celadon, minimum width=0.5cm, align=center] at (0, 1.5) {\textbf{Medium Regime} \\ {\footnotesize sub-gaussian $\&$} \\ {\footnotesize $\srank{\bBtrue}\gg \log n$} };

\node (ub3) [draw=black, rounded corners, fill=celadon, minimum width=0.5cm, align=center] at (4.5, 1.5) {\textbf{Hard Regime} \\ {\footnotesize log-concave sub-gaussian} \\ {\footnotesize  $\&$ $\srank{\bBtrue} \geq c$} };

\node (mtot) [draw=black, rounded corners, fill=chromeyellow, minimum width=0.5cm, align=center] at (-7.5, 0) {Lemma~\ref{lemma:termtot_bound} };

\node (m1) [draw=black, rounded corners, fill=chromeyellow, minimum width=0.5cm, align=center] at (-5, 0) {Lemma~\ref{lemma:term1_bound}};

\node (m2) [draw=black, rounded corners, fill=chromeyellow, minimum width=0.5cm, align=center] at (-2.5, 0) {Lemma~\ref{lemma:term2_bound}};

\node (m3) [draw=black, rounded corners, fill=chromeyellow, minimum width=0.5cm, align=center] at (0, 0) {Lemma~\ref{lemma:term3_bound}};

\node (mrhs) [draw=black, rounded corners, fill=chromeyellow, minimum width=0.5cm, align=center] at (2.5, 0) {Lemma~\ref{lemma:rhs_general}};

\node (mxb) [draw=black, rounded corners, fill=chromeyellow, minimum width=0.5cm, align=center] at (5, 0) {Lemma~\ref{lemma:xw_ub}};

\node (mrhslogconcave) [draw=black, rounded corners, fill=chromeyellow, minimum width=0.5cm, align=center] at (7.5, 0) {Lemma~\ref{lemma:event9_logconcave}};

\node (basenode) [draw=black, rounded corners, fill=columbiablue, minimum width=0.5cm, align=center] at (0, -1.5) {Lemma~\ref{lemma:x_row_norm_ub},~\ref{lemma:inner_product},~\ref{lemma:Xmat_fnorm},~\ref{lemma:event5},~\ref{lemma:Bmat_perturb},~\ref{lemma:xb_single_idx},~\ref{lemma:xb_multi_indices}, and \ref{lemma:beta_perturb}};

\draw [->, thick] (ub1) -- (mtot);
\draw [->, thick] (ub1) -- (m1);
\draw [->, thick] (ub1) -- (m2);
\draw [->, thick] (ub1) -- (m3);

\draw [->, thick] (ub2) -- (m2);
\draw [->, thick] (ub2) -- (m3);
\draw [->, thick] (ub2) -- (mrhs);
\draw [->, thick] (ub2) -- (mxb);

\draw [->, thick] (ub3) -- (m2);
\draw [->, thick] (ub3) -- (m3);
\draw [->, thick] (ub3) -- (mrhslogconcave);
\draw [->, thick] (ub3) -- (mxb);

\draw [->, thick] (mtot) -- (basenode);
\draw [->, thick] (m1) -- (basenode);
\draw [->, thick] (m2) -- (basenode);

\draw [->, thick] (m3) -- (basenode);
\draw [->, thick] (mrhs) -- (basenode);
\draw [->, thick] (mxb) -- (basenode);
\draw [->, thick] (mrhslogconcave) -- (basenode);

\end{tikzpicture}
\caption{Dependence diagram of lemmas.}
\label{fig:dependence_diagram}
\end{figure}
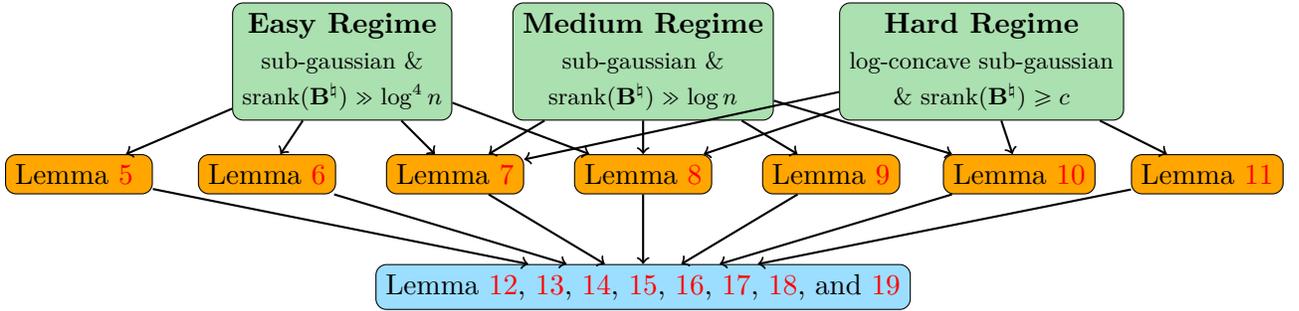

Before proceeding to the technical details, we first restate
our construction method of $\sampminus{\bX}{s}$,
$\wtminus{\bB}{s}$, and $\wtminus{\bB}{s,t}$, respectively.
To begin with, we construct the leave-one-out
sensing matrices. For each row $\bX_{s, :}$ of $\bX$ $(1\leq s \leq n)$,
we draw an independent copy denoted as
$\bX_{s, :}^{'}$. Matrix $\sampminus{\bX}{s}$
is then constructed by $(i)$ replacing the $s$th row in the sensing matrix $\bX$ with $\bX_{s, :}^{'}$ and $(ii)$ copying the rest rows.

In addition, we construct leave-one-out samples $\wtminus{\bB}{s}$ ($1\leq s \leq n$).
This is done by replacing the role of $\bX_{s,:}$ in $\wt{\bB}$ with
its independent copy $\bX_{s, :}^{'}$.
In formulae:
\[
\wtminus{\bB}{s} =
(n-h)^{-1}\bigg(
\sum_{\substack{k \neq s \\ \pi^{\natural}(k) \neq s}} \bX_{\pi^{\natural}(k), :}\bX_{k, :}^{\rmt}
+ \sum_{\substack{k = s \textup{ or} \\  \pi^{\natural}(k) = s}} {\bX}^{'}_{\pi^{\natural}(k), :}{\bX}^{'\rmt}_{k, :}\bigg)\bBtrue.
\]
Easily, we can verify that $\wtminus{\bB}{s}$ is independent of
$\bX_{s, :}$. Similarly, we construct matrices
$\{\wtminus{\bB}{s, t}\}_{1\leq s\neq t \leq n}$ as
\[
\wtminus{\bB}{s, t} =
(n-h)^{-1}\bigg(
\sum_{\substack{k \neq s, t\\ \pi^{\natural}(k) \neq s,t }} \bX_{\pi^{\natural}(k), :}\bX_{k, :}^{\rmt}
+ \sum_{\substack{k = s \textup{ or } k = t \textup{ or} \\ \pi^{\natural}(k) =s
  \textup{ or }\pi^{\natural}(k) = t}} {\bX}^{'}_{\pi^{\natural}(k), :}{\bX}^{'\rmt}_{k, :}\bigg)\bBtrue.
\]
Moreover, we can show that $\wtminus{\bB}{s, t}$ is independent from
the rows $\bX_{s, :}$ and $\bX_{t, :}$ $(1\leq s\neq t \leq n)$.

\subsection{Easy regime: proof of Theorem~\ref{thm:multi_snr_require_general_sub_gauss}}
\begin{proof}
To begin with, we notice that the construction error occurs when there exists some permutation matrix $\bPi \neq \bPitrue$ such that $\langle \bY, \bPi \bX \wh{\bB}  \rangle  \geq \langle \bY, \bPitrue \bX \wh{\bB}\rangle$.

\vspace{0.1in}\noindent
\textbf{Step I.}
Defining event $\calF_{\textup{err-relax}}^{(\textup{multi}, 1)}$ as
\[
\calF_{\textup{err-relax}}^{(\textup{multi}, 1)} \defequal
\set{ \la \bY_{i, :},  \wh{\bB}^{\rmt}\bX_{\pi^{\natural}(i), :} \ra
\leq \la \bY_{i, :}, \wh{\bB}^{\rmt} \bX_{j, :} \ra,~\exists~1\leq \pi^{\natural}(i) \neq  j \leq n},
\]
we first show $\{\bPiopt \neq \bPitrue\}\subseteq \calF_{\textup{err-relax}}^{(\textup{multi}, 1)}$.
The detailed reasoning is as follows.
Conditional on $\br{\calF}_{\textup{err-relax}}^{(\textup{multi}, 1)}$, we have
\[
\langle \bY, \bPitrue \bX \wh{\bB}\rangle =
\sum_i \la \bY_{i, :},  \wh{\bB}^{\rmt}\bX_{\pi^{\natural}(i), :} \ra
\geq \sum_{i} \la \bY_{i, :}, \wh{\bB}^{\rmt} \bX_{\pi(i), :} \ra
= \langle \bY, \bPi \bX \wh{\bB}\rangle,
\]
which means ground truth $\bPitrue$ will be returned by our estimator in
Algorithm~\ref{alg:one_step_estim}, in other words,
$\br{\calF}_{\textup{err-relax}}^{(\textup{multi}, 1)} \subseteq \{\bPiopt = \bPitrue\}$.

\vspace{0.1in}
\noindent
\textbf{Step II.}
Then we will upper bound the error probability with
$\Prob(\calF_{\textup{err-relax}}^{(\textup{multi}, 1)})$.
With the relation
$\bY_{i, :} = \bB^{\natural \rmt}  \bX_{\pi^{\natural}(i), :}+ \bW_{i, :}$
and $\wh{\bB} = \wt{\bB} + (n-h)^{-1}\bX^{\rmt}\bW$, we obtain an equivalent form of
condition
$\langle \bY_{i, :},  \wh{\bB}^{\rmt}\bX_{\pi^{\natural}(i), :} \rangle
\leq \langle \bY_{i, :}, \wh{\bB}^{\rmt} \bX_{j, :} \rangle$ reading as
\begin{align}
\label{eq:greed_optim_condition}
& \la \bB^{\natural \rmt}  \bX_{\pi^{\natural}(i), :}+ \bW_{i, :},
\bracket{\wt{\bB} +(n-h)^{-1}\bX^{\rmt}\bW}^{\rmt}
\bX_{\pi^{\natural}(i), :}
\ra \notag \\
\leq~&
\la \bB^{\natural \rmt}  \bX_{\pi^{\natural}(i), :}+ \bW_{i, :},
\bracket{\wt{\bB} + (n-h)^{-1}\bX^{\rmt}\bW}^{\rmt}
\bX_{j, :}\ra.
\end{align}
For the notation conciseness, we define terms
$\term_i$ ($1\leq i \leq 4$) as
\begin{align}
\term_{\textup{tot}}&= \la \bB^{\natural \rmt}  \bX_{\pi^{\natural}(i), :}, \wt{\bB}^{\rmt}
\bracket{\bX_{\pi^{\natural}(i), :} -  \bX_{j, :}}\ra; \label{eq:term1_def}\\
\term_{1} &= (n-h)^{-1}\la \bB^{\natural \rmt}  \bX_{\pi^{\natural}(i), :}, \bW^{\rmt}\bX\bracket{\bX_{j, :} - \bX_{\pi^{\natural}(i), :}} \ra; \label{eq:term2_def}\\
\term_{2} &= \la \bW_{i, :}, \wt{\bB}^{\rmt}\bracket{\bX_{j, :} - \bX_{\pi^{\natural}(i), :}} \ra;\label{eq:term3_def} \\
\term_{3} &= (n-h)^{-1}\la \bW\bW_{i, :}, \bX\bracket{\bX_{j, :} - \bX_{\pi^{\natural}(i), :}} \ra. \label{eq:term4_def}
\end{align}
Then \eqref{eq:greed_optim_condition} can be rewritten in a concise form, i.e.,
$\term_{\textup{tot}} \leq \term_{1} + \term_{2} + \term_{3}$. With the union bound,
we obtain
\begin{align}
\label{eq:main_theorem_tot}
\Prob(\pi^{\natural}(i) \neq \wh{\pi}(i), \exists~i)
=~& \Expc\Bracket{\Ind\bracket{\term_{\textup{tot}} \leq \term_{1} + \term_{2} + \term_{3},~\exists~i, j}\Ind\bracket{\bigcap_{a=1}^8 \calF_a}} +
\sum_{a=1}^8 \Prob\bracket{\br{\calF}_a} \notag \\
\stackrel{\cirone}{\leq}~& n^2 \cdot \Expc\Bracket{\Ind\bracket{\term_{\textup{tot}} \leq \term_{1} + \term_{2} + \term_{3}}\Ind\bracket{\bigcap_{a=1}^8 \calF_a}} +
c_0 p^{-c_1} + c_2 n^{-c_3},
\end{align}
where in $\cirone$ we invoke Lemma~\ref{lemma:x_row_norm_ub}, Lemma~\ref{lemma:inner_product},
Lemma~\ref{lemma:Xmat_fnorm}, Lemma~\ref{lemma:event5},
Lemma~\ref{lemma:Bmat_perturb}, Lemma~\ref{lemma:xb_single_idx},
Lemma~\ref{lemma:xb_multi_indices}, and Lemma~\ref{lemma:beta_perturb}.

Regarding the term $\Expc\Bracket{\Ind\bracket{\term_{\textup{tot}} \leq \term_{1} + \term_{2} + \term_{3},~\exists~i, j}\Ind\bracket{\bigcap_{a=1}^8 \calF_a}}$, we perform the following decomposition
\begin{align}
\label{eq:main_theorem_union}	
& \Expc\Bracket{\Ind\bracket{\term_{\textup{tot}} \leq \term_{1} + \term_{2} + \term_{3}}\Ind\bracket{\bigcap_{a=1}^8 \calF_a}} \notag \\
\leq~& \Expc\Bracket{\Ind\bracket{\term_{\textup{tot}} \leq \Delta_{(\textup{multi, 1})}}\Ind\bracket{\bigcap_{a=1}^8 \calF_a}} + \Expc\Bracket{\Ind\bracket{\term_{1} \geq \Delta_1}\Ind\bracket{\bigcap_{a=1}^8 \calF_a}}\notag  \\
+~&
\Expc\Bracket{\Ind\bracket{\term_{2} \geq \Delta_2}\Ind\bracket{\bigcap_{a=1}^8 \calF_a}} +
\Expc\Bracket{\Ind\bracket{\term_{3} \geq \Delta_3}\Ind\bracket{\bigcap_{a=1}^8 \calF_a}},
\end{align}
where the definitions of
$\Delta_1$, $\Delta_2$, $\Delta_3$, and $\Delta_{(\textup{multi, 1})}$ are referred to
Subsection~\ref{subsec:multi_proof_appendix_notation}.
The proof is then completed by
combining ~\eqref{eq:main_theorem_tot} and ~\eqref{eq:main_theorem_union}
and invoking
Lemma~\ref{lemma:termtot_bound}, Lemma~\ref{lemma:term1_bound},
Lemma~\ref{lemma:term2_bound}, and Lemma~\ref{lemma:term3_bound}.

\end{proof}

\begin{lemma}
\label{lemma:termtot_bound}
Consider fixed indices $\pi^{\natural}(i)$ and $j$ such that $j\neq \pi^{\natural}(i)$.
Assuming that $(i)$ $\srank{\bBtrue}\gg \log^4 n$,
$(ii)$ $n\gsim p \log^6 n$, $(iii)$
intersection of events
$\calF_1(\bBtrue)\bigcap { \calF_1\big(\bBtrue \wtminus{\bB}{\pi^{\natural}(i), j}^{\rmt} \big) }\bigcap \calF_6 \bigcap \calF_7 \bigcap \calF_8$ holds, and $(iv)$ $\snr \geq c$,
 we have $\{\term_{\textup{tot}} \geq \Delta_{(\textup{multi, 1})}\}$
hold with probability
$1 - n^{-c}$ when $n$ and $p$ are sufficiently large.
Here $\term_{\textup{tot}}$ and $\Delta_{(\textup{multi, 1})}$ are defined in
\eqref{eq:term1_def} and Section~\ref{subsec:multi_proof_appendix_notation}, respectively.
\end{lemma}

\begin{proof}
We begin the proof with the decomposition
\[
\term_{\textup{tot}} = \norm{ \bB^{\natural \rmt}  \bX_{\pi^{\natural}(i), :}}{2}^2 +
\underbrace{\big\langle \bB^{\natural \rmt}  \bX_{\pi^{\natural}(i), :} ,
\bracket{\wt{\bB} - \bBtrue}^{\rmt} \bX_{\pi^{\natural}(i), :}\big\rangle}_{\defequal \term_{\textup{tot}, 1}}
- \underbrace{\la \bB^{\natural \rmt}  \bX_{\pi^{\natural}(i), :}, \wt{\bB}^{\rmt}\bX_{j, :}\ra}_{\defequal \term_{\textup{tot}, 2}}.
\]
Then we obtain
\begin{align}
\label{eq:term1_tot}
\Prob \bracket{\term_{\textup{tot}} \leq \Delta_{(\textup{multi, 1})}}
=~&\Prob\bracket{\frac{\Delta_{(\textup{multi, 1})}}{\norm{ \bB^{\natural \rmt}  \bX_{\pi^{\natural}(i), :}}{2}^2}
- \frac{\term_{\textup{tot}, 1}}{\norm{ \bB^{\natural \rmt}  \bX_{\pi^{\natural}(i), :}}{2}^2}
+ \frac{\term_{\textup{tot}, 2}}{\norm{ \bB^{\natural \rmt}  \bX_{\pi^{\natural}(i), :}}{2}^2}
\geq 1} \notag \\
\leq~&
\underbrace{\Prob\bracket{\norm{ \bB^{\natural \rmt}  \bX_{\pi^{\natural}(i), :}}{2} \leq
\delta}}_{\defequal~\zeta_1} +
\underbrace{\Prob\bracket{
\frac{\Delta_{(\textup{multi, 1})}}{\delta^2} +
\frac{\abs{\term_{\textup{tot}, 1}}}{\delta^2}
+ \frac{\abs{\term_{\textup{tot}, 2}}}{\delta^2} \geq 1}}_{\defequal~\zeta_2}.
\end{align}
Setting $\delta$ as $1/2 \Fnorm{\bBtrue}$,
we separately bound the probabilities $\zeta_1$ and
$\zeta_2$.
For the term $\zeta_1$, we invoke the small ball probability (Lemma~\ref{lemma:small_ball_subgauss}) and conclude
\begin{align}
\label{eq:term1_zeta1}
\Prob\bracket{\norm{ \bB^{\natural \rmt}  \bX_{\pi^{\natural}(i), :}}{2} \leq
\frac{1}{2}\Fnorm{\bBtrue}}\leq e^{-c_0 \cdot \srank{\bBtrue}}.
\end{align}
For probability $\zeta_2$, we will show
it is upper-bounded by $c_1\cdot n^{-c_2}$ provided
$\snr \geq C$.
A detailed explanation comes as follows.

\newpage
\noindent
\textbf{Phase I.}
First, we consider $\term_{\textup{tot}, 1}$. Conditional on the intersection of
events $\calF_1(\bBtrue)\bigcap \calF_8$, we have
\[
\abs{\term_{\textup{tot}, 1}} \leq ~& \norm{\bB^{\natural \rmt} \bX_{i, :}}{2} \cdot  \norm{(\wt{\bB} - \bBtrue)^{\rmt} \bX_{\pi^{\natural}(i), :}}{2} \lsim
\sqrt{\log n}\Fnorm{\bBtrue}
\frac{(\log n)^{\nfrac{3}{2}} (\log n^2 p^3) \sqrt{p}}{\sqrt{n}}
\Fnorm{\bBtrue} \\
=~&
(\log^2 n)(\log n^2 p^3)\sqrt{\frac{p}{n}}\Fnorm{\bBtrue}^2.
\]

\vspace{0.1in}\noindent
\textbf{Phase II.}
Then we turn to  $\term_{\textup{tot}, 2}$.  Adopting the leave-out-out trick, we can
expand it as
\[
\term_{\textup{tot}, 2} =~
\underbrace{\la \bB^{\natural \rmt}  \bX_{\pi^{\natural}(i), :},
(\wt{\bB} - \wtminus{\bB}{\pi^{\natural}(i), j})^{\rmt} \bX_{j, :}\ra}_{\term_{\textup{tot}, 2, 1}}
+ \underbrace{\la \bB^{\natural \rmt}  \bX_{\pi^{\natural}(i), :}, \wtminus{\bB}{\pi^{\natural}(i), j}^{\rmt} \bX_{j, :}\ra}_{\term_{\textup{tot}, 2, 2}}.
\]
For $\term_{\textup{tot}, 2, 1}$, we have
\[
\term_{\textup{tot}, 2, 1} \leq~&
 \norm{ \bB^{\natural \rmt}  \bX_{\pi^{\natural}(i), :}}{2}
\big\|(\wt{\bB} - \wtminus{\bB}{\pi^{\natural}(i), j})^{\rmt} \bX_{j, :}\big\|_{2}
\stackrel{\cirone}{\lsim} \sqrt{\log n}\Fnorm{\bBtrue}
\frac{p\log^{\nfrac{3}{2}} n}{n}\Fnorm{\bBtrue} \\
=~& \frac{p\log^2 n}{n} \Fnorm{\bBtrue}^2,
\]
where in $\cirone$ we condition on event $\calF_{7, \textup{double}}$.
Regarding $\term_{2,2,2}$, we notice that $\wtminus{\bB}{\pi^{\natural}(i), j}$ is independent of the rows $\bX_{\pi^{\natural}(i), :}$ and
$\bX_{j, :}$ due to its construction method.
Hence we can bound $\term_{2,2,2}$ by
conditioning on rows $\{\bX_{s, :}\}_{s\neq \pi^{\natural}}$
and only viewing $\bX_{\pi^{\natural}(i), :}$ as the RV, which
yields
\begin{align}
\label{eq:term122}
\term_{\textup{tot}, 2, 2} \lsim \sqrt{\log n}\norm{\bBtrue \wtminus{\bB}{\pi^{\natural}(i), j}^{\rmt} \bX_{j, :}}{2},
\end{align}
holds with probability $1-n^{-c}$. Conditional on
event $\calF_1(\bBtrue \wtminus{\bB}{\pi^{\natural}(i), j}^{\rmt})$,
we have
\[
\term_{\textup{tot}, 2, 2} \lsim~& \log n\Fnorm{\bBtrue \wtminus{\bB}{\pi^{\natural}(i), j}^{\rmt} }
\lsim \log n \Opnorm{\bBtrue}\Fnorm{\wtminus{\bB}{\pi^{\natural}(i), j}^{\rmt} } \\
\stackrel{\cirtwo}{\leq}~& (\log n)\Opnorm{\bBtrue}
\Bracket{\Fnorm{\wtminus{\bB}{\pi^{\natural}(i), j} - \bBtrue} + \Fnorm{\bBtrue}}
\stackrel{\cirthree}{\lsim} \frac{(\log n)\Fnorm{\bBtrue}^2}{\sqrt{\srank{\bBtrue}}},
\]
where $\cirtwo$ uses the definition of stable rank,
and $\cirthree$ is conditional on event $\calF_6$, $n\geq p$,
and $n\gsim p\log^6 n$.

\vspace{0.1in}\noindent
\textbf{Phase III.}
Conditioning on ~\eqref{eq:term122}, we  expand the sum $\nfrac{\Delta_{(\textup{multi, 1})}}{\delta^2} + \term_{\textup{tot}, 1}/\delta^2 +
\term_{\textup{tot}, 2}/\delta^2$ as
\[
& \frac{\Delta_{(\textup{multi, 1})}}{\delta^2} + \frac{\term_{\textup{tot}, 1}}{\delta^2} +
\frac{\term_{\textup{tot}, 2}}{\delta^2} \\
= ~&
c_0 \sigma(\log n)^{\nfrac{5}{2}}\sqrt{\frac{p}{n}}\frac{1}{\Fnorm{\bBtrue}}
+ \frac{ c_1\sigma (\log^2 n)}{\Fnorm{\bBtrue}}
+ c_2
\bracket{\frac{pm}{n} + \sqrt{\frac{mp}{n}}} \frac{\sigma^2(\log^2 n)}{\Fnorm{\bBtrue}^2} \\
+~& \frac{c_3 (\log^2 n)(\log n^2 p^3)\sqrt{p} }{\sqrt{n}} +
\frac{c_4 p\log^2 n}{n}
+ \frac{c_5 \log n}{\sqrt{\srank{\bBtrue}}}
 \\
\asymp~&
c_0 \sqrt{\frac{p}{nm}}\frac{\bracket{\log n}^{\nfrac{5}{2}}}{\sqrt{\snr}}
+ \frac{c_1 \log^2 n}{\sqrt{m \cdot \snr}}
+ \frac{c_2 p (\log^2 n)}{n\cdot  \snr}
+c_2 \sqrt{\frac{p}{mn}}\frac{\log^2 n }{\snr} \\
+~&\frac{c_3 (\log^2 n)(\log n^2 p^3)\sqrt{p} }{\sqrt{n}} +
\frac{c_4 p \log^2 n}{n}
+ \frac{c_5 \log n}{\sqrt{\srank{\bBtrue}}}.
\]
Provided that $\snr \geq c$, $\srank{\bBtrue} \gg \log^4 n$ and
$n\gsim p \log^6 n$,
we can verify the sum
$\nfrac{\Delta_{(\textup{multi, 1})}}{\delta^2} + \term_{\textup{tot}, 1}/\delta^2 +
\term_{\textup{tot}, 2}/\delta^2 \ll 1$ when
$n$ and $p$ are sufficiently large.
Hence we can conclude
\[
\zeta_2 \leq \Prob\bracket{\term_{\textup{tot}, 2, 2} \gsim \sqrt{\log n}
\norm{\bBtrue \wtminus{\bB}{\pi^{\natural}(i), j}^{\rmt} \bX_{j, :}}{2}
} \leq n^{-c}.
\]
Hence the proof is completed by combining
~\eqref{eq:term1_tot} and ~\eqref{eq:term1_zeta1}.
\end{proof}

\begin{remark}
If we strength the requirement on $\snr$ from $\snr\geq c$ to
$\snr \gsim \log^2 n$, we can relax the requirement
on the stable rank $\srank{\bBtrue}$ from $\srank{\bBtrue}\gg \log^4 n$ to
$\srank{\bBtrue}\gg \log^2 n$.
\end{remark}

\begin{lemma}
\label{lemma:term1_bound}
Conditional on the intersection of events $\calF_1(\bBtrue)\bigcap \calF_5$ and fixing the indices $\pi^{\natural}(i)$ and $j$ $(j\neq \pi^{\natural}(i))$,
we have
\[
\term_{1} \lsim \sigma(\log n)^{\nfrac{5}{2}}\sqrt{\frac{p}{n}}\Fnorm{\bBtrue}.
\]
hold with probability at least $1-n^{-c}$, where $\term_1$ is defined in
\eqref{eq:term2_def}.
\end{lemma}

\begin{proof}
Define vectors $\bu_{\bX}$ and  $\bv_{\bX}^{\rmt}$ as
\[
\bu_{\bX} &= \bX\bracket{\bX_{j, :} - \bX_{\pi^{\natural}(i), :}}, \\
\bv_{\bX} &= \bB^{\natural \rmt}\bX_{\pi^{\natural}(i), :},
\]
respectively. We can rewrite $\term_{1}$ as
\[
\term_{1} =~&
(n-h)^{-1}\trace\Bracket{\bX\bracket{\bX_{j, :} - \bX_{\pi^{\natural}(i), :}}
\bX^{\rmt}_{\pi^{\natural}(i), :}
\bBtrue \bW^{\rmt}}
= (n-h)^{-1}\bu_{\bX}^{\rmt}\bW \bv_{\bX}.
\]
Invoking the union bound, we conclude
\begin{align}
\label{eq:term2_tot}
& \Prob\bracket{\term_{1} \gsim \sigma(\log n)^{\nfrac{5}{2}}\sqrt{\frac{p}{n}}\Fnorm{\bBtrue}}  \notag \\
\leq~&
\Prob\bracket{\term_{1} \gsim \sigma(\log n)^{\nfrac{5}{2}}\sqrt{\frac{p}{n}}\Fnorm{\bBtrue},~\norm{\bu_{\bX}}{2}\norm{\bv_{\bX}}{2} \lsim
(\log n)^{\nfrac{3}{2}}\sqrt{np} \Fnorm{\bBtrue}} \notag \\
+~& \Prob\bracket{\norm{\bu_{\bX}}{2}\norm{\bv_{\bX}}{2} \gsim
(\log n)^{\nfrac{3}{2}}\sqrt{np} \Fnorm{\bBtrue} }\notag  \\
\leq~&
\underbrace{\Prob\bracket{\term_{1} \gsim  \frac{\sigma (\log n) \norm{\bu_{\bX}}{2} \norm{\bv_{\bX}}{2}}{n-h} }}_{\defequal~\zeta_1}
+ \underbrace{\Prob\bracket{ \norm{\bu_{\bX}}{2}\norm{\bv_{\bX}}{2} \gsim
(\log n)^{\nfrac{3}{2}}\sqrt{np} \Fnorm{\bBtrue}}}_{\defequal~\zeta_2}.
\end{align}
Then we separately bound the probabilities $\zeta_1$ and $\zeta_2$.

\vspace{0.1in}\noindent
\textbf{Phase I.}
For probability $\zeta_1$, we exploit the independence between $\bX$ and $\bW$.
We can view $\term_{1}$ as a Gaussian RV conditional on $\bX$, since it is a linear combination of Gaussian RVs $\set{\bW_{i,j}}_{1\leq i \leq n, 1\leq j \leq m}$.
Easily we can calculate its mean as zero and its variance as
\[
\Expc_{\bW}(\term_{1})^2 = \frac{\sigma^2}{(n-h)^2} \norm{\bu_{\bX}}{2} \norm{\bv_{\bX}}{2}^2.
\]
Thus we can upper-bound $\zeta_1$ as
\begin{align}
\label{eq:term2_zeta1}
\zeta_1= \Expc_{\bX}\Expc_{\bW}\Ind\bracket{\term_{1} \gsim \frac{\sigma (\log n) \norm{\bu_{\bX}}{2} \norm{\bv_{\bX}}{2}}{n-h}} \stackrel{\cirone}{\leq} \Expc_{\bX}\exp\bracket{-c_0 \log n} = n^{-c},
\end{align}
where $\cirone$ is due to the bound on the tail-probability of Gaussian RV.

\vspace{0.1in}\noindent
\textbf{Phase II.}
As for $\zeta_2$, easily we can verify it to be zero conditional on the intersection
of events $\calF_1(\bBtrue)\bigcap \calF_5$ since
\[
\norm{\bu_{\bX}}{2}\norm{\bv_{\bX}}{2}\lsim
\sqrt{\log n}\Fnorm{\bBtrue}  \cdot
\bracket{\norm{\bX\bX_{j, :}}{2} + \norm{\bX\bX_{\pi^{\natural}(i), :}}{2}}
\lsim \bracket{\log n}^{\nfrac{3}{2}}\sqrt{np}\Fnorm{\bBtrue}.
\]
The proof is then completed by combining
\eqref{eq:term2_tot} and \eqref{eq:term2_zeta1}.
\end{proof}

\begin{lemma}
\label{lemma:term2_bound}
Conditional on the intersection of events
$\calF_6 \bigcap \calF_7$ and fixing the indices
$\pi^{\natural}(i)$ and $j$ $(j \neq \pi^{\natural}(i))$, we know
$\term_{2} \lsim \sigma (\log^2 n) \Fnorm{\bBtrue}$ hold
with probability at least $1 - n^{-c}$, where
$\term_2$ is defined in \eqref{eq:term3_def}.
\end{lemma}

\begin{proof}
Following a similar proof strategy as in Lemma~\ref{lemma:term2_bound},
we first invoke the union bound and obtain
\begin{align}
\label{eq:term3_tot}
& \Prob\bracket{\term_{2} \gsim \sigma(\log n)^2\Fnorm{\bBtrue}} \notag \\
\leq~& \Prob\bracket{\term_{2} \gsim \sigma(\log n)^2\Fnorm{\bBtrue},~
\norm{\wt{\bB}^{\rmt}\bracket{\bX_{j, :} - \bX_{\pi^{\natural}(i), :}}}{2} \lsim (\log n)\Fnorm{\bBtrue} } \notag \\
+~& \Prob\bracket{\norm{\wt{\bB}^{\rmt}\bracket{\bX_{j, :} - \bX_{\pi^{\natural}(i), :}}}{2} \gsim (\log n)\Fnorm{\bBtrue} } \notag \\
\leq~& \underbrace{\Prob\bracket{\term_{2} \gsim \sigma(\log n) \norm{\wt{\bB}^{\rmt}\bracket{\bX_{j, :} - \bX_{\pi^{\natural}(i), :}}}{2}}}_{\defequal~\zeta_1}
+ \underbrace{\Prob\bracket{\norm{\wt{\bB}^{\rmt}\bracket{\bX_{j, :} - \bX_{\pi^{\natural}(i), :}}}{2} \gsim (\log n)\Fnorm{\bBtrue}}}_{\defequal~\zeta_2}.
\end{align}
The following analysis separately investigates the two probabilities
$\zeta_1$ and $\zeta_2$.

\vspace{0.1in}\noindent
\textbf{Phase I.}
Exploiting the independence between $\bX$ and $\bW$, we can bound $\zeta_1$ as
\begin{align}
\label{eq:term3_zeta1}
\zeta_1 = \Expc_{\bX}\Expc_{\bW}\Ind\bracket{\term_{2} \gsim \sigma(\log n) \norm{\wt{\bB}^{\rmt}\bracket{\bX_{j, :} - \bX_{\pi^{\natural}(i), :}}}{2}} \stackrel{\cirone}{\leq}
\Expc_{\bX} \exp\bracket{-c_0 \log n} = n^{-c_0},
\end{align}
where in $\cirone$ we use the fact that $\term_{2}$ is a Gaussian RV
$\normdist(0, \|\wt{\bB}^{\rmt}(\bX_{j, :} - \bX_{\pi^{\natural}(i), :})\|_{2})$  conditional~on~$\bX$.

\vspace{0.1in}\noindent
\textbf{Phase II.}
Then we bound term $\zeta_2$. Note that
\[
\norm{\wt{\bB}^{\rmt}\bracket{\bX_{j, :} - \bX_{\pi^{\natural}(i), :}}}{2}
\leq~& \norm{(\wt{\bB}-  \wtminus{\bB}{\pi^{\natural}(i), j})^{\rmt}\bracket{\bX_{j, :} - \bX_{\pi^{\natural}(i), :}}}{2}
+ \norm{\wtminus{\bB}{\pi^{\natural}(i), j}^{\rmt}\bracket{\bX_{j, :} - \bX_{\pi^{\natural}(i), :}}}{2} \\
\leq~& \norm{(\wt{\bB}-  \wtminus{\bB}{\pi^{\natural}(i), j})^{\rmt} \bX_{j, :}}{2}
+ \norm{(\wt{\bB}-  \wtminus{\bB}{\pi^{\natural}(i), j})^{\rmt} \bX_{\pi^{\natural}(i), :}}{2} \\
+~& \norm{\wtminus{\bB}{\pi^{\natural}(i), j}^{\rmt}\bracket{\bX_{j, :} - \bX_{\pi^{\natural}(i), :}}}{2},
\]
we conclude
\begin{align}
\label{eq:term3_zeta2_tot}
\zeta_2 \stackrel{\cirtwo}{\leq}~&
\underbrace{\Prob\bracket{\norm{(\wt{\bB}-  \wtminus{\bB}{\pi^{\natural}(i), j})^{\rmt} \bX_{j, :}}{2}
+ \norm{(\wt{\bB}-  \wtminus{\bB}{\pi^{\natural}(i), j})^{\rmt} \bX_{\pi^{\natural}(i), :}}{2} \gsim \frac{p\log^{\nfrac{3}{2}} n}{n}\Fnorm{\bBtrue} }}_{\defequal~\zeta_{2, 1}} \notag \\
+~& \underbrace{\Prob\bracket{\norm{\wtminus{\bB}{\pi^{\natural}(i), j}^{\rmt}\bracket{\bX_{j, :} - \bX_{\pi^{\natural}(i), :}}}{2} \gsim (\log n)\Fnorm{\bBtrue}}}_{\defequal~\zeta_{2, 2}},
\end{align}
where in $\cirtwo$ we use the fact $n\gsim p\log^6 n$.
Recalling the definition of $\calF_7$ then yields $\zeta_{2, 1} = 0$.
For term $\zeta_{2, 2}$, we exploit the independence between
$\wtminus{\bB}{\pi^{\natural}(i), j}$ and $\bX_{j, :}$, $\bX_{\pi^{\natural}(i), :}$.
Via the Hanson-Wright inequality (Theorem $6.2.1$ in~\citet{vershynin2018high}),
we have
\begin{align}
\label{eq:term3_zeta22}
\zeta_{2,2} \leq \exp\Bracket{-c_0 \bracket{\frac{(\log n)^2\Fnorm{\bBtrue}^2 }{\Opnorm{\wtminus{\bB}{\pi^{\natural}(i), j}^{\rmt} \wtminus{\bB}{\pi^{\natural}(i), j}}} \vcap \frac{(\log n)^4 \Fnorm{\bBtrue}^4 }{\Fnorm{\wtminus{\bB}{\pi^{\natural}(i), j}^{\rmt} \wtminus{\bB}{\pi^{\natural}(i), j}}^2} } }
\stackrel{\cirthree}{\leq} n^{-c},
\end{align}
where $\cirthree$ is due to the fact
\[
\Fnorm{\wtminus{\bB}{\pi^{\natural}(i), j}}
\leq \Fnorm{\bBtrue}
+ \Fnorm{\wtminus{\bB}{\pi^{\natural}(i), j} - \bBtrue}
\stackrel{
\cirfour}{\lsim} \Fnorm{\bBtrue},
\]
and $\cirfour$ is conditional on event $\calF_6$ and $n\gsim p\log^6 n$.
Combining \eqref{eq:term3_tot}, \eqref{eq:term3_zeta1},
\eqref{eq:term3_zeta2_tot}, and \eqref{eq:term3_zeta22} then completes
the proof.
\end{proof}

\begin{lemma}
\label{lemma:term3_bound}
Conditional on event $\calF_2\bigcap \calF_5$ and fixing the
indices $\pi^{\natural}(i)$ and $j$, we have
$\term_{3} \lsim  \frac{mp(\log^2 n)\sigma^2}{n}
+ (\log^2 n)\sigma^2\sqrt{\frac{mp}{n}}$ hold
with probability exceeding $1 - c_0 n^{-c_1}$,
where $\term_3$ is defined in \eqref{eq:term4_def}, and
$c_0, c_1 > 0$ are some fixed positive constants.
\end{lemma}

\begin{proof}
For the benefits of presentation, we first define $\bXi^{\pi^{\natural}(i), j}$ as
\[
\bXi^{\pi^{\natural}(i), j} = \bX \bracket{\bX_{\pi^{\natural}(i), :} - \bX_{j, :}}.
\]
Then we can rewrite $\term_{3}$ as
$(n-h)^{-1}\bW_{i, :}^{\rmt} \bW^{\rmt} \bOmega^{\pi^{\natural}(i), j}$
and expand it as
\[
\abs{\term_{3}} =~&
\bracket{n-h}^{-1}\abs{\Xi^{\pi^{\natural}(i), j}_{i}\bW_{i, :}^{\rmt} \bW_{i, :}
+  \bW_{i, :}^{\rmt}\bigg(\sum_{k\neq i} \Xi^{\pi^{\natural}(i), j}_{k}
\bW_{k, :}\bigg)} \\
\leq~&
\frac{1}{n-h}
\abs{\Xi^{\pi^{\natural}(i), j}_{i}}\cdot \norm{\bW_{i, :}}{2}^2
+ \frac{1}{n-h}\abs{\la \bW_{i, :}, \sum_{k\neq i} \Xi^{\pi^{\natural}(i), j}_{k}
\bW_{k, :}\ra} \\
\stackrel{\cirone}{\leq}~&
\frac{p\log n}{n-h}\norm{\bW_{i, :}}{2}^2
+ \frac{1}{n-h}\abs{\la \bW_{i, :}, \sum_{k\neq i} \Xi^{\pi^{\natural}(i), j}_{k}
\bW_{k, :}\ra},
\]
where in $\cirone$ we condition on event $\calF_2$ and have
$\abs{\Xi^{\pi^{\natural}(i), j}_{i}} \leq \norm{\bX_{\pi^{\natural}(i), :}}{2}^2+ \norm{\bX_{j, :}}{2}^2 \lsim p\log n$.
With the union bound, we obtain
{\small \vspace{-2mm}
\begin{align}
\label{eq:term4_tot}
& \Prob\bracket{\term_{3} \gsim
 \frac{mp(\log^2 n)\sigma^2}{n}
+ \sigma^2(\log^2 n)\sqrt{\frac{mp}{n}}
} \notag \\
\stackrel{\cirtwo}{\leq}~&
\underbrace{\Prob\bracket{
\frac{p\log n}{n-h}\norm{\bW_{i, :}}{2}^2
\gsim \frac{mp(\log^2 n)\sigma^2}{n} }}_{\defequal \zeta_1} + \underbrace{\Prob\bracket{
\frac{1}{n-h}\abs{\la \bW_{i, :}, \sum_{k\neq i} \Omega^{\pi^{\natural}(i), j}_{k}
\bW_{k, :}\ra} \gsim \sigma^2(\log^2 n)\sqrt{\frac{mp}{n}} }}_{\defequal \zeta_2}.
\end{align}
}
\noindent Then we separately bound the
two terms $\zeta_1$ and $\zeta_2$.

\vspace{0.1in}\noindent
\textbf{Phase I.}
For term $\zeta_1$, we have
\begin{align}
\label{eq:term4_zeta1}
\zeta_1 \leq
\Prob\bracket{\norm{\bW_{i, :}}{2}^2 \gsim m(\log n)\sigma^2}
\stackrel{\cirthree}{=} e^{-c_0 \log n} = n^{-c_0},
\end{align}
where in $\cirthree$ we use the fact that $\norm{\bW_{i, :}}{2}^2/\sigma^2$
is a $\chi^2$-RV with freedom $m$ and invoke Lemma~\ref{lemma:chi_square}.

\vspace{0.1in}\noindent
\textbf{Phase II.}
Then we upper-bound $\zeta_2$ as
\begin{align}
\label{eq:term4_zeta2}
\zeta_{2}\leq~&
\underbrace{\Prob\bracket{
\frac{1}{n-h}\abs{\la \bW_{i, :}, \sum_{k\neq i} \Omega^{\pi^{\natural}(i), j}_{k}
\bW_{k, :}\ra} \gsim \frac{\sigma\sqrt{\log n}}{n} \bigg\|\sum_{k\neq i} \Xi^{\pi^{\natural}(i), j}_{k}
\bW_{k, :}\bigg\|_{2}}}_{\defequal~\zeta_{2, 1}} \notag \\
+~& \underbrace{\Prob\bracket{\bigg\|\sum_{k\neq i} \Xi^{\pi^{\natural}(i), j}_{k}
\bW_{k, :}\bigg\|_{2}^2\gsim mnp(\log n)^3 \sigma^2}}_{\defequal~\zeta_{2,2}}.
\end{align}
For term $\zeta_{2, 1}$, we exploit the independence across
the rows of the matrix $\bW$.
Conditional on $\set{\bW_{k, :}}_{k\neq i}$, we conclude
the inner-product $\la \bW_{i, :}, \sum_{k\neq i} \Xi^{\pi^{\natural}(i), j}_{k}
\bW_{k, :}\ra$ to be a Gaussian RV with zero mean and
$\big\|\sum_{k\neq i} \Xi^{\pi^{\natural}(i), j}_{k}
\bW_{k, :}\big\|_{2}^2$ variance, which yields
$\zeta_{2,1} \leq n^{-c}$.
For term $\zeta_{2, 2}$, we have
\begin{align}
\label{eq:term4_zeta22}
\zeta_{2, 2}
\leq~& \
\underbrace{\Prob\bracket{\bigg\|\sum_{k\neq i} \Xi^{\pi^{\natural}(i), j}_{k}
\bW_{k, :}\bigg\|_{2}^2\gsim m(\log n) \sigma^2\bigg[\sum_{k\neq i} (\Xi^{\pi^{\natural}(i), j}_{k})^2\bigg],~\sum_{k\neq i} (\Xi^{\pi^{\natural}(i), j}_{k})^2 \lsim (\log^2 n) np}}_{\defequal~ \zeta_{2, 2,1} } \notag \\
+ ~&\underbrace{\Prob\bracket{\sum_{k\neq i} (\Xi^{\pi^{\natural}(i), j}_{k})^2 \gsim
(\log^2 n) np}}_{\defequal~\zeta_{2,2,2}}.
\end{align}
Due to the independence across $\bX$ and $\bW$,
we can verify $\norm{\sum_{k\neq i} \Xi^{\pi^{\natural}(i), j}_{k}
\bW_{k, :}}{2}^2/[\sigma^2\sum_{k\neq i} (\Xi^{\pi^{\natural}(i), j}_{k})^2]$
to be a $\chi^2$-RV with freedom $m$ when conditional on $\bX$.
Invoking Lemma~\ref{lemma:chi_square}, we can upper-bound $\xi_{2,2,1}$ as
\begin{align}
\label{eq:term4_zeta221}
\zeta_{2, 2, 1} \leq
\Prob\bracket{\bigg\|\sum_{k\neq i} \Xi^{\pi^{\natural}(i), j}_{k}
\bW_{k, :}\bigg\|_{2}^2\gsim m(\log n) \sigma^2\bigg[\sum_{k\neq i} (\Xi^{\pi^{\natural}(i), j}_{k})^2\bigg]} \leq n^{-c}.
\end{align}
As for $\xi_{2,2,2 }$, we condition on event $\calF_5$ and have
\begin{align}
\label{eq:term4_zeta222}
\zeta_{2, 2, 2} \leq \Prob\bracket{\norm{\bX \bX_{\pi^{\natural}(i), :}}{2} + \norm{\bX \bX_{j, :}}{2} \gsim (\log n)\sqrt{np}} = 0.
\end{align}
Then the proof is completed by combining
~\eqref{eq:term4_tot}, ~\eqref{eq:term4_zeta1},
~\eqref{eq:term4_zeta2}, ~\eqref{eq:term4_zeta22},
~\eqref{eq:term4_zeta221}, and ~\eqref{eq:term4_zeta222}.
\end{proof}

\subsection{Medium regime: proof of Theorem~\ref{thm:multi_snr_require_general_sub_gauss}}

\begin{proof}
With the aforementioned analytical framework, we find the
constraint $\srank{\bBtrue} \gg \log^2 n$ to be inevitable.
In this subsection, we relax it to $\srank{\bBtrue} \gg \log n$ by
considering a different relaxation event.
First, we transform the solution of \eqref{eq:optim_estim_pi} to
that of the following optimization problem
\[
\bPiopt = \argmin_{\bPi}~\Fnorm{\bY - \bPi\bX \wh{\bB} }.
\]
This is due to the energy preserving property
of $\bPi$, i.e., $\fnorm{\bPi \bM} = \fnorm{\bM}$ for
an arbitrary matrix $\bM \in \RR^{n\times (\cdot)}$.
Then, we define the relaxation event $\calF_{\textup{err-relax}}^{(\textup{multi}, 2)}$ as
\[
\calF_{\textup{err-relax}}^{(\textup{multi}, 2)} \defequal \
\set{\big\|\bY_{i, :} - \wh{\bB}^{\rmt}\bX_{\pi^{\natural}(i), :} \big\|_{2}^2 \geq \big\|\bY_{i, :} - \wh{\bB}^{\rmt}\bX_{j, :} \big\|_{2}^2,~~\exists~1\leq j\neq \pi^{\natural}(i) \leq n}.
\]

\noindent
\textbf{Step I.}
We would like to show $\br{\calF}_{\textup{err-relax}}^{(\textup{multi}, 2)} \subseteq \{\bPiopt = \bPitrue\}$. Its proof comes as follows.
Conditional on $\br{\calF}_{\textup{err-relax}}^{(\textup{multi}, 2)}$, we have
\[
\Fnorm{\bY - \bPitrue\bX\wh{\bB} }^2 =
\sum_{i=1}^n \big \|\bY_{i, :} - \wh{\bB}^{\rmt}\bX_{\pi^{\natural}(i), :}\big \|_{2}^2
< \sum_{i=1}^n \big \|\bY_{i, :} -\wh{\bB}^{\rmt} \bX_{\pi(i), :}\big \|_{2}^2 = \
\Fnorm{\bY - \bPi\bX\wh{\bB} }^2,
\]
which implies $\bPiopt = \bPitrue$ and further leads to
$\{\bPiopt \neq \bPitrue\} \subseteq \calF_{\textup{err-relax}}^{(\textup{multi}, 2)}$.
Hence, we can upper-bound the error probability
by $\Prob(\calF_{\textup{err-relax}}^{(\textup{multi}, 2)})$.

\vspace{0.1in}\noindent
\textbf{Step II.}
We decompose the task of upper-bounding $\Prob(\calF_{\textup{err-relax}}^{(\textup{multi}, 2)})$
into a series of sub-tasks that are amenable to analysis.
First, we verify that $\big\|\bY_{i, :} - \wh{\bB}^{\rmt}\bX_{\pi^{\natural}(i), :}\big\|_{2}^2\geq
\big\|\bY_{i, :} - \wh{\bB}^{\rmt}\bX_{j, :} \big\|_{2}^2$
is equivalent to
\[
\underbrace{2\la \bW_{i, :},~\wh{\bB}^{\rmt}\bracket{\bX_{j, :} -\bX_{\pi^{\natural}(i), : }} \ra}_{\textup{LHS}} \
\geq ~&
\underbrace{\big\|\bB^{\natural \rmt}\bX_{\pi^{\natural}(i), :} - \wh{\bB}^{\rmt}\bX_{j, :} \big\|_{2}^2
-  \big\|(\bBtrue - \wh{\bB})^{\rmt} \bX_{\pi^{\natural}(i), : } \big\|_{2}^2}_{\textup{RHS}}.
\]
With the union bound, we conclude
\[
\Prob(\calF_{\textup{err-relax}}^{(\textup{multi}, 2)})
\leq~& \Prob\bracket{\textup{LHS} \gsim \Delta_{(\textup{multi, 2})},~\exists~1\leq \pi^{\natural}(i)\neq j \leq n}
+ \Prob\bracket{\textup{RHS} \lsim \Delta_{(\textup{multi, 2})},~\exists~1\leq \pi^{\natural}(i)\neq j \leq n} \\
\leq~& n^2\cdot \Bracket{
\Prob\bracket{\textup{LHS} \gsim \Delta_{(\textup{multi, 2})}}
+ \Prob\bracket{\textup{RHS} \lsim \Delta_{(\textup{multi, 2})}}
},
\]
where $\Delta_{(\textup{multi, 2})}$ is defined as
$\Delta_2 + \Delta_3$, whose definitions are referred
to \eqref{eq:Delta2_def} and \eqref{eq:Delta3_def}.

Then we complete the proof by separately bounding
$\Prob\bracket{\textup{LHS} \gsim \Delta_{(\textup{multi, 2})}}$ and
 $\Prob\bracket{\textup{RHS} \lsim \Delta_{(\textup{multi, 2})}}$.
To upper-bound $\Prob\bracket{\textup{LHS} \gsim \Delta_{(\textup{multi, 2})}}$,
we first perform the following decomposition
\[
\textup{LHS}
=2\langle \bW_{i, :},~\wt{\bB}^{\rmt}(\bX_{j, :} -\bX_{\pi^{\natural}(i), : })\rangle
+ 2(n-h)^{-1}\langle \bW \bW_{i, :},~\bX (\bX_{j, :} -\bX_{\pi^{\natural}(i), : } )\rangle.
\]
Recalling the definitions of $\term_2$ and $\term_3$ in
\eqref{eq:term3_def} and \eqref{eq:term4_def},
we invoke Lemma~\ref{lemma:term2_bound} and Lemma~\ref{lemma:term3_bound}
and obtain
$\Prob\bracket{\textup{LHS} \gsim \Delta_{(\textup{multi, 2})}} \leq c_0 n^{-c_1}$.
For $\Prob\bracket{\textup{RHS} \lsim \Delta_{(\textup{multi, 2})}}$,
we put its analysis in Lemma~\ref{lemma:rhs_general} and then complete the whole proof.
\end{proof}

\begin{lemma}
\label{lemma:rhs_general}
Consider the sensing matrix $\bX$ with its entries
$\bX_{i, j}$ being sub-gaussian RV with zero mean and unit variance $(1\leq i \leq n, 1\leq j \leq p)$.
Assume that $(i)$ $n\gg p \log^3 n\cdot \log^2(n^2 p^3)$,
$(ii)$ $\srank{\bBtrue} \gg \log n$, $(iii)$ $h\leq c\cdot  n$,
$(iv)$ conditional on $\calF_8$,
and $(v)$ $\log \snr \gsim \log \log n$,
we conclude
\[
\Prob\bracket{\big\|\bB^{\natural \rmt}\bX_{\pi^{\natural}(i), :} - \wh{\bB}^{\rmt}\bX_{j, :} \big\|_{2}^2
-  \big\|(\bBtrue - \wh{\bB})^{\rmt} \bX_{\pi^{\natural}(i), : } \big\|_{2}^2
\gsim \Delta_{(\textup{multi, 2})}} \geq 1 - c_0 n^{-c_1},
\]
where $\Delta_{(\textup{multi, 2})} = \Delta_2 + \Delta_3$, which are defined in
\eqref{eq:Delta2_def} and \eqref{eq:Delta3_def}, respectively.
\end{lemma}

\begin{proof}
We begin the proof as
\[
& \Prob\bracket{\big\|\bB^{\natural \rmt}\bX_{\pi^{\natural}(i), :} - \wh{\bB}^{\rmt}\bX_{j, :} \big\|_{2}^2
-  \big\|(\bBtrue - \wh{\bB})^{\rmt} \bX_{\pi^{\natural}(i), : }  \big\|_{2}^2 \leq \Delta_{(\textup{multi, 2})},~\exists~i, j} \\
\leq ~&
\Prob\big(\big\|\bB^{\natural \rmt}\big(\bX_{\pi^{\natural}(i), :} - \bX_{j, :}\big)  \big\|_{2}^2
- 2\big\| \bB^{\natural \rmt} \bracket{\bX_{\pi^{\natural}(i), :} - \bX_{j, :}}\big\|_2 \big\| (\bBtrue - \wh{\bB})^{\rmt}\bX_{j, :}\big\|_2 \\
& \qquad \qquad \qquad \qquad \qquad \qquad -  \big\| (\bBtrue - \wh{\bB})^{\rmt} \bX_{\pi^{\natural}(i), : } \big\|_{2}^2 \leq\Delta_{(\textup{multi, 2})},~\exists~i, j \big)
\\
\leq~& \underbrace{\Prob\big(\big\|\bB^{\natural \rmt} \bracket{\bX_{\pi^{\natural}(i), :} - \bX_{j, :}} \big\|_{2} \leq \delta,~\exists~i, j\big) }_{\defequal~\zeta_1} \\
+~&
\underbrace{
\Prob\bracket{\frac{\big\|(\bBtrue - \wh{\bB})^{\rmt} \bX_{\pi^{\natural}(i), :} \big\|_{2}^2}{\delta^2} + \frac{2\big\|(\bBtrue - \wh{\bB})^{\rmt} \bX_{j, :} \big\|_{2}}{\delta} + \frac{\Delta_{(\textup{multi, 2})}}{\delta^2}\geq 1,~\exists~i, j}
}_{\defequal~\zeta_2}. 	
\]	
The following context separately discusses $\zeta_1$ and $\zeta_2$.
Setting $\delta$ as $\nfrac{\Fnorm{\bBtrue}}{4}$,
we will show $\zeta_1 \lsim n^{-c}$ and
$\zeta_2 \lsim n^{-c}$ under the assumptions
in Lemma~\ref{lemma:rhs_general} and $\calF_8$.

\vspace{0.1in}\noindent
\textbf{Analysis of $\zeta_1$.}
We set $\delta$ as $\Fnorm{\bBtrue}/4$
and can upper bound $\zeta_1$ as
\begin{align}
\label{eq:event9_eta_1_general}
\zeta_1 \leq~& \sum_{i=1}^n \sum_{j\neq \pi^{\natural}(i)} \
\Prob\bracket{\big\|\bB^{\natural \rmt} (\bX_{\pi^{\natural}(i), :} - \bX_{j, :}) \big\|_{2} \leq \delta}
\stackrel{\cirone}{\leq} \
\sum_{i=1}^n \sum_{j\neq \pi^{\natural}(i)}
e^{-c_0 \cdot \srank{\bBtrue}} \notag \\
\leq~&
n^2 \cdot e^{-c_0 \cdot \srank{\bBtrue} },
\end{align}
where $\cirone$ comes from the \emph{small ball} probability as restated in Lemma~\ref{lemma:small_ball_subgauss}.
Due to the assumption such that $\srank{\bBtrue} \gg \log n$,
we have $\zeta_1 \lsim n^{-c}$.
When $\bX_{ij}$ is with log-concave property,
we can pick a smaller $\delta$ and remove this assumption.
A detailed explanation is deferred to
Lemma~\ref{lemma:event9_logconcave}.

\vspace{0.1in}\noindent
\textbf{Analysis of $\zeta_2$.}
We prove that $\zeta_2$ to be less than $c_0\cdot n^{-c_1}$ under the assumptions in
Lemma~\ref{lemma:rhs_general}.
Conditional on $\calF_8$ and invoke Lemma~\ref{lemma:xw_ub} (attached as follows),
with probability $1-c_0 \cdot n^{-c_1}$ we have
\begin{align}
\frac{\big\|(\bBtrue - \wh{\bB}  )^{\rmt} \bX_{\pi^{\natural}(i), :} \big\|^2_{2}}{\delta^2}
\leq~& \frac{2\big\|(\wt{\bB} - \bBtrue )^{\rmt} \bX_{\pi^{\natural}(i), :} \big\|^2_{2}}{\delta^2}
+ \frac{2\big\|\bW^{\rmt}\bX \bX_{i, :}\big\|^2_2}{(n-h)^2 \delta^2} \notag \\
\lsim~&
\frac{p \log^3 n\cdot \log^2(n^2 p^3)}{n}
+ \frac{\log^2 n}{\snr}.
\label{eq:event9_eta12_general}
\end{align}
Recalling the assumptions $\snr \gsim \log^c n$ and $n \gg p\cdot (\log^3 n)\cdot \log^2(n^2p^3)$,
we have $\nfrac{\|(\wh{\bB} - \bBtrue )^{\rmt} \bX_{\pi^{\natural}(i), :}\|_{2}}{\delta}$
approach zero as $n$ goes to infinity.
Following the same logic, we conclude
$\nfrac{\|(\bBtrue - \wh{\bB})^{\rmt} \bX_{j, :} \|_{2}}{\delta} \rightarrow 0$
as $n\rightarrow \infty$.
Then we turn to $\nfrac{\Delta_{(\textup{multi, 2})}}{\delta^2}$ and obtain
\begin{align}
\frac{\Delta_{(\textup{multi, 2})}}{\delta^2} \lsim~& \
\frac{\log^2 n}{\sqrt{m \cdot \snr}}
+ \frac{ \log^2 n}{\snr}
\sqrt{\frac{p}{nm}} +
\frac{p\cdot \log^2 n }{n\cdot \snr}.
\label{eq:event9_delta_ratio_general}
\end{align}
Following similar procedures as above, we can prove
$\nfrac{\Delta_{(\textup{multi, 2})}}{\delta^2}$ to be a small positive constant.
Combing \eqref{eq:event9_eta12_general} and \eqref{eq:event9_delta_ratio_general} together,
we conclude
\[
\frac{\big\|(\bBtrue - \wh{\bB})^{\rmt} \bX_{\pi^{\natural}(i), :} \big\|_{2}^2}{\delta^2} + \frac{2\big\|(\bBtrue - \wh{\bB})^{\rmt} \bX_{j, :} \big\|_{2}}{\delta} + \frac{\Delta_{(\textup{multi, 2})}}{\delta^2} < 1,
\]
conditional on $\calF_8$ and $\norm{\bW^{\rmt}\bX \bX_{i, :} }{2} \lsim \
(n\log n)\sqrt{m \sigma^2}$.
Hence, we obtain
\begin{align}
\label{eq:medium_xw_relation}
\zeta_2 \leq \Prob\bracket{\norm{\bW^{\rmt}\bX \bX_{i, :} }{2} \gsim \
(n\log n)\sqrt{m \sigma^2},~\exists~1\leq i \leq n }
\stackrel{\cirtwo}{\leq} c_0 \cdot n^{-c_1}
\end{align}
where $\cirtwo$ is due to
Lemma~\ref{lemma:xw_ub}.
Combining with \eqref{eq:event9_eta_1_general} then
completes the proof.
\end{proof}

\begin{lemma}
\label{lemma:xw_ub}
Conditioning on $\calF_2$,
we have
\[
\Prob\bracket{\norm{\bW^{\rmt}\bX \bX_{i, :} }{2} \leq \
c (n\log n)\sqrt{m \sigma^2},~\exists~1\leq i \leq n } \geq
1 - c_0 n^{-c_1}.
\]
\end{lemma}

\begin{proof}
First, we consider a fixed index $i$.
Adopting the \emph{leave-one-out}
technique, we construct a perturbed matrix $\sampminus{\bX}{i}$
by replicating matrix $\bX$ except its $i$th row $\bX_{i, :}$, which
is replaced with an i.i.d.
sample $\bX^{'}_{i, :}$.
Then we obtain
\[
& \Prob\bracket{\norm{\bX_{i, :}\bX^{\rmt}\bW}{2} \geq c (n\log n)\sqrt{m\sigma^2}} \\
\leq~& \Prob\bracket{\big\|\bX_{i, :}\sampminus{\bX}{i}^{\rmt}\bW\big\|_{2} +  \big\|\bX_{i,:}(\bX -\sampminus{\bX}{i})^{\rmt}\bW \big\|_{2}\geq c (n\log n)\sqrt{m\sigma^2} } \\
\stackrel{\cirone}{\leq}~& \underbrace{\Prob\bracket{\big\|\bX_{i,:}\big(\bX -\sampminus{\bX}{i}\big)^{\rmt}\bW\big\|_{2} \geq  c_0 (n\log n)\sqrt{m\sigma^2}}}_{\defequal~\zeta_1}+
\underbrace{\Prob\bracket{\big\|\bX_{i, :}\sampminus{\bX}{i}^{\rmt}\bW \big\|_2 \geq c_1 (n\log n)\sqrt{m\sigma^2} }}_{\defequal~\zeta_2},
\]
where $\cirone$ is due to the union bound.

\vspace{0.1in}\noindent
\textbf{Analysis of $\zeta_1$}.
One noticeable property of $\sampminus{\bX}{i}- \bX$ is that only its $i$th row
is non-zero. Hence we obtain the relation
\[
\big\|\bX_{i,:}(\bX - \sampminus{\bX}{i})^{\rmt}\bW \big\|_{2} \
\leq \norm{\bX_{i, :}}{2} \cdot \fnorm{(\bX - \sampminus{\bX}{i})^{\rmt}\bW}
= \norm{\bX_{i, :}}{2}\cdot \big\|\bX_{i, :} - \bX^{'}_{i, :}\big\|_{2} \norm{\bW_{i, :}}{2}.
\]
Recalling the definition of $\calF_2$, $n\gg p \log^6 n$, and $\norm{\bW_{i, :}}{2} \leq 2\sqrt{m\sigma^2}$ holds
with probability exceeding $1-c_0 \cdot n^{-c_1}$, we have
 $\zeta_1 \lsim n^{-c_1}$ conditioning on $\calF_2$.

\vspace{0.1in}\noindent
\textbf{Analysis of $\zeta_2$.}
Due to the construction of
$\sampminus{\bX}{i}$, we have $\bX_{i, :}$ to be independent of $\sampminus{\bX}{i}$.
Hence, we condition on $\sampminus{\bX}{i}^{\rmt}\bW$ and obtain
\[
\zeta_2 \leq~& \Prob\bracket{\big\|\bX_{i, :}\sampminus{\bX}{i}^{\rmt}\bW\big\|_{2} \geq
c_1 n(\log n)\sqrt{m\sigma^2},~ \fnorm{\sampminus{\bX}{i}^{\rmt}\bW} \lsim n\sqrt{m\sigma^2}}
+ \Prob\bracket{\fnorm{\sampminus{\bX}{i}^{\rmt}\bW} \gsim n\sqrt{m\sigma^2}} \\
\leq ~&\
\underbrace{\Expc_{\sampminus{\bX}{i}^{\rmt}\bW}
\Ind\bracket{\big\|\bX_{i, :}\sampminus{\bX}{i}^{\rmt}\bW\big\|_{2} \geq
c_2 (\log n)\fnorm{\sampminus{\bX}{i}^{\rmt}\bW}}}_{\defequal~\zeta_{2, 1}} + \
\underbrace{\Prob\bracket{\fnorm{\sampminus{\bX}{i}^{\rmt}\bW} \gsim n\sqrt{m\sigma^2}} }_{\defequal~\zeta_{2, 2}}.
\]
For $\zeta_{2, 1}$, we define $Z = \big\|\bX_{i, :}\sampminus{\bX}{i}^{\rmt}\bW\big\|_{2}^2$
and have
\begin{align*}
\zeta_{2, 1} \stackrel{\cirtwo}{\leq}~& \Expc_{\sampminus{\bX}{i}^{\rmt}\bW}
\Ind\bracket{\abs{Z - \Expc Z} \geq c_3 (\log^2 n)\fnorm{\sampminus{\bX}{i}^{\rmt}\bW}^2} \\
\stackrel{\cirthree}{\leq}~& \Expc_{\sampminus{\bX}{i}^{\rmt}\bW}
\exp\bracket{- \bracket{\frac{(\log^4 n) \fnorm{\sampminus{\bX}{i}^{\rmt}\bW}^4}{\fnorm{\sampminus{\bX}{i}^{\rmt}\bW \bW^{\rmt} \sampminus{\bX}{i}}^2} \vcap \
\frac{\bracket{\log^2 n}\fnorm{\sampminus{\bX}{i}^{\rmt}\bW}^2}{\opnorm{
\sampminus{\bX}{i}^{\rmt}\bW \bW^{\rmt} \sampminus{\bX}{i}}}}} \
\stackrel{\cirfour}{\leq} n^{-c},
\end{align*}
where $\cirtwo$ is due to $\Expc_{\bX_{i, :}}Z = \fnorm{\sampminus{\bX}{i}^{\rmt}\bW}^2$,
$\cirthree$ is because of the Hanson-Wright inequality
(Theorem $6.2.1$ in~\citet{vershynin2018high}),
and $\cirfour$ is due to the stable rank $\srank{\sampminus{\bX}{i}^{\rmt}\bW} \geq 1$.
Then we move to $\zeta_{2, 2}$ and have
\begin{align*}
& \Prob\bracket{\fnorm{\sampminus{\bX}{i}^{\rmt}\bW} \gsim n\sqrt{m\sigma^2}} \leq \
\Prob\bracket{\opnorm{\sampminus{\bX}{i}}\Fnorm{\bW} \gsim n\sqrt{m\sigma^2} }  \\
\stackrel{\cirfive}{\leq}~&  \Prob\bracket{\opnorm{\sampminus{\bX}{i}} \gsim \sqrt{n} + \sqrt{p}} + \
\Prob\bracket{\Fnorm{\bW} \gsim \frac{n\sqrt{m\sigma^2}}{\sqrt{n} + \sqrt{p}}, ~\
\opnorm{\sampminus{\bX}{i}} \lsim \sqrt{n} + \sqrt{p}}\\
\stackrel{\cirsix}{\leq} ~& \Prob\bracket{\opnorm{\sampminus{\bX}{i}} \gsim \sqrt{n} + \sqrt{p}} + \
\Prob\bracket{\Fnorm{\bW} \geq \sqrt{2nm} \sigma} \
\stackrel{\cirseven}{\leq} e^{-c_0n} +e^{-c_1 nm},
\end{align*}
where $\cirfive$ is because of the union bound,
in $\cirsix$ we use $p\leq n$, and in $\cirseven$ we use
$\opnorm{\bX} \gsim \sqrt{n} + \sqrt{p}$ with
probability less than $e^{-c_0 n}$ (c.f. Theorem
$4.6.1$ in \citet{vershynin2018high})
 and the fact
$\Fnorm{\bW}^2/\sigma^2$ is a
$\chi^2$-RV with $nm$ freedom, and
Lemma~\ref{lemma:chi_square}.

Ultimately, we complete the proof by iterating the above proof procedures to all indices and
invoking the union bound.
\end{proof}

\subsection{Hard regime: proof of Theorem~\ref{thm:multi_snr_require_log_concave}}
This subsection aims to strength Theorem~\ref{thm:multi_snr_require_general_sub_gauss}, which reduces the
requirement on $\srank{\bBtrue}$ from $\Omega(\log n)$ to
$\Omega(1)$. As compensation, we need to put an extra assumption on $\bX_{ij}$'s distribution,
namely, $\bX_{ij}$ follows a log-concave sub-gaussian distribution.

\begin{lemma}
\label{lemma:event9_logconcave}
Consider the sensing matrix $\bX$ with its entries
$\bX_{i, j}$ being log-concave sub-gaussian RV with zero mean and unit variance
$(1\leq i \leq n, 1\leq j \leq p)$.
Assume that $(i)$ $n\gg p^{1 + \varepsilon} \log^{3(1 +\varepsilon)} n\cdot \log^{2(1+\varepsilon)}(n^2 p^3)$,
$(ii)$ $\srank{\bBtrue} \geq c(\varepsilon)$,
$(iii)$ $h\leq c\cdot n$, $(iv)$ conditional on $\calF_8$,
and
\begin{align}
\label{eq:event9_logconcave_snr_require}
\textup(v)~~\log \snr \gsim \frac{\log n}{\srank{\bBtrue}} + \log \log n,
\end{align}
we conclude
\[
\Prob\bracket{\big\|\bB^{\natural \rmt}\bX_{\pi^{\natural}(i), :} - \wh{\bB}^{\rmt}\bX_{j, :} \big\|_{2}^2
-  \big\|(\bBtrue - \wh{\bB})^{\rmt} \bX_{\pi^{\natural}(i), : } \big\|_{2}^2
\gsim \Delta_{(\textup{multi, 2})}} \geq 1 - c_0 n^{-c_1},
\]
where $\varepsilon > 0$ is an arbitrary positive constant;
$c(\varepsilon)$ is some positive constant depending only on $\varepsilon$;
and $\Delta_{(\textup{multi, 2})} = \Delta_2 + \Delta_3$, whose definitions can be found in
\eqref{eq:Delta2_def} and \eqref{eq:Delta3_def}, respectively.
\end{lemma}

\begin{proof}
Same as the procedure in
Lemma~\ref{lemma:rhs_general},
we begin the proof with the union bound
\[
& \Prob\bracket{\big\|\bB^{\natural \rmt}\bX_{\pi^{\natural}(i), :} - \wh{\bB}^{\rmt}\bX_{j, :} \big\|_{2}^2
-  \big\|(\bBtrue - \wh{\bB})^{\rmt} \bX_{\pi^{\natural}(i), : }  \big\|_{2}^2 \leq \Delta_{(\textup{multi, 2})},~\exists~i, j} \\
\leq ~&
\Prob\big(\big\|\bB^{\natural \rmt}\big(\bX_{\pi^{\natural}(i), :} - \bX_{j, :}\big)  \big\|_{2}^2
- 2\big\| \bB^{\natural \rmt} \bracket{\bX_{\pi^{\natural}(i), :} - \bX_{j, :}}\big\|_2 \big\| (\bBtrue - \wh{\bB})^{\rmt}\bX_{j, :}\big\|_2 \\
& \qquad \qquad \qquad \qquad \qquad \qquad -  \big\| (\bBtrue - \wh{\bB})^{\rmt} \bX_{\pi^{\natural}(i), : } \big\|_{2}^2 \leq \Delta_{(\textup{multi, 2})},~\exists~i, j \big)
\\
\leq~& \underbrace{\Prob\big(\big\|\bB^{\natural \rmt} \bracket{\bX_{\pi^{\natural}(i), :} - \bX_{j, :}} \big\|_{2} \leq \delta,~\exists~i, j\big) }_{\defequal~\zeta_1} \\
+~&
\underbrace{
\Prob\bracket{\frac{\big\|(\bBtrue - \wh{\bB})^{\rmt} \bX_{\pi^{\natural}(i), :} \big\|_{2}^2}{\delta^2} + \frac{2\big\|(\bBtrue - \wh{\bB})^{\rmt} \bX_{j, :} \big\|_{2}}{\delta} + \frac{\Delta_{(\textup{multi, 2})}}{\delta^2}\geq 1,~\exists~i, j}
}_{\defequal~\zeta_2}. 	
\]
Due to the log-concave assumption on $\bX_{ij}$'s distributions,
we are able to pick a smaller $\delta$.
Here we set $\delta$ as $\fnorm{\bBtrue}\cdot n^{-\nfrac{c_0}{\srank{\bBtrue}}}$.
We would like to show $\zeta_1 \leq c_0 n^{-c_1}$ and
$\zeta_2 \leq c_0 n^{-c_1}$ under the assumptions
in Lemma~\ref{lemma:event9_logconcave}.

\vspace{0.1in}\noindent
\textbf{Analysis of $\zeta_1$.}
According to the \emph{small ball} probability in Theorem $1.3$
\citep{paouris2012small}, which is also stated as Lemma~\ref{lemma:small_ball_log_concave},
we have
\begin{align}
\label{eq:event9_zeta_1_logconcave}
\zeta_1 \leq \sum_{i=1}^n \sum_{j\neq \pi^{\natural}(i)} \
\Prob\bracket{\norm{\bracket{\bX_{\pi^{\natural}(i), :} - \bX_{j, :}} \bBtrue}{2} \leq \delta} \stackrel{\cirone}{\leq} \
\sum_{i=1}^n \sum_{j\neq \pi^{\natural}(i)} n^{-c_0} \lsim n^{-c_1}.
\end{align}
Note that the requirement $\srank{\bBtrue} \gg \log n$ is no longer needed here.

\vspace{0.1in}\noindent
\textbf{Analysis of $\zeta_2$.}
Then we move on to the analysis of $\zeta_2$. Same as Lemma~\ref{lemma:event9_logconcave}, we
have
\begin{align}
\frac{\big\|(\bBtrue - \wh{\bB}  )^{\rmt} \bX_{\pi^{\natural}(i), :} \big\|^2_{2}}{\delta^2}
\leq~& \frac{2\big\|(\wt{\bB} - \bBtrue )^{\rmt} \bX_{\pi^{\natural}(i), :} \big\|^2_{2}}{\delta^2}
+ \frac{2\big\|\bW^{\rmt}\bX \bX_{i, :}\big\|^2_2}{(n-h)^2 \delta^2} \notag \\
\lsim~&
\frac{p \log^3 n\cdot \log^2(n^2 p^3)}{n^{1- \nfrac{c}{\srank{\bBtrue}}}}
+ \frac{\log^2 n \cdot n^{\nfrac{c}{\srank{\bBtrue}}}}{\snr}
\label{eq:event9_eta12_logconcave}
\end{align}
hold with probability $1-c_0 n^{-c_1}$ when conditional on
$\calF_8$. When $\srank{\bBtrue} \geq \nfrac{c(1+\varepsilon)}{\varepsilon}$ and
$\log\snr \gsim \frac{\log n}{\srank{\bBtrue}} + \log \log n$,
we have $\nfrac{\big\|(\bBtrue - \wh{\bB}  )^{\rmt} \bX_{\pi^{\natural}(i), :} \big\|^2_{2}}{\delta^2}$
approach to zero when $n$ goes to infinity.
Following the same logic, we have
$\nfrac{\|(\bBtrue - \wh{\bB})^{\rmt} \bX_{j, :} \|_{2}}{\delta}$
be arbitrarily small provided $n$ is sufficiently large.
Afterwards, we consider $\frac{\Delta_{(\textup{multi, 2})}}{\delta^2}$ and expand it as
\begin{align}
& \frac{\Delta_{(\textup{multi, 2})}}{\delta^2} \lsim \
n^{\nfrac{c}{\srank{\bBtrue}}}\bracket{\frac{\log^2 n}{\sqrt{m \cdot \snr}}
+ \frac{ \log^2 n}{\snr}
\sqrt{\frac{p}{nm}} +
\frac{p\cdot \log^2 n }{n\cdot \snr}}.
\label{eq:event9_delta_ratio_logconcave}
\end{align}
Following similar procedures as above, we can prove
$\frac{\Delta_{(\textup{multi, 2})}}{\delta^2}$ to be a small positive constant
given \eqref{eq:event9_logconcave_snr_require}, which
enables us to bound $\zeta_2$ in the same way as \eqref{eq:medium_xw_relation}.
Combing \eqref{eq:event9_zeta_1_logconcave}, \eqref{eq:event9_eta12_logconcave}, and \eqref{eq:event9_delta_ratio_logconcave} hence completes the proof.
\end{proof}

\subsection{Supporting lemmas}
\label{subsec:appendix_proof_multi_observe_support_lemma}

\begin{lemma}
\label{lemma:x_row_norm_ub}
For an arbitrary row $\bX_{i, :}$, we have
\[
\norm{\bM^{\rmt}\bX_{i, :}}{2} \lsim \sqrt{\log n} \Fnorm{\bBtrue},
\]
with probability exceeding $1 - n^{-c}$.
\end{lemma}

\begin{proof}
This lemma is a direct consequence of the
Hanson-Wright inequality (Theorem $6.2.1$ in~\citet{vershynin2018high}).
Easily we can verify
$\Expc \norm{\bM^{\rmt}\bX_{i,:}}{2}^2 = \fnorm{\bM}^2$ and
hence
\[
\Prob\bracket{ \norm{\bM^{ \rmt}\bX_{i,:}}{2}^2 \gsim \log n \Fnorm{\bM}^2}
{\leq}~&
\Prob\bracket{ \abs{\norm{\bM^{\rmt}\bX_{i,:}}{2}^2 - \fnorm{\bM}^2} \gsim (\log n) \fnorm{\bM}^2} \\
\leq~& \exp\bracket{-c_0 \bracket{\frac{\log n \fnorm{\bM}^2}{\Opnorm{\bM^{\rmt}\bM }} \vcap
\frac{(\log^2 n) \fnorm{\bM}^4}{\Fnorm{\bM^{\rmt}\bM}^2}}} \leq n^{-1-c}.
\]
Adopting the union bound, we have
\[
\Prob\bracket{ \norm{\bM^{\rmt}\bX_{i,:}}{2}^2 \gsim \log n \Fnorm{\bM}^2,~\forall~i}
\leq n\cdot n^{-1-c} = n^{-c}.
\]
Following the same procedure, we can obtain the same conclusions for
$\big \|\bM^{\rmt}\bX_{i, :}^{'}\big \|_{2}$.
\end{proof}

\begin{lemma}
\label{lemma:inner_product}	
For an arbitrary row $\bX_{i, :}$ (or $\bX_{i, :}^{'}$), we have
\[
\langle \bX_{i, :}, \bX_{j, :}^{'}\rangle  ~&\lsim (\log n)\sqrt{p} , ~~1\leq i, j \leq n; \\
\langle \bX_{i, :}, \bX_{j, :}\rangle ~&\lsim  (\log n)\sqrt{p},~~1\leq i\neq j \leq n; \\
\langle \bX^{'}_{i, :}, \bX^{'}_{j, :}\rangle ~&\lsim  (\log n)\sqrt{p},~~1\leq i\neq j \leq n,
\]
hold with probability $1-n^{-c}$.
\end{lemma}

\begin{proof}
Denote $\bz_1, \bz_2 \in \RR^p$ are  two i.i.d isotropic sub-gaussian RVs with $\norm{\bz_i}{\psi_2}\leq K$ $(i= 1,2)$,
where definition of $\norm{\cdot}{\psi_2}$ is referred to~\cite{vershynin2018high} (Definition~$3.2.1$).
Exploiting the independence between $\bz_1$ and $\bz_2$, we conclude
\[
\Prob\bracket{\la \bz_1, \bz_2 \ra\gsim \sqrt{p\log n}}
\stackrel{\cirone}{\leq}~&
\Prob\bracket{\la \bz_1, \bz_2 \ra\gsim (\log n)\sqrt{p}, \norm{\bz_2}{2}\leq \sqrt{p\log n}}
+ \Prob\bracket{\norm{\bz_2}{2} \gsim \sqrt{p\log n}} \\
\stackrel{\cirtwo}{\leq} ~& 2\exp\bracket{-\frac{c_0 \cdot p\log^2 n}{2 p \log n}}
+ \Prob\bracket{\abs{\norm{\bz_2}{2}^2 - p }\geq p\log n}
\stackrel{\cirthree}{\leq} n^{-c},
\]
where $\cirone$ and $\cirtwo$ are due to the union bound,
and $\cirthree$ is because of Hanson-Wright inequality.
Following the same procedure, we can prove the claims
in Lemma~\ref{lemma:inner_product} by invoking the union bound.
\end{proof}

\begin{lemma}
\label{lemma:Xmat_fnorm}
We conclude $\Prob\bracket{\calF_4} \geq 1-n e^{-cnp}$.
\end{lemma}
This lemma is a direct consequence of Lemma~\ref{lemma:chi_square} and hence its proof is omitted.

\begin{lemma}
\label{lemma:event5}
Conditional on the intersection of
events $\calF_3 \bigcap \calF_4$, we have
$\Prob\bracket{\calF_5}\geq 1 - n^{-c}$.
\end{lemma}

\begin{proof}
For a fixed row index $s$ ($1\leq s\leq n$),
we have
\[
& \Prob\bracket{\norm{\bX\bX_{s, :}}{2} \gsim (\log n)\sqrt{np}} \\
\stackrel{\cirone}{\leq}~& \Prob\bracket{\norm{\bracket{\bX - \sampminus{\bX}{s}} \bX_{s, :} }{2} \gsim
p\log n}
+ \Prob\bracket{\norm{\sampminus{\bX}{s} \bX_{s, :}}{2} \gsim (\log n)\sqrt{np}} \\
\stackrel{\cirtwo}{\leq}~&
\underbrace{\Prob\bracket{\bracket{\norm{\bX_{s, :}}{2}
+ \|\bX_{s, :}^{'}\|_{2}}\norm{\bX_{s, :}}{2} \gsim p\log n}}_{\defequal \zeta_1}
+ \underbrace{\Prob\bracket{\norm{\sampminus{\bX}{s} \bX_{s, :}}{2} \gsim (\log n)\sqrt{np}}}_{\defequal \zeta_2 },
\]
where in $\cirone$ we use the union bound and
the fact $n\geq p$; and in
$\cirtwo$ we use the definition of $\sampminus{\bX}{s}$ such that
the difference $\bX - \sampminus{\bX}{s}$ only have non-zero elements in the $s$th column.
Conditional on the intersection of
events $\calF_2 \bigcap \calF_3 \bigcap \calF_4$, we conclude that
probability $\zeta_1$ is zero and
probability $\zeta_2$ is upper-bounded as
\[
\Prob\bracket{\norm{\sampminus{\bX}{s} \bX_{s, :}}{2} \gsim (\log n)\sqrt{np}}
\leq~&\Prob\bracket{\abs{\norm{\sampminus{\bX}{s} \bX_{s, :} }{2}^2 - \Fnorm{\sampminus{\bX}{s}}^2} \gsim (\log^2 n)np } \\
\leq~& \exp\bracket{-c_0 \bracket{\frac{(\log^2 n)np }{\opnorm{\sampminus{\bX}{s}^{\rmt}\sampminus{\bX}{s}}} \vcap \frac{(\log n)^4 n^2p^2}{ \fnorm{\sampminus{\bX}{s}^{\rmt}\sampminus{\bX}{s}}^2}}} \leq n^{-c}.
\]
Thus the proof is completed by invoking the union bound since
\[
\Prob\bracket{\norm{\bX\bX_{s, :}}{2} \gsim (\log n)\sqrt{np},~\exists~s}
\leq n \cdot \Prob\bracket{\norm{\bX\bX_{s, :}}{2} \gsim (\log n)\sqrt{np}} \
\leq n\bracket{\zeta_1 + \zeta_2} \leq n^{1-c} = n^{-c^{'}}.
\]	
\end{proof}

\begin{lemma}
\label{lemma:Bmat_perturb}
Conditional on $\calF_4$,
we have $\Prob(\calF_6) \geq 1- c p^{-2}$.	
\end{lemma}
\begin{proof}
We assume that
the first $h$ rows of $\bX$ are permuted w.l.o.g.
Due to the i.i.d. distribution of
$\{\bX_{i, :}\}_{i=1}^n$ and $\{\bX^{'}_{i, :}\}_{i=1}^n$, we
conclude
\begin{align}
\label{eq:event6_tot}
\Prob(\calF_6) \leq
n^2 \cdot \Prob\bracket{\big\|\bBtrue - \wt{\bB}\big\|_{2}\gsim \
\frac{(\log n)(\log n^2 p^3)\sqrt{p}}{\sqrt{n}}\Fnorm{\bBtrue}}.
\end{align}
First, we expand $\bX^{\rmt}\bPitrue\bX$ as
\[
\bX^{\rmt}\bPitrue\bX =
\sum_{i=1}^h \bX_{\pi^{\natural}(i), :}\bX_{i, :}^{\rmt} + \
\sum_{i=h+1}^n \bX_{i, :}\bX_{i, :}^{\rmt},
\]
and obtain
\begin{align}
\label{eq:event6_tot_bernstein_tot}	
& \Prob\bracket{\big\|\bBtrue - \wt{\bB}\big\|_{2}\gsim \
\frac{(\log n)(\log n^2 p^3)\sqrt{p}}{\sqrt{n}}\Fnorm{\bBtrue} }\notag \\
\leq~& \
\Prob\bracket{\frac{1}{n-h}\Fnorm{\sum_{i=1}^h \bX_{\pi^{\natural}(i), :}\bX_{i, :}^{\rmt} \bBtrue } + \frac{1}{n-h}\Fnorm{\sum_{i=h+1}^n \bracket{\bX_{i, :}\bX^{\rmt} _{i, :} -\bI}\bBtrue}\gsim \frac{(\log n)(\log n^2 p^3)\sqrt{p}}{\sqrt{n}}\Fnorm{\bBtrue}  } \notag \\
\stackrel{\cirone}{\leq}~&\
\underbrace{\Prob\bracket{\frac{1}{n-h}\Fnorm{\sum_{i=1}^h \bX_{\pi^{\natural}(i), :}\bX^{\rmt}_{i, :} \bBtrue } \gsim    \frac{(\log n)(\log n^2 p^3)\sqrt{p}}{\sqrt{n}}\Fnorm{\bBtrue} } }_{\defequal~\zeta_1} \notag \\
+~& \
\underbrace{\Prob\bracket{\frac{1}{n-h}\Fnorm{\sum_{i=h+1}^n \bracket{\bX_{i, :} \bX_{i, :}^{\rmt} -\bI}\bBtrue}\gsim
\frac{(\log n)(\log n^2 p^3)\sqrt{p}}{\sqrt{n}}\Fnorm{\bBtrue}
 }}_{\defequal~\zeta_2},
\end{align}
where $\cirone$ is because of the union bound.
The proof is complete by
proving $\zeta_1 \leq 6n^{-2}p^{-2}$ and $\zeta_2 \leq 4n^{-2}p^{-2}$.
The technical details come as follows.

\vsp \noindent
\textbf{Analysis of $\zeta_1$.}
According to Lemma~$8$ in~\citet{pananjady2018linear} (restated as Lemma~\ref{lemma:permute_decomp}),
we can decompose the set $\set{i: \pi^{\natural}(i)\neq i}$
into three disjoint sets $\calI_{\ell}$ $(1\leq \ell \leq 3)$,
such that $i$ and $\pi^{\natural}(i)$ does not reside within the same set.
And the cardinality $h_{\ell}$ of set $\calI_{\ell}$ satisfies
$ h_{\ell} \geq \lfloor h/5 \rfloor$.
Adopting the union bound, we can upper-bound
$\zeta_1$ as
\begin{align}
\zeta_1 \leq~& \sum_{\ell =1}^3 \Prob\bracket{\frac{1}{n-h}\Fnorm{\sum_{i\in \calI_{\ell}}
\bX_{\pi^{\natural}(i), :}\bX_{i, :}^{\rmt}\bBtrue} \gsim
\frac{(\log n)(\log n^2 p^3)\sqrt{p}}{\sqrt{n}}\Fnorm{\bBtrue}} \notag \\
\leq ~&
\sum_{\ell=1}^3 \Prob\bracket{\frac{1}{n-h}\Opnorm{\sum_{i\in \calI_{\ell}}
\bX_{\pi^{\natural}(i), :}\bX_{i, :}^{\rmt}} \gsim
\frac{(\log n)(\log n^2 p^3)\sqrt{p}}{\sqrt{n}}}.
\label{eq:event6_zeta1}
\end{align}
Defining $\bZ_{\ell}$ as
$\bZ_{\ell} = \sum_{i\in \calI_{\ell}}\bX_{\pi^{\natural}(i), :}\bX^{\rmt}_{i, :}$,
we would bound the above probability by invoking the matrix
Bernstein inequality (Theorem~$7.3.1$ in~\citet{tropp2015introduction}).
First, we have
\begin{align*}
\Expc\bracket{\bX_{\pi^{\natural}(i), :}\bX^{\rmt}_{i, :}} =
\bracket{\Expc\bX_{\pi^{\natural}(i), :}}\bracket{ \Expc \bX_{i, :}}^{\rmt} = \bZero,
\end{align*}
due to the independence between $\bX_{\pi^{\natural}(i), :}$ and $\bX_{i, :}$.
Then we upper bound $\norm{\bX_{\pi^{\natural}(i), :}\bX_{i, :}^{\rmt}}{2}$ as
\begin{align*}
\norm{\bX_{\pi^{\natural}(i), :}\bX_{i, :}^{\rmt}}{2} \stackrel{\cirtwo}{=} \
\Fnorm{\bX_{\pi^{\natural}(i), :}\bX_{i, :}^{\rmt}} \stackrel{\cirthree}{=} \
\norm{\bX_{\pi^{\natural}(i), :}}{2} \norm{\bX_{i, :}}{2} \stackrel{\cirfour}{\lsim}
p\log n,
\end{align*}
where $\cirtwo$ is because $\bX_{\pi^{\natural}(i), :}\bX_{i, :}^{\rmt}$ is rank-$1$,
$\cirthree$ is due to the fact $\Fnorm{\bu\bv^{\rmt}}^2 = \trace\bracket{\bu \bv^{\rmt}\bv \bu^{\rmt}} = \norm{\bu}{2}^2 \norm{\bv}{2}^2$ for
arbitrary vector $\bu, \bv \in \RR^p$, and
$\cirfour$ is because of event $\calF_3$.

In the end,
we calculate $\Expc\bracket{\bZ_{\ell}\bZ_{\ell}^{\rmt}}$ and $\Expc\bracket{\bZ_{\ell}^{\rmt}\bZ_{\ell}}$ as
\[
\Expc\bracket{\bZ_{\ell}\bZ_{\ell}^{\rmt}} =~&
\Expc\bigg(\sum_{i_1, i_2\in \calI_{\ell}} \bX_{\pi^{\natural}(i_1), :}\bX^{\rmt}_{i_1, :}\bX_{i_2, :} \bX^{\rmt}_{\pi^{\natural}(i_2), :} \bigg) \stackrel{\cirfive}{=} \
\Expc\bigg(\sum_{i \in \calI_{\ell}} \bX_{\pi^{\natural}(i), :}\bX_{i, :}^{\rmt}\bX_{i, :} \bX^{\rmt}_{\pi^{\natural}(i), :} \bigg)
\\ \stackrel{\cirsix}{=} ~& \
\Expc\bigg(\sum_{i\in \calI_{\ell}} \bX_{\pi^{\natural}(i), :}\Expc\bracket{\bX^{\rmt}_{i, :}\bX_{i, :} }   \bX^{\rmt}_{\pi^{\natural}(i), :}   \bigg) =
p \bigg(\sum_{i\in\calI_{\ell}} \Expc \bX_{\pi^{\natural}(i), :} \bX^{\rmt}_{\pi^{\natural}(i), :}\bigg)
=
ph_i \bI_{p\times p} = \Expc\bracket{\bZ_{\ell}^{\rmt}\bZ_{\ell}},
\]
where $\cirfive$ and $\cirsix$ is because of the
fact such that $i$ and $\pi^{\natural}(i)$ are not within the set $\calI_{\ell}$ simultaneously.
To sum up, we invoke the matrix Bernstein inequality (Theorem~$7.3.1$ in~\citet{tropp2015introduction})
and have
\begin{align*}
\frac{1}{n-h}\Opnorm{\sum_{i\in \calI_{\ell}}\bX_{\pi^{\natural}(i), :}\bX_{i, :}^{\rmt}}
\leq~&
\frac{p (\log n) \log(n^2 p^3)}{3 (n-h)}+\frac{\sqrt{p^2 (\log^2 n) \log^2\left(n^2 p^3\right) + 18 p h_i \log \left(n^2 p^3\right)}}{n-h} \\
\stackrel{\cirseven}{\lsim}~& \frac{p (\log n) \log(n^2 p^3)}{n} + \
\frac{p}{n}\sqrt{(\log^2 n) \log^2\left(n^2 p^3\right)  + \frac{n}{p}(\log n^2 p^3)} \\
\stackrel{\cireight}{\lsim}~& \frac{p (\log n) \log(n^2 p^3)}{n}
+  \frac{(\log n)(\log n^2 p^3)\sqrt{p}}{\sqrt{n}}
\stackrel{\cirnine}{\lsim} \frac{(\log n)(\log n^2 p^3)\sqrt{p}}{\sqrt{n}}
\end{align*}
holds with probability $1 - 2(np)^{-2}$,
where in $\cirseven$, $\cireight$, and $\cirnine$ we
use the fact that
$h\lsim n$, $h_i \leq h$, and $n\gsim p$. Hence we can show
$\zeta_1$ in ~\eqref{eq:event6_zeta1} to be less than
$6n^{-2}p^{-2}$.

\vspace{0.1in}\noindent
\textbf{Analysis of $\zeta_2$.}
We upper bound $\zeta_2$ as
\begin{align*}
\zeta_2 \leq~& \
\Prob\bracket{\frac{1}{n-h}\Fnorm{\sum_{i=h+1}^n \bracket{\bX_{i, :} \bX_{i, :}^{\rmt} -\bI}\bBtrue}\gsim \frac{(\log n)(\log n^2p^3)\sqrt{p}}{\sqrt{n}}\Fnorm{\bBtrue} } \\
\leq ~&\Prob\bracket{
\Opnorm{\sum_{i=h+1}^n \bracket{\bX_{i, :} \bX_{i, :}^{\rmt} -\bI}}
\gsim (\log n)(\log n^2 p^3)\sqrt{np} }.
\end{align*}
Similar to above, we define
$\wt{\bZ}_i = \bX_{i, :} \bX_{i, :}^{\rmt} -\bI$.
First, we verify that
$\Expc \wt{\bZ}_i = \bZero$ and
$\wt{\bZ}_i$ are independent with each other. Then we bound
$\opnorm{\wt{\bZ}_i}$ as
\[
\opnorm{\wt{\bZ}_i} \leq \Opnorm{\bX_{i, :} \bX^{\rmt}_{i, :}} +
\opnorm{\bI} \stackrel{\cira}{=} \
\norm{\bX_{i, :}}{2}^2 + 1
\stackrel{\cirb}{\lsim} p\log n + 1 \lsim p\log n,
\]
where in $\cira$
we use $\Opnorm{\bu \bu^{\rmt}} = \norm{\bu}{2}^2$
for arbitrary vector $\bu$, in
$\cirb$ we condition on event $\calF_4$.
In the end, we compute
$\Expc(\wt{\bZ}_i\wt{\bZ}_i^{\rmt})$ as
\[
\Expc(\wt{\bZ}_i \wt{\bZ}_i^{\rmt}) =
\Expc \bracket{\norm{\bX_{i, :}}{2}^2\bX_{i, :} \bX^{\rmt}_{i, :}}
- \bI \preceq
p\log n\cdot \Expc\bracket{\bX_{i, :}\bX^{\rmt}_{i, :}} - \bI
\preceq (p\log n)\bI.
\]
Invoking the matrix Bernstein inequality
(Theorem $7.3.1$ in~\citet{tropp2015introduction}),
we conclude
\[
\zeta_2 \leq 4p\cdot
\exp\bracket{-\frac{3 n(\log n) \log ^2\left(n^2 p^3\right) }{\sqrt{n p} (\log n) \log \left(n^2 p^3\right)+6}}
\stackrel{\circc}{\leq}  4n^{-2}p^{-2},
\]
where in $\circc$ we use the
fact $n\gsim p$. Combining it with
\eqref{eq:event6_tot} and \eqref{eq:event6_tot_bernstein_tot} then completes the proof.
\end{proof}

\begin{lemma}
\label{lemma:xb_single_idx}
Conditional on the intersection of
events $\calF_1(\bBtrue) \bigcap \calF_2$, we conclude
\[
\big \|\big(\wt{\bB} - \wtminus{\bB}{s}\big)^{\rmt}\bX_{s, :}\big \|_{2}
\lsim \frac{p\log^{\nfrac{3}{2}} n}{n} \Fnorm{\bBtrue}.
\] 	
\end{lemma}
\begin{proof}
Here we focus on the case when $\pi^{\natural}(s) = s$.
The proof of the case when $\pi^{\natural}(s) \neq s$ can be completed
effortless with a similar strategy.
First, we notice
\[
\big \|\big(\wt{\bB} - \wtminus{\bB}{s}\big)^{\rmt}\bX_{s, :}\big \|_{2}  =~&
\bracket{n-h}^{-1}
\norm{\bB^{\natural \rmt}
(\wt{\bX}_{s, :} \wt{\bX}^{\rmt}_{s, :} -
\bX_{s, :}\bX^{\rmt}_{s, :} )\bX_{s,:} }{2}\\
\leq~& \bracket{n-h}^{-1}
\bracket{|\langle \bX_{s, :}, \wt{\bX}_{s, :}\rangle|\
\| \bB^{\natural \rmt} \wt{\bX}_{s, :}\|_{2}   + \norm{\bX_{s, :}}{2}^2 \cdot \|\bB^{\natural \rmt}\bX_{s, :}\|_{2} }.
\]
Conditional on the intersection of events
$\calF_1(\bBtrue) \bigcap \calF_2$, we conclude
\[
\big \|\big(\wt{\bB} - \wtminus{\bB}{s}\big)^{\rmt}\bX_{s, :}\big \|_{2}
\lsim \frac{p\log^{\nfrac{3}{2}} n}{n-h} \Fnorm{\bBtrue} \stackrel{}{\asymp}
\frac{p\log^{\nfrac{3}{2}} n}{n} \Fnorm{\bBtrue}.
\]
\end{proof}

Following the same strategy, we can prove that
\begin{lemma}
\label{lemma:xb_multi_indices}
Conditional on the intersection of events
$\calF_1(\bBtrue) \bigcap \calF_2$, we conclude
\[
\norm{(\wt{\bB} - \wtminus{\bB}{s, t})^{\rmt}\bX_{s, :} }{2}
\lsim \frac{p\log^{\nfrac{3}{2}} n}{n}\Fnorm{\bBtrue}.
\] 	
\end{lemma}

\begin{lemma}
\label{lemma:beta_perturb}
Conditional on the intersection of
events $\calF_6\bigcap \calF_7$, we
conclude $\Prob(\calF_8) \geq 1-c_0 \cdot n^{-c_1}$.
\end{lemma}

\begin{proof}
We adopt the leave-one-out trick and construct
the matrix $\wtminus{\bB}{s}$ as
\[
\wtminus{\bB}{s} =
(n-h)^{-1}\bigg(
\sum_{\substack{k \neq s \\ \pi^{\natural}(k) \neq s}} \bX_{\pi^{\natural}(k), :}\bX_{k, :}^{\rmt}
+ \sum_{\substack{k = s \textup{ or}\\ \pi^{\natural}(k) = s}} \bX^{'}_{\pi^{\natural}(k), :}\bX^{' \rmt}_{k, :}\bigg)\bBtrue,
\]
where $\bX^{'}_{s, :}$ are the independent copy of $\bX_{s, :}$.
Adopting the union bound, we conclude
\[
& \Prob\bracket{
\norm{(\wt{\bB} - \bBtrue)^{\rmt} \bX_{s, :}}{2}
\gsim \frac{(\log n)^{\nfrac{3}{2}}(\log n^2 p^3)\sqrt{p}}{\sqrt{n}}\Fnorm{\bBtrue} } \\
\leq ~&
\Prob\bracket{
\norm{(\bBtrue - \wtminus{\bB}{s})^{\rmt}\bX_{s, :} }{2}
+
\norm{(\wtminus{\bB}{s} - \wt{\bB})^{\rmt} \bX_{s, :}}{2} \gsim
\frac{(\log n)^{\nfrac{3}{2}}(\log n^2 p^3)\sqrt{p}}{\sqrt{n}}\Fnorm{\bBtrue}
} \\
\leq~& \underbrace{\Prob\bracket{
\norm{(\bBtrue - \wtminus{\bB}{s})^{\rmt}\bX_{s, :} }{2}
\gsim \frac{(\log n)^{\nfrac{3}{2}}(\log n^2 p^3)\sqrt{p}}{\sqrt{n}}\Fnorm{\bBtrue}}}_{\defequal \zeta_1} \\
+~&
\underbrace{\Prob\bracket{
\norm{(\wtminus{\bB}{s} - \wt{\bB})^{\rmt} \bX_{s, :}}{2} \gsim
\frac{p\log^{\nfrac{3}{2}} n}{n}\Fnorm{\bBtrue}}}_{\defequal \zeta_2}.
\]
First, we study the probability $\zeta_1$.
Due to the construction of $\wtminus{\bB}{s}$, we have
$\bX_{s, :}$ to be independent of
$\bBtrue - \wtminus{\bB}{s}$. Conditional
on $\bBtrue - \wtminus{\bB}{s}$, we conclude
\[
\zeta_1 \stackrel{\cirone}{\leq}
\Prob\bracket{\norm{(\bBtrue - \wtminus{\bB}{s})^{\rmt}\bX_{s, :} }{2} \geq \sqrt{\log n} \Fnorm{\bBtrue - \wtminus{\bB}{s}} }
\leq n^{-c},
\]
where in $\cirone$ we condition on event $\calF_6$
such that $\Fnorm{\bBtrue - \wtminus{\bB}{s}} \lsim (\log n)(\log n^2 p^3)\sqrt{\nfrac{p}{n}}\Fnorm{\bBtrue}$.
As for probability $\zeta_2$, we have it to be zero
conditional on $\calF_7$.
The proof is thus completed.
\end{proof}

\section{Useful Facts}
This section lists some useful
facts for the sake of self-containing.

\begin{lemma}
\label{lemma:chi_square}
For a $\chi^2$-RV $Z$ with $\ell$ freedom, we have
\begin{align*}
\Prob\bracket{Z\leq t} \leq \exp\bracket{\frac{\ell}{2}\bracket{\log \frac{t}{\ell} - \frac{t}{\ell} + 1}},~~t < \ell;  \\
\Prob\bracket{Z\geq t} \leq \exp\bracket{\frac{\ell}{2}\bracket{\log \frac{t}{\ell} - \frac{t}{\ell} + 1}},~~ t >  \ell.
\end{align*}
\end{lemma}

\begin{lemma}[Lemma~$8$ in~\citet{pananjady2018linear}]
\label{lemma:permute_decomp}
Consider an arbitrary permutation map $\pi$
with Hamming distance $h$ from the
identity map, i.e., $\dh\bracket{\bI, \bH} = h$.
We define the index
set $\set{i:~i\neq \pi(i)}$ and can
decompose it into $3$ independent
sets $\calI_{\ell}$ $(1\leq \ell \leq 3)$ such that
the cardinality of each set satisfies
$|\calI_{\ell}|\geq \lfloor \nfrac{h}{3}\rfloor \geq \nfrac{h}{5}$.
\end{lemma}

\begin{lemma}[\cite{latala2007banach}]
\label{lemma:small_ball_subgauss}
Let $\bg \in \RR^n$ be a random vector with
each entry to be independent sub-gaussian RV with
$\Var(g_i) \geq 1$ and sub-gaussian constants
bounded by $K$, and $\bA$ is a non-zero
$n\times n$ matrix. For any $\by\in \RR^n$ and
$\varepsilon\in (0, c_1)$, one has
\begin{align*}
\Prob\bracket{\norm{\by - \bA\bg}{2} \leq \nfrac{1}{2}\fnorm{\bA}}
\leq 2\exp\bracket{-\frac{c_0}{K^4}\srank{\bA} }.
\end{align*}
\end{lemma}

\begin{lemma}[Theorem $1.3$ in~\citet{paouris2012small}]
\label{lemma:small_ball_log_concave}
Let $\bg \in \RR^n$ be an isotropic log-concave random vector
with sub-gaussian constant $K$, and $\bA$ is a non-zero
$n\times n$ matrix. For any $\by\in \RR^n$ and
$\varepsilon\in (0, c_1)$, one has
\begin{align*}
\Prob\bracket{\norm{\by - \bA\bg}{2} \leq \varepsilon \fnorm{\bA}}
\leq \exp\bracket{\kappa(K)\srank{\bA}\log \varepsilon },
\end{align*}
where $\kappa = c_1/K^2$. 	
\end{lemma}

\end{document}